\documentclass{article}
\usepackage[utf8]{inputenc}
\usepackage{proximalfm_arxiv}
\usepackage{multicol}

\usepackage[utf8]{inputenc}
\usepackage{amsmath, amssymb}
\usepackage{algorithm}
\usepackage{algpseudocode}
\usepackage{tcolorbox}
\usepackage{tikz}
\usepackage{dsfont}
\usepackage{svg}
\usepackage{booktabs}
\usepackage{mathtools}
\usepackage{etoc}
\usepackage{appendix}
\usepackage{amsthm}
\usepackage{xspace}
\usepackage{bbm}
\usepackage{tabularx}
\usepackage{longtable}
\usepackage{caption}
\usepackage{multirow}
\usepackage{placeins}
\usepackage{hyperref}
\usepackage{cleveref}
\usepackage{adjustbox}
\usepackage[T1]{fontenc}
\usepackage{lmodern}

\usepackage[shortlabels]{enumitem}

\usetikzlibrary{positioning, arrows.meta}
\usetikzlibrary{calc}
\usetikzlibrary{decorations.pathreplacing}

\setlist[itemize]{leftmargin=12pt}
\setlist[enumerate]{leftmargin=12pt}

\newcommand{\E}{\mathbb{E}}
\newcommand{\Var}{\mathrm{Var}}
\newcommand{\1}{\mathbf{1}}
\newcommand{\R}{\mathbb{R}}
\newcommand{\T}{\top}
\newcommand{\norm}[1]{\left\lVert #1 \right\rVert}

\newcommand{\prob}{\mathbb{P}}
\newcommand{\indep}{\perp \!\!\! \perp}

\newcommand{\modelname}{\textsc{ProximalFM}\xspace}
\newcommand{\codeavailability}{}
\newcommand*\diff{\mathop{}\!\mathrm{d}}

\theoremstyle{plain}
\newtheorem{theorem}{Theorem}[section]

\newtheorem{lemma}[theorem]{Lemma}

\newtheorem{assumption}{Assumption}
\newtheorem{definition}[theorem]{Definition}

\crefname{assumption}{Assumption}{Assumptions}
\Crefname{assumption}{Assumption}{Assumptions}

\AddToHook{env/appendices/begin}{%
  \crefalias{section}{appendix}%
  \crefalias{subsection}{appendix}%
  \crefalias{subsubsection}{appendix}%
}

\title{\modelname: Amortized Proximal Causal Inference under Hidden Confounding}
\author{
  Christophe Muller$^{1}$\thanks{Corresponding author: \href{mailto:christophe.muller@stats.ox.ac.uk}{christophe.muller@stats.ox.ac.uk}},
  Ayub Kharel$^{1}$, 
  Alex Luedtke$^{2}$,
  Chan Park$^{3}$,
  Eric Tchetgen Tchetgen$^{4}$,\\[0.8ex]
  Juan L. Gamella$^{5}$,
  Rahul G. Krishnan$^{6}$,
  Ricardo Silva$^{7}$,
  Jakob Zeitler$^{1,8}$ \\[1.5ex]
  {\fontsize{9}{11}\selectfont
    $^{1}$\,University of Oxford\quad
    $^{2}$\,Harvard University\quad
    $^{3}$\,University of Illinois Urbana-Champaign}\\
  {\fontsize{9}{11}\selectfont
    $^{4}$\,Perelman School of Medicine, University of Pennsylvania\quad
    $^{5}$\,Causal Chamber}\\
  {\fontsize{9}{11}\selectfont
    $^{6}$\,University of Toronto \& Vector Institute\quad
    $^{7}$\,University College London\quad
    $^{8}$\,Pioneer Centre for SMARTbiomed}
}
\date{}

\begin{document}

\maketitle
\addtocontents{toc}{\protect\etocsettocdepth{0}}

\begin{abstract}
    Standard causal identification methods often assume no unmeasured confounding and can fail when relevant confounders are unobserved. Proximal causal inference instead uses proxy variables to identify effects under hidden confounding and additional assumptions. However, nonparametric proximal estimation can be challenging in practice: recovering causal estimands such as the conditional average treatment effect (CATE) requires solving an ill-posed integral equation that is data-hungry, hyperparameter-sensitive, and optimization-unstable. Bayesian inference for such models provides a desirable alternative, mitigating these difficulties by regularizing through the prior. However, computing a posterior is itself challenging, as a typical likelihood function will include latent variables. Following the recent success of tabular foundation models in backdoor, instrumental variable, and frontdoor settings, we propose that prior-data fitted networks (PFNs) are uniquely suited to resolve this bottleneck. Indeed, by training on synthetic data sampled from compliant structural causal models with access to oracle counterfactuals, we simplify the task substantially, amortizing the implied Bayesian operator inversion into a single transformer forward pass. Compared to prior literature that focuses primarily on point estimation, our model, \modelname, explicitly targets the Bayesian posterior distribution of the CATE. One unique aspect of this problem is that we need to provide Monte Carlo estimates of the oracle CATEs, leading to a novel variation of PFNs that accounts for the added stochastic error. Across a diverse suite of proximal regimes, \modelname achieves consistently strong CATE-estimation performance without dataset-specific tuning, with its largest advantage when latent confounding is substantial and the proxies are weakly informative; it also provides fast inference through a single amortized forward pass.
\end{abstract}

\begin{center}
  \textbf{Reproducibility and weights:}\\
  \url{https://github.com/ChristopheMuller/ProximalFM}\\
  \url{https://huggingface.co/ChristopheMuller/ProximalFM}
\end{center}

\section{Introduction}
\label{sec:intro}

Causal inference from observational data is a fundamental task across disciplines, yet its validity typically hinges on the \emph{no unmeasured confounding assumption}. This assumption requires that investigators have measured a sufficiently rich set of covariates to ensure that supervised regression and classification functions can, in theory, also provide how an outcome varies under \emph{intervention} on a possible cause of interest \citep{hernan:20}, with specialized estimation methods improving on off-the-shelf uses of supervised learning. In practice, however, there are often unmeasured variables that are plausible confounders (e.g., unrecorded disease severity).

One of the most common tasks in causal inference is estimating the \emph{conditional average treatment effect} (CATE) function: how an outcome variable $Y$ potentially varies in expectation, as we vary a treatment variable $A$ by \emph{intervening} on the latter. This is a function of an observed realization of a vector of covariates $X$. Using the notation $Y(a)$ to denote the \emph{potential outcome} under a hypothetical intervention on $A$ at level $a$ \citep{rubinEstimatingCausalEffects1972}, the CATE function for a \emph{binary} $A$ is given by 
\begin{equation}
Q(x) = \mathbb{E}[Y(1)-Y(0)\mid X=x].
\label{eq:cate}
\end{equation}

This is identifiable from observational data if unconfoundedness holds. That failing, other assumptions are needed. \emph{Proximal causal learning} has emerged as a principled framework to identify and estimate causal effects when hidden confounders cannot be ruled out \citep{tchetgenIntroductionProximalCausal2020}. The main idea is to exploit auxiliary variables, \emph{proxies}, which are associated with the hidden confounders without fully removing their influence. However, using proxy information while avoiding parametric models requires solving ill-posed integral equations, leading to a difficult estimation problem. As a result, current nonparametric \citep{ghassami2022minimaxkernel, sverdrupProximalCausalLearning2023, YangFW23} and neural estimation methods \citep[NMMR, ][]{kompaDeepLearningMethods2022} are computationally demanding, requiring intensive, dataset-specific optimization.

Simultaneously, the machine learning community is shifting toward amortized inference via foundation models. Prior-data fitted networks \citep[PFNs, ][]{mullerTransformersCanBayesian2024} have demonstrated that transformers can learn to perform Bayesian inference through in-context learning. Models like TabPFN and its successors have shown state-of-the-art performance on small tabular datasets without requiring retraining \citep{grinsztajn2026tabpfn3}. As we review in the next section, recent work has extended PFN-style models to causal inference settings. However, they generally do not make use of proximal assumptions, limiting their utility in the presence of hidden confounders. Our contributions are:


\begin{enumerate}
\item \textbf{\modelname: a foundation model for proximal causal learning.}
We introduce \modelname, the first tabular foundation model designed for proximal causal inference. It leverages proxy variables without requiring dataset-specific training at deployment. Causal effects are obtained in a single forward pass through a target dataset, providing a Bayesian solution to an ill-posed inverse problem without hyperparameter tuning.

\item \textbf{A structural prior and Monte Carlo CATE targets for proximal causal learning.}
We introduce a diverse prior over structural causal models that enforces the proximal assumptions for identification. 
We use Monte Carlo integration to approximate the CATE. These approximations are then used as targets to train \modelname. This prior determines the inductive bias regularizing the proximal estimation problem amortized by \modelname.

\item \textbf{Comprehensive evaluation, from semi-synthetic data to physical systems.}
We evaluate \modelname on diverse semi-synthetic proximal benchmarks that combine real-world covariate distributions with known, controlled causal mechanisms. 
Beyond semi-synthetic benchmarks, we take a first step towards evaluation on real physical systems using Causal Chambers \citep{causalchambers}, a controlled physical testbed for
causal inference.

\end{enumerate}

\subsection{Related work}
\label{sec:lit_review}

\paragraph{Tabular foundation models}

Recent advances have demonstrated the potential of transformer-based foundation models for tabular prediction. TabPFN pioneered this approach for supervised learning tasks with tabular data \citep{hollmannTabPFNTransformerThat2023, hollmannAccuratePredictionsSmall2025, grinsztajn2026tabpfn3}. This success prompted a rapidly expanding literature. Examples include work on more efficient tabular in-context learning, such as TabICL \citep{quTabICLTabularFoundation2025, quTabICLv2BetterFaster2026}, and tuning on real-world data, as in RealTabPFN \citep{gargRealTabPFNImprovingTabular2025}. This progress has also motivated standardized evaluation. TabArena provides a living benchmark for IID tabular classification and regression \citep{erickson2026tabarena}. Extending this evaluation beyond conventional IID settings, \citet{purucker2026beyondiid} find that current tabular foundation models are particularly strong on small- to medium-scale IID datasets, but do not yet consistently outperform conventional methods on non-IID, large-scale, or high-dimensional data.

\paragraph{Amortized causal learning} 

The success of tabular models has motivated their application to causal inference. Causal foundation models are pretrained on a prior over causal data-generating processes and use in-context learning to estimate causal quantities on new datasets in a single forward pass. \citet{stith2026causal} provide a recent overview of this emerging literature.

\citet{robertsonDoPFNInContextLearning2025} introduced Do-PFN, demonstrating that transformer-based foundation models can learn causal estimation procedures directly from synthetic data and perform zero-shot inference on unseen datasets. \citet{balazadehCausalPFNAmortizedCausal2025} introduced CausalPFN, a framework for doing amortized learning of causal effects identified through unconfoundedness. This setting was also considered in \citet{luedtke2024adversarial}, under an adversarially chosen prior. \citet{maFoundationModelsCausal2026} proposed CausalFM, extending causal foundation models beyond unconfoundedness, such as the case of identifiable additive-error instrumental-variable (IV) settings. Rather than learning a single causal estimator, CausalFM trains separate models for distinct causal identification strategies, thereby preserving the classical separation between identification and estimation. \cite{balazadeh2026iv} relaxes the assumption of additive-error in  the IV setting to build in-context learners for \emph{bounding} causal effects for instrumental variables. Different from proximal learning, such partial identification IV methods do not typically exploit multiple proxies, nor exploit completeness assumptions (as defined in the next section) that could otherwise allow for identification.

Despite these advances, a foundation model designed for proximal causal inference remains absent from the literature.

\section{Proximal causal inference}
\label{sec:problem_setup}

We consider an observational causal inference setting where the standard assumption of no unmeasured confounding fails. Our objective is to estimate the CATE, as in Eq. (\ref{eq:cate}), for binary treatments. Let the observed variables be $O = (L, A, Y)$, where $A \in \{0,1\}$ denotes a treatment of interest, $Y$ is the outcome, and $L$ is a set of measured covariates. The unmeasured confounders are denoted by $U$.

In the presence of $U$, the standard exchangeability condition $Y(a) \indep A \mid L$ is violated \citep{hernan:20}. To achieve identification, we adopt the proximal causal learning framework \citep[see, e.g.,][]{miaoIdentifyingCausalEffects2018, tchetgenIntroductionProximalCausal2020}. This framework partitions $L$ into three subsets, each with a distinct role. Let $L=(X,Z,W)$ where $X$ contains baseline covariates, while $Z$ and $W$ are complementary proxies for $U$. $Z$ is a \emph{treatment-confounding} proxy: while it carries information about the latent $U$, it may otherwise be related directly to treatment assignment. Conversely, $W$ is an \emph{outcome-confounding} proxy which carries information about $U$ and may otherwise directly shape the outcome. \Cref{fig:dag_proximal} illustrates some admissible proximal directed acyclic graphs (DAGs), with formal restrictions stated in the following section.

\begin{figure}[t]
    \centering
    \begin{tikzpicture}[
line cap=round,
line join=round,
x=1cm, y=1cm,
main node/.style={circle, draw=black, line width=0.65pt, minimum size=6.4mm, inner sep=0pt, font=\small\itshape},
observed/.style={main node, fill=white, text=black},
latent/.style={main node, fill=white, text=black, dashed, dash pattern=on 2.5pt off 1.8pt},
outcome/.style={main node, draw=blue!80!black, fill=blue!5, text=blue!85!black, line width=0.8pt},
edge/.style={draw=black!85, line width=0.55pt, -{Latex[length=1.6mm,width=1.1mm]}},
latent edge/.style={edge, draw=black!70, dashed, dash pattern=on 2.2pt off 1.6pt},
sublabel/.style={font=\small\bfseries, text=black!85, align=center}
]

    \begin{scope}[xshift=0cm]
        \node[latent]   (U_a) at (0, 1.05) {$U$};
        \node[observed] (X_a) at (0, 0) {$X$};
        \node[observed] (Z_a) at (-1.35, 0) {$Z$};
        \node[observed] (W_a) at (1.35, 0) {$W$};
        \node[observed] (A_a) at (-0.85, -0.95) {$A$};
        \node[outcome]  (Y_a) at (0.85, -0.95) {$Y$};

        \draw[latent edge] (U_a) -- (X_a);
        \draw[latent edge] (U_a) -- (Z_a);
        \draw[latent edge] (U_a) -- (W_a);
        \draw[latent edge] (U_a) to[bend right=10] (A_a);
        \draw[latent edge] (U_a) to[bend left=10] (Y_a);

        \draw[edge] (X_a) -- (Z_a);
        \draw[edge] (X_a) -- (W_a);
        \draw[edge] (X_a) -- (A_a);
        \draw[edge] (X_a) -- (Y_a);

        \draw[edge] (Z_a) -- (A_a);
        \draw[edge] (W_a) -- (Y_a);
        \draw[edge] (A_a) -- (Y_a);

        \node[sublabel] at (0, -1.5) {(a)};
    \end{scope}

    \begin{scope}[xshift=3.8cm]
        \node[latent]   (U_b) at (0, 1.05) {$U$};
        \node[observed] (X_b) at (0, 0) {$X$};
        \node[observed] (Z_b) at (-1.35, 0) {$Z$};
        \node[observed] (W_b) at (1.35, 0) {$W$};
        \node[observed] (A_b) at (-0.85, -0.95) {$A$};
        \node[outcome]  (Y_b) at (0.85, -0.95) {$Y$};

        \draw[latent edge] (U_b) -- (X_b);
        \draw[latent edge] (U_b) -- (Z_b);
        \draw[latent edge] (U_b) -- (W_b);
        \draw[latent edge] (U_b) to[bend right=10] (A_b);
        \draw[latent edge] (U_b) to[bend left=10] (Y_b);

        \draw[edge] (X_b) -- (A_b);
        \draw[edge] (X_b) -- (Y_b);

        \draw[edge] (A_b) -- (Y_b);

        \node[sublabel] at (0, -1.5) {(b)};
    \end{scope}

    \begin{scope}[xshift=7.6cm]
        \node[latent]   (U_c) at (0, 1.05) {$U$};
        \node[observed] (X_c) at (0, 0) {$X$};
        \node[observed] (Z_c) at (-1.35, 0) {$Z$};
        \node[observed] (W_c) at (1.35, 0) {$W$};
        \node[observed] (A_c) at (-0.85, -0.95) {$A$};
        \node[outcome]  (Y_c) at (0.85, -0.95) {$Y$};

        \draw[latent edge] (U_c) -- (Z_c);
        \draw[latent edge] (U_c) -- (W_c);
        \draw[latent edge] (U_c) to[bend right=10] (A_c);
        \draw[latent edge] (U_c) to[bend left=10] (Y_c);

        \draw[edge] (X_c) -- (Y_c);

        \draw[edge] (X_c) -- (U_c);
        \draw[edge] (X_c) -- (A_c);
        \draw[edge] (X_c) -- (Z_c);
        \draw[edge] (X_c) -- (W_c);
        \draw[edge] (A_c) -- (Z_c);
        \draw[edge] (W_c) -- (Y_c);
        \draw[edge] (A_c) -- (Y_c);

        \node[sublabel] at (0, -1.5) {(c)};
    \end{scope}

    \begin{scope}[xshift=11.4cm]
        \node[latent]   (U_d) at (0, 1.05) {$U$};
        \node[observed] (X_d) at (0, 0) {$X$};
        \node[observed] (Z_d) at (-1.35, 0) {$Z$};
        \node[observed] (W_d) at (1.35, 0) {$W$};
        \node[observed] (A_d) at (-0.85, -0.95) {$A$};
        \node[outcome]  (Y_d) at (0.85, -0.95) {$Y$};

        \draw[latent edge] (U_d) -- (Z_d);
        \draw[latent edge] (U_d) -- (W_d);
        \draw[latent edge] (U_d) to[bend right=10] (A_d);
        \draw[latent edge] (U_d) to[bend left=10] (Y_d);

        \draw[edge] (X_d) -- (U_d);

        \draw[edge] (X_d) -- (Z_d);
        \draw[edge] (X_d) -- (A_d);
        \draw[edge] (X_d) -- (Y_d);

        \draw[edge] (W_d) -- (Y_d);
        \draw[edge] (A_d) -- (Y_d);

        \node[sublabel] at (0, -1.5) {(d)};
    \end{scope}

\end{tikzpicture}
    \caption{Examples of DAGs compatible with proximal causal identification assumptions. In every panel, $Z$ and $W$ are informative proxies for the unmeasured confounder $U$. Direct $Z$--$A$ edges (in either direction) and a direct $W \to Y$ edge are optional.}

    \label{fig:dag_proximal}
\end{figure}
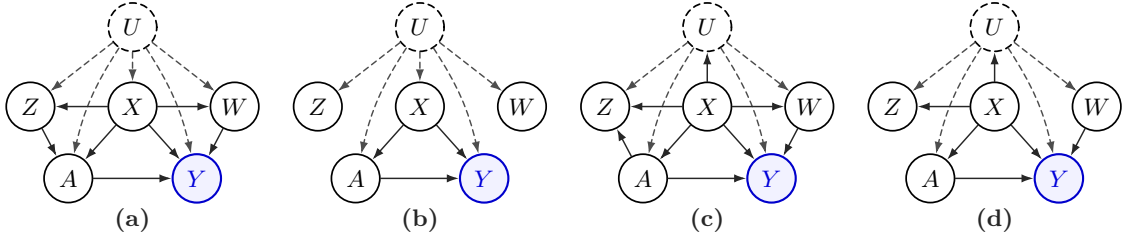

\subsection{Identifying assumptions}
\label{sec:identifying_assumptions}

To identify the causal effect $Q(x)$, it suffices for the joint distribution of $(X, U, Z, W, A, Y)$ to satisfy the following conditions:

\begin{assumption}[Latent exchangeability]
\label{as:latent_ignorability}

The potential outcomes are independent of the treatment assignment conditional on the observed covariates $X$ and the unmeasured confounders $U$:
\[
    Y(a) \indep A \mid (X, U) \qquad \text{for each } a \in \{0,1\}.
\]

\end{assumption}

\begin{assumption}[Proxy exclusion restrictions]
\label{as:proxy_exclusion}

The proxies must satisfy conditional independencies relative to the treatment and outcome:

\begin{multicols}{2}
\begin{itemize}
  \item Treatment-inducing proxy:\enspace
        \mbox{$Y \indep Z \mid (A,X,U)$}
  \item Outcome-inducing proxy:\enspace
        \mbox{$W \indep (A,Z) \mid (X,U)$}
\end{itemize}
\end{multicols}
\end{assumption}

\begin{assumption}[Latent positivity]
\label{as:positivity}

Every observational unit has a non-zero probability of receiving any given treatment assignment, even when conditioning on the observed covariates and the unmeasured confounder:
\[
P(A = a \mid X, U) > 0 \quad \text{ for } a \in \{0,1\}
\]

\end{assumption}




\begin{assumption}[Completeness]
\label{as:completeness}
The proxies must capture enough variation in $U$ to identify the causal effect.
For any square-integrable measurable functions $v$ of $U$ and $g$ of $Z$,
\[
\begin{aligned}
\E[v(U)\mid Z=z,A=a,X=x]=0\ \text{for almost all }(z,a,x)
    &\implies v(U)=0\ \text{a.s.},\\
\E[g(Z)\mid W=w,A=a,X=x]=0\ \text{for almost all }(w,a,x)
    &\implies g(Z)=0\ \text{a.s.}
\end{aligned}
\]
\end{assumption}

\subsection{Classical and modern estimation via bridge functions}
\label{sec:bridge_functions}

In addition to Assumptions~\ref{as:latent_ignorability}--\ref{as:completeness}, we assume the regularity conditions ensuring the existence of an \emph{outcome confounding bridge function} $h_0(W,A,X)$. This function satisfies the Fredholm integral equation \citep{miaoIdentifyingCausalEffects2018, tchetgenIntroductionProximalCausal2020}
\begin{equation}
\E[Y \mid Z,A,X]
=
\E[h_0(W,A,X)\mid Z,A,X].
\label{eq:bridge}
\end{equation}
The bridge function enables identification of the conditional average treatment effect through the \emph{proximal g-formula}
\begin{equation}
Q(x)
=
\E[h_0(W,1,X)-h_0(W,0,X)\mid X=x].
\label{eq:proximal_g_formula}
\end{equation}
\Cref{eq:proximal_g_formula} identifies the effect conditional on $X$ alone by averaging over the proxies. 

Early approaches to recovering this bridge function focused on parametric proximal g-computation. Under linear specifications, proximal g-computation reduces to a two-stage estimation procedure closely related to instrumental variable methods. 
While computationally efficient, proximal two-stage least squares \citep[\texttt{P2SLS},][]{tchetgenIntroductionProximalCausal2020} relies heavily on strict parametric assumptions. To capture more complex confounding mechanisms, recent two-stage methods have replaced linear regressions with highly flexible nonparametric approximators, such as kernel proxy variable \citep[\texttt{KPV},][]{mastouriProximalCausalLearning2021} and deep feature proxy variable \citep[][]{xu2021deep}.

An alternative is to estimate the outcome bridge function by directly enforcing the conditional moment restriction equivalent to \Cref{eq:bridge}.
Proxy maximum moment restriction \citep[\texttt{PMMR},][]{mastouriProximalCausalLearning2021} leverages a reproducing kernel Hilbert space, while neural maximum moment restriction \citep[\texttt{NMMR},][]{kompaDeepLearningMethods2022} extends this approach by using neural networks. 

A separate strand of methods are doubly-robust meta-learners for heterogeneous treatment effect estimation, e.g. \texttt{P-learner} from \citet{sverdrupProximalCausalLearning2023} and the Forster-Warmuth learner of \citep{YangFW23}. They estimate an outcome bridge and a treatment bridge functions, then constructs a pseudo-outcome from them. These outcomes are used to fit a standard conditional mean regression mapping the domain of covariates $X$ to treatment effect.

While these modern approaches are highly adaptive, they introduce their own practical hurdles. They demand complex, dataset-specific hyperparameter tuning and extensive model retraining from scratch for every new causal task. The kernel methods also require the computation of kernel matrices, which can impose computational and memory burden in large sample sizes. This motivates the search for an amortized alternative capable of directly estimating $Q(x)$ from observational context.
\section{ProximalFM}
\label{sec:proximalFM}

We propose \modelname, a PFN for proximal causal inference. Pre-trained on synthetic proximal models, \modelname predicts a posterior distribution over $Q(x)$ from an observational context and query covariates $x$. Unlike existing approaches, it does this without explicitly solving an ill-posed inverse problem as an intermediate step. Instead, it directly outputs a CATE posterior from a single forward pass of a neural network.\codeavailability

\subsection{A synthetic prior for proximal identification}
\label{sec:proximal_prior}

To train \modelname, we construct a prior $\Pi$ over structural causal models \citep[SCMs, ][]{pearl:09}. An SCM specifies a causal graph, structural functions mapping parent variables and exogenous noise to child variables, and distributions over the exogenous noise. Each SCM defines an episode: its structure stays identical throughout, while each of its units is generated by drawing fresh noise and propagating it through its structural equations:
\begin{equation}
\label{eq:structural_scm}
\begin{aligned}
    X &= f_X(\epsilon_X), &
    U &= f_U(X,\epsilon_U), \\
    Z &= f_Z(X,U,\epsilon_Z), &
    W &= f_W(X,U,\epsilon_W), \\
    A &= f_A(X,U,Z,\epsilon_A), &
    \bigl(Y(0),Y(1)\bigr) &= f_Y(X,U,W,\epsilon_Y).
\end{aligned}
\end{equation}
The exogenous-noise vectors $\epsilon_{(.)}$ are drawn independently for every unit from distributions specified by the episode's SCM.  $X$, $U$, $Z$, and $W$ may contain continuous and categorical columns. $A\in\{0,1\}$ is drawn from a Bernoulli propensity. $Y(0)$ and $Y(1)$ are continuous. A generated unit is the tuple $(X,U,Z,W,A,Y(0),Y(1))$: both potential outcomes are generated, but the observational dataset keeps only $Y=Y(A)$.

$\Pi$ satisfies two complementary objectives. First, each SCM must be valid and satisfy the proximal assumptions. Second, the prior must be \emph{diverse}: we build on the synthetic prior of TabICLv1 \citep{quTabICLTabularFoundation2025}, in which each causal mechanism is a randomly initialized multi-layer perceptron. The prior $\Pi$ induces variation across SCMs in their dimensions, nonlinear functional forms, feature types and noise distributions, as well as in confounding strength, treatment overlap, and treatment-effect heterogeneity. An overview of the synthetic prior is available in \Cref{app:prior_procedure}.


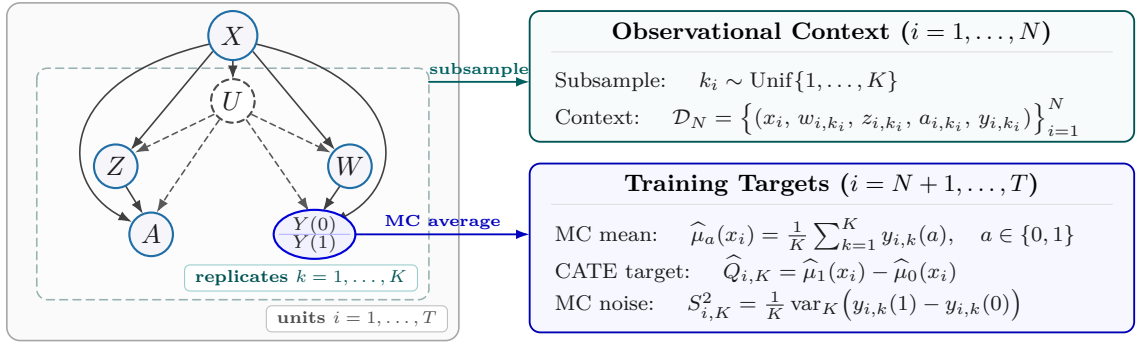
\begin{figure}[t]
    \centering
    \begin{tikzpicture}[
x=0.018cm,
y=-0.018cm,
line cap=round,
line join=round,
plate outer/.style={draw=black!35, fill=black!2!white, line width=0.6pt, rounded corners=5pt},
plate inner/.style={draw=teal!50!black!40, fill=teal!2!gray!2!white, line width=0.6pt, dashed, dash pattern=on 3pt off 2pt, rounded corners=4pt},
plate tag outer/.style={draw=black!25, fill=white, font=\fontsize{6.5}{7.5}\selectfont\bfseries, text=black!70, rounded corners=2pt, inner xsep=4pt, inner ysep=2pt},
plate tag inner/.style={draw=teal!40, fill=white, font=\fontsize{6.5}{7.5}\selectfont\bfseries, text=teal!65!black, rounded corners=2pt, inner xsep=4pt, inner ysep=2pt},
observed/.style={draw=cyan!50!blue!80!black, fill=cyan!6!blue!3, line width=0.8pt},
outcome/.style={draw=blue!85!black, fill=blue!5!white, line width=0.8pt},
latent/.style={draw=black!70, fill=white, line width=0.8pt, dashed, dash pattern=on 2.5pt off 1.8pt},
edge/.style={draw=black!75, line width=0.6pt, -{Latex[length=1.8mm,width=1.3mm]}},
latent edge/.style={edge, draw=black!65, dashed, dash pattern=on 2.2pt off 1.6pt},
pipeline edge context/.style={draw=teal!75!black, line width=0.7pt, -{Latex[length=1.8mm,width=1.3mm]}},
pipeline edge target/.style={draw=blue!75!black, line width=0.7pt, -{Latex[length=1.8mm,width=1.3mm]}},
box context/.style={draw=teal!65!black, fill=teal!3!white, line width=0.7pt, rounded corners=4pt},
box target/.style={draw=blue!70!black, fill=blue!3!white, line width=0.7pt, rounded corners=4pt},
box divider/.style={draw=black!12, line width=0.4pt},
node label/.style={font=\normalsize\itshape, inner sep=0pt, text=black!90},
small text/.style={font=\scriptsize, inner sep=0pt, text=black!85},
output math/.style={font=\footnotesize, inner sep=0pt, text=black!90},
every node/.style={text=black}
]

\draw[plate outer] (18,12) rectangle (350,260);
\draw[plate inner] (40,60) rectangle (328,230);

\node[plate tag inner, anchor=south east] at (320,224) {replicates $k=1,\ldots,K$};
\node[plate tag outer, anchor=south east] at (342,254) {units $i=1,\ldots,T$};

\draw[edge] (184,54) -- (184,68);
\draw[edge] (169,48) -- (109,120);
\draw[edge] (199,48) -- (259,120);
\draw[edge] (168,42) .. controls (54,68) and (54,144) .. (110,173);
\draw[edge] (200,42) .. controls (314,68) and (314,144) .. (258,173);

\draw[latent edge] (170,94) -- (112,122);
\draw[latent edge] (198,94) -- (256,122);
\draw[latent edge] (174,98) .. controls (156,124) and (138,150) .. (128,166);
\draw[latent edge] (194,98) .. controls (212,124) and (230,150) .. (240,165);

\draw[edge] (105,146) -- (118,167);
\draw[edge] (263,146) -- (250,165);

\draw[observed] (184,36) circle[radius=18];
\node[node label] at (184,36) {$X$};

\draw[latent] (184,84) circle[radius=16];
\node[node label] at (184,84) {$U$};

\draw[observed] (98,132) circle[radius=16];
\node[node label] at (98,132) {$Z$};

\draw[observed] (270,132) circle[radius=16];
\node[node label] at (270,132) {$W$};

\draw[observed] (124,182) circle[radius=16];
\node[node label] at (124,182) {$A$};

\draw[outcome] (244,182) ellipse[x radius=30,y radius=18];
\node[small text, font=\scriptsize\bfseries\color{blue!85!black}] at (244,175.5) {$Y(0)$};
\node[small text, font=\scriptsize\bfseries\color{blue!85!black}] at (244,188.5) {$Y(1)$};
\draw[draw=blue!30, line width=0.4pt] (218,182) -- (270,182);

\draw[box context] (402,18) rectangle (846,114);
\node[font=\small\bfseries\color{teal!70!black}] at (624,34) {Observational Context ($i=1,\ldots,N$)};
\draw[box divider] (418,48) -- (830,48);
\node[output math, anchor=west] at (420,70) {Subsample: \quad $k_i \sim \operatorname{Unif}\{1,\ldots,K\}$};
\node[output math, anchor=west] at (420,96) {Context: \quad $\mathcal{D}_N = \bigl\{(x_i,\, w_{i,k_i},\, z_{i,k_i},\, a_{i,k_i},\, y_{i,k_i})\bigr\}_{i=1}^N$};

\draw[pipeline edge context] (328,70) -- node[above, font=\fontsize{6}{7}\selectfont\bfseries, color=teal!75!black, inner sep=1.5pt] {subsample} (402,70);

\draw[box target] (402,130) rectangle (846,254);
\node[font=\small\bfseries\color{blue!80!black}] at (624,146) {Training Targets ($i=N+1,\ldots,T$)};
\draw[box divider] (418,160) -- (830,160);
\node[output math, anchor=west] at (420,182) {MC mean: \quad $\widehat{\mu}_a(x_i) = \frac{1}{K}\sum_{k=1}^K y_{i,k}(a), \quad a \in \{0,1\}$};
\node[output math, anchor=west] at (420,208) {CATE target: \quad $\widehat{Q}_{i,K} = \widehat{\mu}_1(x_i) - \widehat{\mu}_0(x_i)$};
\node[output math, anchor=west] at (420,234) {MC noise: \quad $S_{i,K}^2 = \frac{1}{K}\operatorname{var}_K\bigl(y_{i,k}(1) - y_{i,k}(0)\bigr)$};

\draw[pipeline edge target] (274,182) -- node[above, font=\fontsize{6}{7}\selectfont\bfseries, color=blue!75!black, inner sep=1.5pt] {MC average} (402,182);

\end{tikzpicture}
    \caption{Synthetic-prior episode: each covariate $x_i$ is fixed across $K$ downstream SCM draws. One draw forms the observational context; replicate potential outcomes yield Monte Carlo CATE targets and their sampling variance. $\operatorname{var}_K$ is the unbiased sample variance of the $K$ paired treatment effects.} 
    \label{fig:prior_dag}
\end{figure}



\paragraph{Enforcing proximal assumptions.}

By construction, every SCM drawn from $\Pi$ is designed to satisfy the identification conditions of \Cref{sec:identifying_assumptions} (see \Cref{app:prior_audit} for empirical diagnostics):
\begin{itemize}
    \item \textbf{Structural constraints (\ref{as:latent_ignorability} \& \ref{as:proxy_exclusion}).} By construction in \eqref{eq:structural_scm}, mutual independence of the exogenous noises $\epsilon_{(.)}$ and the specified mechanism inputs encode the required exclusions. The corresponding conditional independencies follow immediately for every sampled SCM.
    \item \textbf{Positivity (\ref{as:positivity}).} Guaranteed by clamping treatment assignment probabilities to $[\delta, 1-\delta]$ with $\delta > 0$ before sampling $A$.
    \item \textbf{Completeness (\ref{as:completeness}).} Unlike structural graph constraints, completeness is a more subtle, distribution-level condition: it is untestable in observational data and cannot be strictly guaranteed by construction for arbitrary non-linear mechanisms \citep{tchetgenIntroductionProximalCausal2020}. The prior therefore relies on a heuristic requiring the dimensions of $W$ and $Z$ to be greater than the one of $U$. Their cardinalities also dominate the one of $U$ (see \Cref{app:prior_postprocess}).
\end{itemize}

\paragraph{Contexts \& Monte Carlo targets.}

A naive training framework partitions each dataset of $T$ units generated from the prior, $\{(x_i,u_i,z_i,w_i,a_i,y_i(0),y_i(1))\}_{i=1}^T$, into $N$ context units and $T-N$ query units. The context becomes $\mathcal{D}_N=\{(x_i,z_i,w_i,a_i,y_i\}_{i=1}^N$ while the query inputs and targets are $\{(x_i,y_i(1)-y_i(0))\}_{i=N+1}^T$. The PFN then uses the query individual treatment effects (ITE, $y_i(1)-y_i(0)$) as targets to be predicted from $(\mathcal{D}_N,x_i)$. While simple, this framework is flawed. First, training on ITEs would target their conditional distribution, whereas our estimand is the CATE $Q(x_i)$. Second, the ITE distribution depends on the joint law of $(Y(1),Y(0))$. Our prior defines this law, but it is typically unknown and not identified from observational data, so the learned ITE uncertainty would only reflect our choice of prior. 

We instead approximate the CATE of the query units. As $Q(x)$ is an average over the heterogeneity in $U$, $W$ and noise, it cannot be generated in a single draw. Each query $x_i$ is kept fixed and we run the downstream SCM $K$ times with independent noise. We obtain the individual effects $\tau_{i,k} = y_{i,k}(1) - y_{i,k}(0)$ whose mean is $\widehat{Q}_{i,K}$, with estimated sampling variance $S_{i,K}^2$ (\Cref{fig:prior_dag}). The estimator $\widehat{Q}_{i,K}$ is unbiased for $Q(x_i)$ at every $K$ and $S_{i,K}^2 \to 0$ as $K$ grows, so $K$ trades compute against target precision. The transformer receives exclusively the observational context $\mathcal{D}_N$ and the query covariates $x_i$; the Monte Carlo estimates $\widehat{Q}_{i,K}$ and $S_{i,K}^2$ are reserved for training supervision.


\subsection{Model \& objective}
\label{sec:model_training}

\modelname is a PFN trained to map an observational dataset directly to a Bayesian posterior distribution over the CATE estimand. Given an observational context set $\mathcal{D}_N = \bigl\{ (x_i,\, z_i,\, w_i,\, a_i,\, y_i) \bigr\}_{i=1}^{N}$ of observed units and a query covariate vector $x^\star$, the model outputs a predictive distribution over the CATE $Q(x^\star)$ in a single forward pass. As it is trained on datasets from the proximal prior $\Pi$ (\Cref{sec:proximal_prior}), its output approximates the CATE posterior predictive distribution under $\Pi$ \citep{mullerTransformersCanBayesian2024}: dataset-specific Bayesian inference is amortized once at pre-training time.

\begin{figure}[t]
    \centering
    \begin{tikzpicture}[
x=0.015cm,
y=-0.015cm,
line cap=round,
line join=round,
plate color/.style={draw=black!45, fill=gray!3, line width=0.5pt, rounded corners=3pt},
op color/.style={draw=teal!80!black, fill=teal!5, line width=0.6pt, rounded corners=2pt},
out color/.style={draw=blue!80!black, fill=blue!3, line width=0.8pt, rounded corners=2pt},
cellX/.style={draw=cyan!50!blue!80!black, fill=cyan!5, line width=0.6pt},
cellZ/.style={draw=orange!75!black, fill=orange!8, line width=0.6pt},
cellW/.style={draw=violet!70!black, fill=violet!6, line width=0.6pt},
cellA/.style={draw=green!45!black, fill=green!8, line width=0.6pt},
cellY/.style={draw=blue!80!black, fill=blue!4, line width=0.6pt},
roleX/.style={cellX, fill=cyan!18, rounded corners=2pt, line width=0.7pt},
roleZ/.style={cellZ, fill=orange!20, rounded corners=2pt, line width=0.7pt},
roleW/.style={cellW, fill=violet!16, rounded corners=2pt, line width=0.7pt},
plate/.style={plate color},
edge/.style={draw=black!75, line width=0.6pt, -{Latex[length=1.8mm,width=1.3mm]}},
operation/.style={op color},
output/.style={out color},
histbar/.style={draw=blue!70!black, fill=blue!22, line width=0.4pt},
meanline/.style={draw=black!60, dashed, dash pattern=on 2pt off 1.5pt, line width=0.5pt},
small text/.style={font=\scriptsize, inner sep=0pt, text=black!85},
tiny text/.style={font=\fontsize{7}{8.5}\selectfont, inner sep=0pt, text=black!60},
output math/.style={font=\footnotesize, inner sep=0pt, text=black!90},
every node/.style={text=black}
]

\draw[plate] (18,50) rectangle (404,267);
\node[tiny text, anchor=south west] at (28,261) {\textbf{context embedding}};

\node[small text] at (36,74) {$X$};
\draw[cellX] (50,61) rectangle (104,87);
\node[small text] at (77,74) {$x_1$};
\draw[cellX] (110,61) rectangle (164,87);
\node[small text] at (137,74) {$x_2$};
\node[tiny text] at (176,74) {$\cdots$};
\draw[cellX] (188,61) rectangle (242,87);
\node[small text] at (215,74) {$x_N$};
\draw[draw=black!55, line width=0.5pt] (256,74) circle (6);
\node[tiny text] at (256,74) {$+$};
\draw[roleX] (270,61) rectangle (308,87);
\node[small text] at (289,74) {$e_X$};

\node[small text] at (36,108) {$Z$};
\draw[cellZ] (50,95) rectangle (104,121);
\node[small text] at (77,108) {$z_1$};
\draw[cellZ] (110,95) rectangle (164,121);
\node[small text] at (137,108) {$z_2$};
\node[tiny text] at (176,108) {$\cdots$};
\draw[cellZ] (188,95) rectangle (242,121);
\node[small text] at (215,108) {$z_N$};
\draw[draw=black!55, line width=0.5pt] (256,108) circle (6);
\node[tiny text] at (256,108) {$+$};
\draw[roleZ] (270,95) rectangle (308,121);
\node[small text] at (289,108) {$e_Z$};

\node[small text] at (36,142) {$W$};
\draw[cellW] (50,129) rectangle (104,155);
\node[small text] at (77,142) {$w_1$};
\draw[cellW] (110,129) rectangle (164,155);
\node[small text] at (137,142) {$w_2$};
\node[tiny text] at (176,142) {$\cdots$};
\draw[cellW] (188,129) rectangle (242,155);
\node[small text] at (215,142) {$w_N$};
\draw[draw=black!55, line width=0.5pt] (256,142) circle (6);
\node[tiny text] at (256,142) {$+$};
\draw[roleW] (270,129) rectangle (308,155);
\node[small text] at (289,142) {$e_W$};

\draw[draw=black!30, line width=0.5pt, dash pattern=on 1.6pt off 1.6pt] (30,166) -- (392,166);

\node[small text] at (36,188) {$a$};
\draw[cellA] (50,175) rectangle (104,201);
\node[small text] at (77,188) {$a_1$};
\draw[cellA] (110,175) rectangle (164,201);
\node[small text] at (137,188) {$a_2$};
\node[tiny text] at (176,188) {$\cdots$};
\draw[cellA] (188,175) rectangle (242,201);
\node[small text] at (215,188) {$a_N$};

\node[small text] at (36,222) {$y$};
\draw[cellY] (50,209) rectangle (104,235);
\node[small text] at (77,222) {$y_1$};
\draw[cellY] (110,209) rectangle (164,235);
\node[small text] at (137,222) {$y_2$};
\node[tiny text] at (176,222) {$\cdots$};
\draw[cellY] (188,209) rectangle (242,235);
\node[small text] at (215,222) {$y_N$};

\draw[decorate, decoration={brace, mirror, amplitude=4pt}, draw=black!55, line width=0.5pt]
    (318,155) -- (318,61);
\draw[edge] (324,108) .. controls (360,108) and (400,122) .. (415,132);

\draw[decorate, decoration={brace, mirror, amplitude=4pt}, draw=black!55, line width=0.5pt]
    (250,235) -- (250,175);
\draw[operation] (272,188) rectangle (368,222);
\node[small text] at (320,205) {$e_{y \times a}$};
\draw[edge] (260,205) -- (272,205);

\draw[draw=black!55, line width=0.5pt] (422,138) circle (7);
\node[small text] at (422,138) {$+$};
\draw[edge] (368,205) .. controls (402,205) and (420,165) .. (422,145);
\draw[edge] (429,138) -- (470,138);

\draw[plate] (18,290) rectangle (404,370);
\node[tiny text, anchor=south west] at (28,364) {\textbf{query embedding}};

\node[small text] at (36,322) {$X^\star$};
\draw[cellX] (50,309) rectangle (104,335);
\node[small text] at (77,322) {$x^\star_1$};
\draw[cellX] (110,309) rectangle (164,335);
\node[small text] at (137,322) {$x^\star_2$};
\node[tiny text] at (176,322) {$\cdots$};
\draw[cellX] (188,309) rectangle (242,335);
\node[small text] at (215,322) {$x^\star_M$};
\draw[draw=black!55, line width=0.5pt] (256,322) circle (6);
\node[tiny text] at (256,322) {$+$};
\draw[roleX] (270,309) rectangle (308,335);
\node[small text] at (289,322) {$e_X$};

\draw[edge] (308,322) -- (470,322);

\draw[operation] (470,50) rectangle (620,370);
\node[small text, font=\normalsize\bfseries\color{teal!50!black}] at (545,198) {Transformer};
\node[small text, font=\normalsize\bfseries\color{teal!50!black}] at (545,222) {Encoder};

\draw[plate] (676,50) rectangle (1008,370);
\node[tiny text, anchor=south east] at (1000,362) {\textbf{prediction head}};

\draw[edge] (620,138) -- (695,138);

\draw[output] (695,62) rectangle (989,235);
\node[small text] at (922,80) {$q_\theta\bigl(\cdot \mid x^\star, \mathcal{D}_N\bigr)$};

\draw[draw=black!50, line width=0.5pt] (705,220) -- (979,220);

\draw[histbar] (708,213) rectangle (723,220);
\draw[histbar] (727,202) rectangle (742,220);
\draw[histbar] (746,180) rectangle (761,220);
\draw[histbar] (765,145) rectangle (780,220);
\draw[histbar] (784,112) rectangle (799,220); 
\draw[histbar] (803,128) rectangle (818,220);
\draw[histbar] (822,156) rectangle (837,220);
\draw[histbar] (841,149) rectangle (856,220); 
\draw[histbar] (860,169) rectangle (875,220);
\draw[histbar] (879,188) rectangle (894,220);
\draw[histbar] (898,202) rectangle (913,220);
\draw[histbar] (917,210) rectangle (932,220);
\draw[histbar] (936,215) rectangle (951,220);

\node[tiny text, anchor=south] at (811,96) {$\hat{Q}_\theta(x^\star)$};
\draw[meanline] (811,100) -- (811,220);

\draw[draw=blue!80!black, line width=0.4pt, <->] (748,206) -- (892,206);
\node[tiny text, font=\fontsize{6}{7}\selectfont\color{blue!80!black}, fill=blue!4, inner sep=1pt] at (820,206) {posterior width};

\draw[edge] (842,235) -- (842,265);

\draw[output] (765,265) rectangle (919,340);
\node[output math] at (842,292) {$\hat{Q}_\theta(x^\star)$};
\node[tiny text, font=\fontsize{7}{8}\selectfont\bfseries\color{blue!60!black}] at (842,320) {CATE estimate};

\end{tikzpicture}
    \caption{\modelname architecture. $e_X$, $e_Z$, and $e_W$ are learned structural-role embeddings; $e_{y\times a}$ encodes the factual outcome according to treatment.}
    \label{fig:model}
\end{figure}
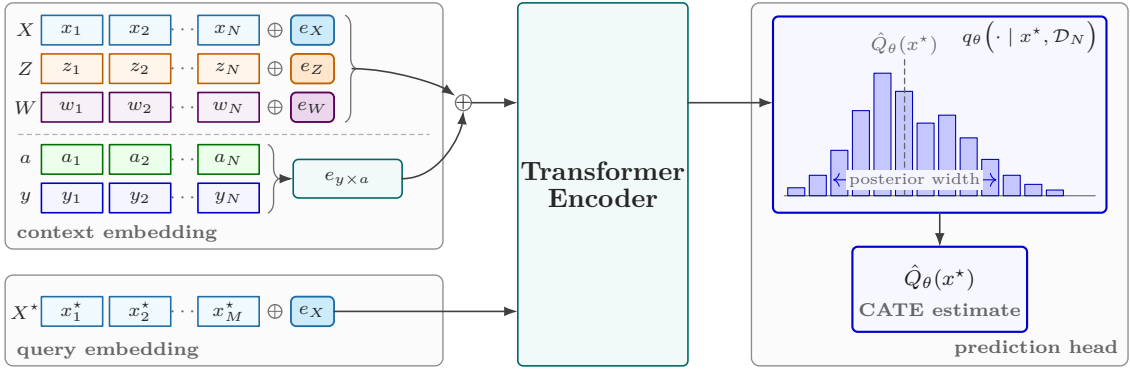

\paragraph{Encoding the proximal setting.}

The backbone of \modelname follows the TabICLv2 architecture \citep{quTabICLv2BetterFaster2026}; we defer architectural details to \Cref{app:model_arch}. Our modifications concern the input processing which must carry the causal role of each column as the roles of $(X,Z,W)$ are not identifiable from their position or values. 
\modelname differs from regular tabular models such as TabICLv2 in two respects.  First, its target is not observed in the context: the ITE is never observed, only the factual outcome $Y=Y(A)$ is.  These observations only provide evidence about the underlying SCM and its CATE function $Q(x)$.  Second, \modelname targets the posterior distribution over the CATE, $p\bigl(Q(x^\star)\mid\mathcal{D}_N\bigr)$, rather than a distribution over an individual-level label such as the ITE.

Each query is represented by $x^\star$ only: including its proxy values $(w^\star,z^\star)$ would condition the prediction on the proxies and change the target away from the CATE $Q(x^\star)$. We add \emph{structural role embeddings},  a learnable token per feature group ($X$, $Z$, $W$), added to its column embeddings. A \emph{treatment-aware target encoder} is also implemented. It consists of two linear projections of the outcome selected by the treatment indicator $A$, so that outcomes observed under treatment and control are embedded separately (\Cref{fig:model}).

\paragraph{Posterior output \& noise-aware loss.}
\modelname uses a single distributional head to predict the posterior distribution $q_\theta(Q \mid x^\star, \mathcal{D}_N)$ directly, rather than separate heads for the two potential outcomes. We use regression-as-classification: $q_\theta$ is predicted as a histogram over $2\,000$ equal-width bins spanning $[-5, 5]$, in units of the standard deviation of the observed outcomes in context. 

The synthetic targets $\widehat Q_{i,K}$ are Monte Carlo averages over $K$ replicates. For sufficiently large $K$, their Monte Carlo error is approximately Gaussian by the central limit theorem, with variance estimated by $S_{i,K}^2$. Scoring them directly under a distributional loss would inflate the predicted variance by approximately $S_{i,K}^2$. Because $S_{i,K}^2$ is estimated, we correct for it by blurring the model rather than the target: we maximize the likelihood of $\widehat Q_{i,K}$ under the \emph{noise-convolved} predictive distribution
\begin{equation}
    \tilde{q}_\theta(\cdot \mid x_i, \mathcal{D}_N) \coloneqq q_\theta(\cdot \mid x_i, \mathcal{D}_N) \ast \mathcal{N}\bigl(0,\, S_i^2\bigr).
    \label{eq:blurred_distribution}
\end{equation}
giving, over the query points $i \in \{1, \dots, n_q\}$,
\begin{equation}
    \mathcal{L}_{\mathrm{noise\_aware}} = -\frac{1}{n_q} \sum_{i=1}^{n_q} \log \tilde{q}_\theta\bigl(\widehat Q_{i,K} \mid x_i, \mathcal{D}_N\bigr).
    \label{eq:noise_aware_loss}
\end{equation}

The Monte Carlo inflation is removed: the minimizer of \eqref{eq:noise_aware_loss} approximates the posterior over $Q$,  up to finite histogram resolution and the Gaussian approximation to the $K$ Monte Carlo noise. The closed form of \eqref{eq:blurred_distribution} for a histogram head is given in \Cref{app:noise_aware_loss}. We add a squared-error regularizer on the posterior mean, $\mathcal{L}_{\mathrm{mean}} = \frac{1}{n_q} \sum_{i=1}^{n_q} \bigl( \E_{q_\theta}[Q \mid x_i, \mathcal{D}_N] - \hat Q_K(x_i) \bigr)^2$. The total training objective is $\mathcal{L} = \mathcal{L}_{\mathrm{noise\_aware}} + \lambda_{\mathrm{mean}} \mathcal{L}_{\mathrm{mean}}$. At inference time, \modelname outputs the full CATE posterior distribution $q_\theta(Q \mid x^\star, \mathcal{D}_N)$ in a single forward pass, with point estimate $\hat Q_\theta(x^\star) = \E_{q_\theta}[Q \mid x^\star, \mathcal{D}_N]$. Details on the training are available in \Cref{app:model_training}.

\subsection{From proximal identification to Bayesian consistency}
\label{sec:bayesian_consistency}

Let $\psi$ parametrize an SCM drawn from $\Pi$. We adapt the Bayesian consistency argument of \cite{balazadehCausalPFNAmortizedCausal2025} to the CATE $Q(x;\psi)=\mu_1(x;\psi)-\mu_0(x;\psi)$, where $\mu_a(x;\psi)=\mathbb{E}_{\psi}[Y(a)\mid X=x]$, the conditional expected potential outcome (CEPO). Under the proximal identifying conditions, including completeness and existence of an outcome bridge, each CEPO is an observational functional $\mu_a(x;\psi)=\mathbb{E}_{\psi}[h_0(W,a,X)\mid X=x]$. Their difference therefore identifies $Q(x)$ and directly targeting CATE requires no change to this identification argument.

Suppose these conditions hold almost surely under the SCM prior $\Pi$. Under the measurability and integrability assumptions in Appendix~\ref{sec:consistency_result}, for each admissible fixed query $x$ and $\Pi$-almost every true SCM $\psi^\star$, the oracle Bayesian posterior mean satisfies
\begin{equation}
\label{eq:main_cate_consistency}
\overline Q_N(x) :=\int Q(x;\psi)\,\Pi(d\psi\mid\mathcal{D}_N) \xrightarrow[N\to\infty]{\mathrm{a.s.}} Q(x;\psi^\star).
\end{equation}
Theorem~\ref{theorem:ident} establishes asymptotic consistency of the Bayesian CATE target underlying \modelname: as the observational context grows, the oracle posterior concentrates on the true CATE and its mean converges to that value. This provides a theoretical foundation for the inference target that \modelname learns to approximate through Monte Carlo supervision and amortized inference. Our trained model's posterior-mean predictions inherit consistency when their approximation error relative to the oracle also vanishes.
\section{Evaluation}
\label{sec:eval}

\subsection{Evaluation setting}
\label{sec:eval_setting}

\paragraph{Methods evaluated.}

We compare \modelname against three families of methods. First, proximal estimators recover causal effects through bridge functions or conditional moment restrictions. These are the only baselines encoding the proximal identification assumptions. Second, we evaluate existing causal foundation models that assume no unmeasured confounding. Third, conventional S-, T-, and X-meta-learners provide other non-proximal backdoor baselines \citep{kunzelMetalearnersEstimatingHeterogeneous2019}. The backdoor methods are evaluated by either fitting $X$ alone or fitting $(X,W,Z)$ followed by marginalizing the proxy-conditional contrast over the proxies. Both variants return a function of $x$, but neither is generally identified as $Q(x)$ under the proximal assumptions. \Cref{app:results_baselineMethods} lists all methods and refers to their configuration and implementation details.

\paragraph{Datasets.} \Cref{app:results_semisynt} details our semi-synthetic datasets generation. The covariates $X$ and unobserved $U$ are features from real-world data. The proxies $W$ and $Z$ are then generated from $(U,X)$. Similarly, the treatment assignment $A$ and outcomes $Y(1),Y(0)$ are generated from $(U,X)$. Note that we don't use $Z$ and $W$ as direct inputs when generating the treatment and outcomes.

This semi-synthetic approach grounds the simulation on real-world features, while having control over dataset properties such as confounding strength of $U$ and proxy informativeness of $(W,Z)$ about $U$. We consider a factorial design crossing linear versus nonlinear mechanisms, low versus high latent-confounding strength, and low versus high proxy informativeness, yielding eight configurations. For each of the eight settings, we evaluate 12 datasets with three independent replicates each, giving a total of 288 evaluation instances. We set \(d_X=d_U=d_W=d_Z=5\).

\subsection{Empirical results}
\label{sec:results_performance}

\paragraph{\modelname is competitive across settings.}

\Cref{fig:semisynt_learning_curves} summarizes the performance of selected methods, with detailed results available in Tables~\ref{tab:semisynt_exhaustive_npehe_128}--\ref{tab:semisynt_exhaustive_npehe_2048}. Across all eight settings, \modelname is competitive, with its clearest advantage at larger context sizes under high confounding and low proxy reliability. Among the proximal methods, \texttt{P-learner (TabICL)} is also consistently competitive, but typically requires larger context sizes to obtain good performance. \texttt{KPV} and \texttt{PMMR} generally underperform, possibly because of the use of kernel-based bridge functions in this moderate dimension regime ($d_X=d_U=d_W=d_Z=5$). Flexible backdoor methods such as \texttt{S-learner (TabICL)} and \texttt{CausalPFN} perform surprisingly well when combined with simple proxy marginalization.  Their performance deteriorates, however, when strong confounding is combined with low proxy information.

\begin{figure}[t]
    \centering
    \includegraphics[width=\textwidth]{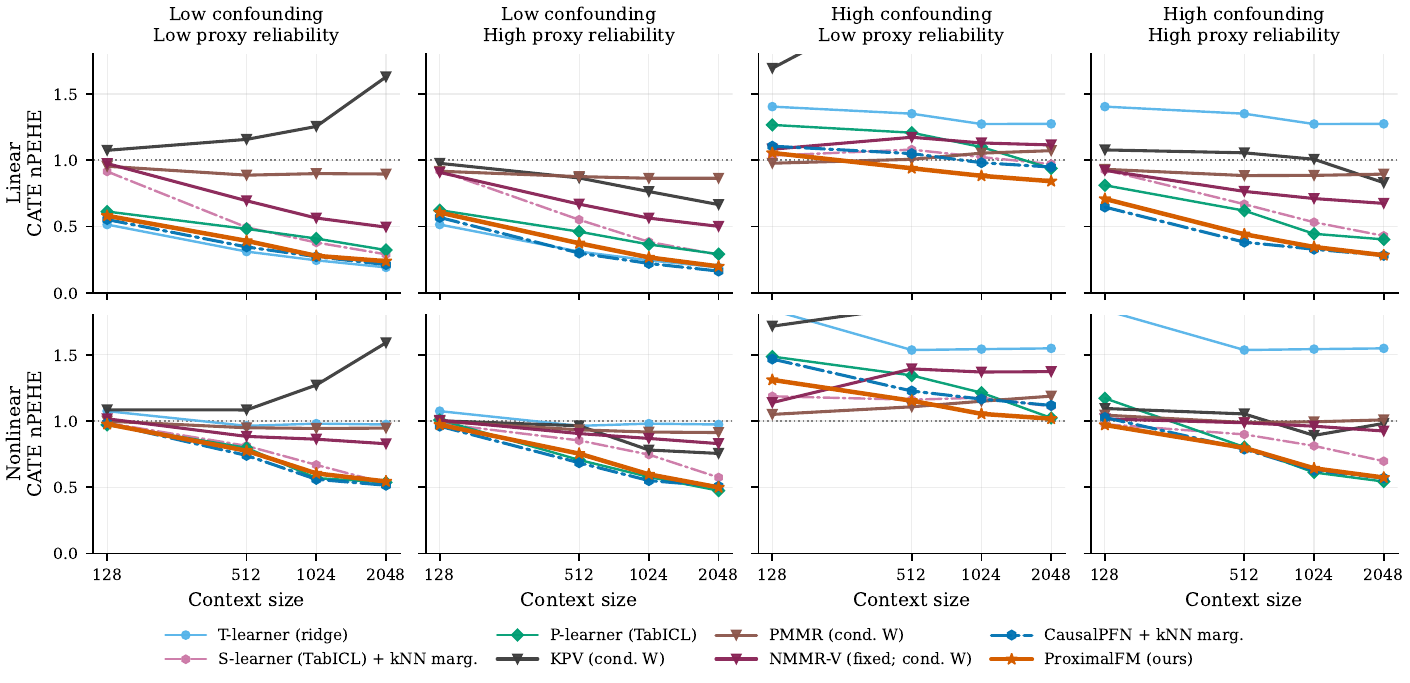}
    \caption{Median CATE nPEHE across benchmark episodes versus context size. Rows show linear and nonlinear mechanisms; columns vary latent confounding and proxy reliability. The dotted line marks the constant-ATE baseline ($\mathrm{nPEHE}=1$); lower is better.}
    \label{fig:semisynt_learning_curves}
\end{figure}

\paragraph{\modelname plateaus beyond its training range.}
We next consider the nonlinear regime with high latent confounding and low proxy reliability (\Cref{fig:semisynt_learning_curves_selection}, \Cref{tab:semisynt_extended_context_npehe}). This is a particularly relevant proximal setting: the confounding cannot be ignored, the proxy signal is weak and thus conditioning on proxies is insufficient. As the context grows beyond $n_{\mathrm{ctx}}=2\,048$, the \texttt{P-learner (TabICL)} continues to improve and becomes the strongest method, whereas \modelname improves little further and remains close to the constant-ATE baseline. A plausible explanation is a mismatch between evaluation and pre-training context sizes: \modelname was trained on contexts of at most $2\,048$ observations. To assess this, we continue training the model for limited time on episodes with contexts of up to $5\,000$ observations (\Cref{app:model_posttraining}, \texttt{ProximalFM-FT}). The adapted model shows small improvements at large context sizes, consistent with limited large-context training support contributing to the observed plateau.

\paragraph{\modelname is fast and requires no task-specific tuning.}
\modelname lies on a favourable accuracy--runtime frontier (\Cref{fig:results_speed_rank}). It achieves among the strongest CATE ranks while requiring little end-to-end runtime compared to several competing proximal methods. Moving inference from CPU to MPS further reduces its runtime without changing its predictions. \modelname does not require hyperparameter selection, simplifying deployment and making its reported runtime unaffected by the choice of a selection grid.

\begin{figure}
    \centering
    \includegraphics[width=\textwidth]{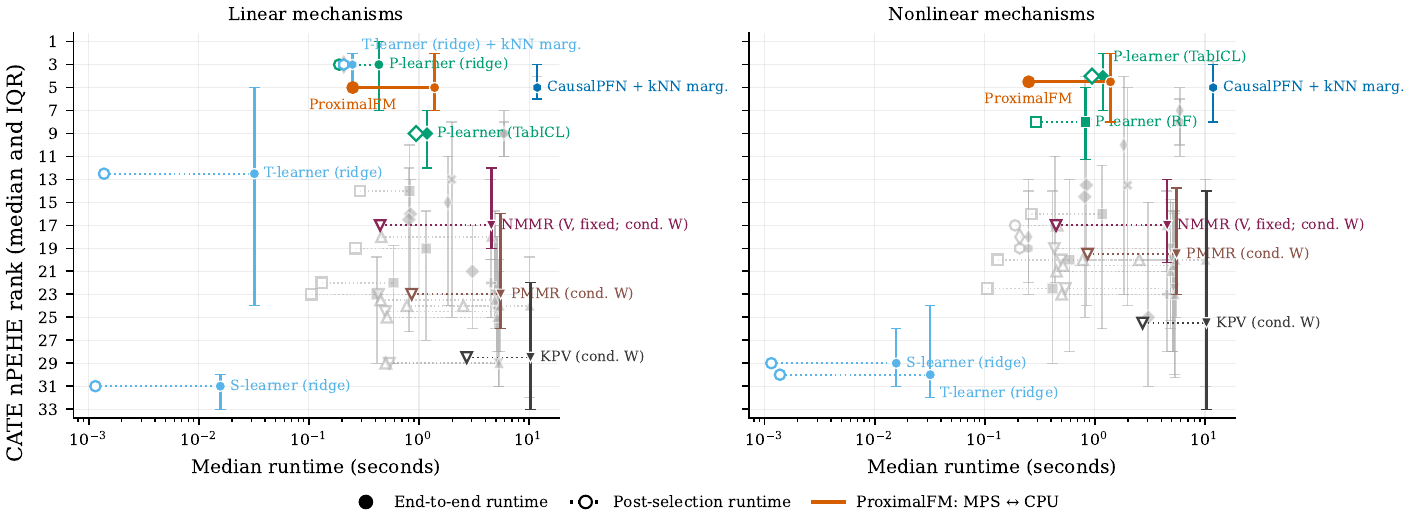}
    \caption{Accuracy and speed at context size $n=2\,048$, shown separately for linear and nonlinear mechanisms. Points show median runtime per episode against median paired CATE nPEHE rank (IQR whiskers). Filled and hollow points respectively include and exclude hyperparameter selection. The orange segment compares \modelname on CPU and MPS. More details in \Cref{app:results_speed}.
    }
    \label{fig:results_speed_rank}
\end{figure}

\paragraph{The posterior is overconfident on semi-synthetic data.}
\Cref{fig:posterior_calibration} compares nominal and empirical credible-interval coverage. On datasets sampled from the pre-training prior, coverage is close to nominal across all context sizes, and posterior uncertainty contracts steadily as the context grows. On the other hand, on the semi-synthetic datasets, the posterior is increasingly overconfident as the context size grows. This undercoverage under distribution shift is consistent with the out-of-distribution experiments from \citet{balazadehCausalPFNAmortizedCausal2025}. Notably, the posterior standard deviation continues to decrease beyond the maximum pre-training context size.

\paragraph{Proxy informativeness matters.}
The good performance of TabICL and CausalPFN in the semi-synthetic benchmark raises the question of whether flexible backdoor methods can substitute for proximal estimation. We investigate this in a setting with linear outcome and Gaussian proxy mechanisms, where we control the information that $(W,Z)$ carries about $U$ (\Cref{app:proxy_informativeness}). At $n_{\mathrm{ctx}}=2\,048$, \modelname outperforms proxy-marginalized CausalPFN throughout the sweep and proxy-adjusted linear regression (OLS) in most settings, although OLS catches up when proxies nearly reveal $U$. Well specified proximal 2SLS \citep{tchetgenIntroductionProximalCausal2020} provides a reference for what matching the estimator to this mechanism can achieve.

\paragraph{Towards a real physical evaluation.} 

We compare ATE estimations on Causal Chambers, a physical system with a known causal graph \citep{causalchambers} (\Cref{app:causal_chambers}). Backdoor meta-learners outperform proximal learners, including \modelname. We interpret this cautiously: \modelname is designed for CATE, this benchmark has only twelve configurations of one one-dimensional system and moderate sample size ($1\,000$ observations each).
\section{Conclusion}
\label{sec:conclusion}

We introduced \modelname, an amortized estimator for CATE under the proximal learning assumptions. Existing causal foundation models address settings such as backdoor, instrumental-variable, and frontdoor identification, but are not designed to exploit proxies for proximal identification. Rather than solving an ill-posed bridge equation separately for each new dataset, \modelname learns from synthetic episodes to estimate CATE from observational data, with regularization implied by the prior. Its CATE posterior also provides uncertainty estimates, though credible intervals can be overconfident beyond the training prior. Together with models for other identification settings, it gives practitioners a fast way to compare estimates under different plausible identifying assumptions.

\subsection*{Acknowledgements}

AK acknowledges funding support from His Majesty’s Government in the development of
this research. JZ is funded by the Pioneer Center for Statistical and computational Methods for Advanced Research to Transform Biomedicine (SMARTbiomed), DNRF grant number P4 and a 2026 Jesus College, University of Oxford, SCR Major Research Grant. RS is partially funded by the EPSRC Open Fellowship EP/W024330/1 and, by the Causality in Healthcare AI (CHAI) Hub, EPSRC grant number EP/Y028856/1. AL was funded through a Patient-Centered Outcomes Research Institute (PCORI) Award (ME-2024C2-
39990).

\printbibliography[title={References}]

\clearpage

\begin{appendices}
\numberwithin{figure}{section}
\numberwithin{table}{section}
\addtocontents{toc}{\protect\etocsettocdepth{3}}
\etocsettocstyle{\section*{Appendix}}{}
\tableofcontents
\newpage

\section{Consistency result}
\label{sec:consistency_result}

We adapt the observational-equivalence argument of \cite{balazadehCausalPFNAmortizedCausal2025} to the CATE functional. The result characterizes the oracle Bayesian learner under the SCM prior $\Pi$ used in the main text. Its application to \modelname is elaborated in Section~\ref{sec:consistency_scope}.

\subsection{Setup and identification}

Let $\psi$ range over a standard Borel parameter space $\Psi$. Each SCM induces a joint law $P^\psi$ of the observed variables and both potential outcomes, with $A\in\{0,1\}$ and $Y=Y(A)$. Write $O=(X,Z,W,A,Y)$ for the observed variables, taking values in a standard Borel space $\mathcal{O}$, and $R(\psi):=P^\psi_{\mathrm{obs}}$ for their law. Under the joint Bayesian model, first draw $\psi\sim\Pi$ and then draw $O_1,O_2,\ldots\mid\psi$ independently from $R(\psi)$. The observational context is $\mathcal{D}_N=(O_1,\ldots,O_N)$.

\begin{assumption}[Measurability]
\label{ass:measurability}
The map $\psi\mapsto P^\psi$ is measurable, and the image $\mathcal{R}:=\{R(\psi):\psi\in\Psi\}$ is a Borel subset of $\mathcal{P}(\mathcal{O})$, the standard Borel space of probability laws on $\mathcal{O}$. We choose jointly measurable versions of
\begin{equation}
\label{eq:cate_functional}
\mu_a(x;\psi) :=\mathbb{E}_\psi[Y(a)\mid X=x],\qquad Q(x;\psi):=\mu_1(x;\psi)-\mu_0(x;\psi).
\end{equation}
\end{assumption}

Throughout what follows, fix an admissible query $x$ at which these conditional means are defined and the stated assumptions apply. Our results apply pointwise in $x$ and we treat the query as fixed (if instead a query is sampled from an SCM-dependent covariate distribution and provides additional evidence, the same argument uses the conditional prior $\Pi(d\psi\mid X_{\mathrm{query}}=x)$, assuming the context and query are conditionally independent given $\psi$ and the assumptions below hold for that prior.)

\begin{assumption}[Integrability]
\label{ass:integrability}
At the query under consideration,
\begin{equation}
\label{eq:integrability}
\int_\Psi |Q(x;\psi)|\,\Pi(d\psi)<\infty.
\end{equation}
\end{assumption}
Integrability of both CEPOs is sufficient, since $|Q(x;\psi)|\leq|\mu_1(x;\psi)|+|\mu_0(x;\psi)|$.

\begin{definition}[Observational equivalence and CATE-identifiability]
\label{def:quotient_space}
\label{def:cate_ident}
Two SCMs are observationally equivalent when $R(\psi)=R(\psi')$. The prior is CATE-identifiable at $x$ if there is a measurable map $F_{x}:\mathcal{R}\to\mathbb{R}$ such that $Q(x;\psi)=F_{x}(R(\psi))$ for $\Pi$-almost every $\psi$.
\end{definition}

\begin{definition}[CEPO-identifiability]
\label{def:cepo-ident}
The prior is CEPO-identifiable at $x$ if, for each $a\in\{0,1\}$, there is a measurable $F_{a,x}$ with $\mu_a(x;\psi)=F_{a,x}(R(\psi))$ for $\Pi$-almost every $\psi$.
\end{definition}
CEPO-identifiability implies CATE-identifiability by taking $F_{x}=F_{1,x}-F_{0,x}$. The converse is unnecessary and need not hold, since a common unidentified component of the two means can cancel in their difference.

The oracle CATE posterior and its mean are
\begin{equation}
\label{eq:cate_posterior}
\begin{aligned}
\Pi^Q(B\mid x,\mathcal{D}_N) &:=\int_\Psi \mathbf{1}\{Q(x;\psi)\in B\} \,\Pi(d\psi\mid\mathcal{D}_N),\qquad B\in\mathcal{B}(\mathbb{R}),\\
\overline Q_N(x) &:=\int_\Psi Q(x;\psi)\,\Pi(d\psi\mid\mathcal{D}_N).
\end{aligned}
\end{equation}

\subsection{Consistency of the oracle CATE posterior}

Working directly with $\mathcal{R}$ represents observational equivalence classes by their common law. Let $\nu$ be the distribution of $R(\psi)$ under $\Pi$, and choose a measurable version of
\begin{equation}
\label{eq:cate_under_equ}
g_{x}(r):= \mathbb{E}_{\Pi}[Q(x;\psi)\mid R(\psi)=r].
\end{equation}

\begin{lemma}[Limit within observational equivalence classes]
\label{lemma:ident}
Under Assumptions~\ref{ass:measurability}--\ref{ass:integrability}, for $\Pi$-almost every $\psi^\star$, i.i.d. observations from $R(\psi^\star)$ satisfy
\begin{equation}
\label{eq:cate_class_limit}
\overline Q_N(x)\xrightarrow[N\to\infty]{\mathrm{a.s.}} g_{x}(R(\psi^\star)).
\end{equation}
\end{lemma}

\noindent\textit{Proof.} Conditional on $R$, the law of the entire observation sequence is $R^{\otimes\infty}$ and does not depend further on $\psi$. Thus
\begin{equation}
\label{equ:cond_independent}
\overline Q_N(x) =\mathbb{E}[Q(x;\psi)\mid\mathcal{D}_N] =\mathbb{E}[g_{x}(R)\mid\mathcal{D}_N].
\end{equation}
Moreover, conditional Jensen's inequality gives $\int|g_{x}|\,d\nu\leq\int|Q(x;\psi)|\,d\Pi<\infty$. The model indexed by $r\in\mathcal{R}$, with sampling law $r$, is measurable and identifiable by construction. Theorem~2.2 of \cite{miller2018detailedtreatmentdoobstheorem} therefore gives $\mathbb{E}[g_{x}(R)\mid\mathcal{D}_N] \to g_{x}(R)$ almost surely under the joint Bayesian law. Intersecting the probability-one events on which \eqref{equ:cond_independent} holds for each $N$ and disintegrating with respect to $\psi$ proves the stated parameterwise conclusion, as in Corollary~2.3 of \cite{miller2018detailedtreatmentdoobstheorem}.\hfill$\square$

\begin{theorem}[CATE identification and oracle consistency]
\label{theorem:ident}
Under Assumptions~\ref{ass:measurability}--\ref{ass:integrability}, the following are equivalent at the fixed query $x$:
\begin{enumerate}
\item The prior is CATE-identifiable at $x$.
\item For $\Pi$-almost every $\psi^\star$, i.i.d. observations from $R(\psi^\star)$ satisfy
\begin{equation}
\label{eq:oracle_cate_consistency}
\overline Q_N(x)\xrightarrow[N\to\infty]{\mathrm{a.s.}} Q(x;\psi^\star).
\end{equation}
\end{enumerate}
Under these conditions the exact CATE posterior also concentrates: for $\Pi$-almost every $\psi^\star$ and every $\epsilon>0$,
\begin{equation}
\label{eq:cate_concentration}
\Pi^Q\!\left(\{q:|q-Q(x;\psi^\star)|>\epsilon\} \mid x,\mathcal{D}_N\right)\xrightarrow[N\to\infty]{\mathrm{a.s.}}0 .
\end{equation}
\end{theorem}

\noindent\textit{Proof.} If $Q(x;\psi)=F_{x}(R(\psi))$ almost surely, then $g_{x}=F_{x}$ $\nu$-almost everywhere. Intersecting the pullback of this set with the full-prior-measure set in Lemma~\ref{lemma:ident} proves consistency. Conversely, consistency and Lemma~\ref{lemma:ident} imply $Q(x;\psi)=g_{x}(R(\psi))$ for $\Pi$-almost every $\psi$. The measurable function $g_{x}$ therefore witnesses CATE-identifiability.

For posterior concentration, apply the same Doob argument to the bounded functions $r\mapsto\mathbf{1}\{F_{x}(r)\leq t\}$ for all rational $t$. Intersecting the resulting countably many full-measure sets gives convergence of the posterior probabilities of these half-lines to their values at $Q(x;\psi^\star)$. For any $\epsilon>0$, choose rational thresholds on either side of $Q(x;\psi^\star)$, each within $\epsilon$ of it. The posterior mass outside these thresholds tends to zero, proving \eqref{eq:cate_concentration}.\hfill$\square$

\subsection{Proximal identification}

Suppose the proximal assumptions hold $\Pi$-almost surely, including the completeness and bridge-existence conditions needed for the conditional proximal g-formula at the fixed query $x$. For an appropriate outcome bridge $h_0$ and the chosen identifying versions of the conditional means, the formula yields
\begin{equation}
\label{eq:proximal_cepo_cate}
\begin{aligned}
\mu_a(x;\psi) &=\mathbb{E}_{\psi}[h_0(W,a,x) \mid X=x],\\
Q(x;\psi) &=\mathbb{E}_{\psi}[h_0(W,1,x) -h_0(W,0,x) \mid X=x].
\end{aligned}
\end{equation}
The right-hand sides are identified by the observational law as their measurable representations give CEPO- and hence CATE-identifiability. Theorem~\ref{theorem:ident} consequently applies. Equivalently, when both CEPO posterior means are consistent and integrable, linearity gives
\begin{equation}
\label{eq:cepo_cate_corollary}
\overline Q_N(x) =\mathbb{E}[\mu_1(x;\psi)\mid\mathcal{D}_N] -\mathbb{E}[\mu_0(x;\psi)\mid\mathcal{D}_N] \xrightarrow[N\to\infty]{\mathrm{a.s.}} Q(x;\psi^\star).
\end{equation}
Note that this identity does not require the trained model to estimate the CEPO.

\subsection{Application to \modelname}
\label{sec:consistency_scope}

For independent downstream simulations from the SCM's conditional potential-outcome law given $X=x$, independent of the context conditional on $\psi$, the target $\widehat Q_K(x)=K^{-1}\sum_{k=1}^K(Y_k(1)-Y_k(0))$ satisfies
\begin{equation}
\label{eq:mc_cate_unbiased}
\mathbb{E}[\widehat Q_K(x)\mid\psi,\mathcal{D}_N] =Q(x;\psi),\qquad \mathbb{E}[\widehat Q_K(x)\mid\mathcal{D}_N] =\overline Q_N(x).
\end{equation}
With finite second moments, population squared-error regression on these labels therefore targets the oracle posterior mean even at finite $K$. Recovering the full posterior from noisy labels additionally requires a correctly specified and identifiable measurement-error model, or a limit in which Monte Carlo error vanishes. Our theoretical result does not establish exact recovery by the finite-$K$ Gaussian noise-aware objective or its combination with a mean penalty but is instead concerned with an oracle learner.

To make this explicit, let $\widehat Q_{\theta,N}(x)$ be the mean predicted by \modelname. Then
\begin{equation}
\label{eq:amortization_error}
\begin{aligned}
|\widehat Q_{\theta,N}(x)-Q(x;\psi^\star)| \leq{}&|\widehat Q_{\theta,N}(x)-\overline Q_N(x)|\\
&+|\overline Q_N(x)-Q(x;\psi^\star)|.
\end{aligned}
\end{equation}
Theorem~\ref{theorem:ident} controls the second term. Consistency of the learned estimator also requires the first term (comprising supervision, representation, and training errors) to vanish as context size grows, which extends beyond the scope of our theoretical contributions and would require substantially more developments in the theory of in-context learning.

\clearpage
\section{Details on the proximal prior}
\label{app:proximal_prior}

\subsection{Step-by-step procedure}
\label{app:prior_procedure}

We generate each synthetic dataset from an independently sampled structural causal model (SCM) compatible with the proximal DAG. The full routine is sketched in \Cref{fig:prior_pipeline} and specified in \Cref{alg:prior}; this section walks through it and fixes notation used across \Cref{app:proximal_prior}.

Throughout, a dataset consists of $T$ units (the sequence length) and the target estimand is the conditional average treatment effect $Q(x)=\E[Y(1)-Y(0)\mid X=x]$. Because $Q(x)$ marginalises over all unit-level heterogeneity beyond $X$, we approximate the conditional expected potential outcomes (CEPOs) $\mu_a(x)=\E[Y(a)\mid X=x]$ by Monte Carlo integration: every unit is replicated $K$ times, sharing its covariate value $x$ while drawing fresh exogenous noise—and therefore fresh latent-confounder and proxy realizations—and the potential outcomes are averaged over the $K$ replicates.

\paragraph{Once per dataset: structure.}
We first sample the latent and observed dimensionalities, drawing the confounder dimension $d_U$ from a law skewed toward small values and then the proxy dimensions $d_W, d_Z$ conditionally on $d_U$ so as to respect the completeness heuristic $d_W,d_Z\ge d_U$, with the covariate dimension $d_X$ drawn independently and uniformly; the sampling laws and the resulting induced marginals are described below in \Cref{app:prior_dimensions}. We then instantiate the causal mechanisms $f_X, f_U, f_W, f_Z, f_A, f_Y$ as independently initialised random multilayer perceptrons in the style of TabICL \citep{quTabICLTabularFoundation2025}, together with an auxiliary confounding-duplication network $f_{\mathrm{aug}}$ (described below), and sample the per-dataset hyperparameters (network depth, width, activations, noise scales; see \Cref{app:prior_hyperparams}). Mechanisms, dimensions and hyperparameters are held fixed across all $T$ units and their $K$ replicates, so the entire dataset is drawn from the \emph{same} SCM.

\paragraph{Per replicate: the topological pass.}
The baseline covariates are generated once,
\begin{equation}
  X = \mathrm{Corrupt}\big(f_X(\varepsilon_X)\big) \in \R^{T\times d_X},
  \label{eq:x}
\end{equation}
and each row is tiled $K$ times. The remaining variables are sampled in topological order over all $T\cdot K$ rows in a single vectorised pass, with mutually independent exogenous noise $\varepsilon_{(\cdot)}$ drawn afresh for every row:
\begin{align}
  U &= \mathrm{Corrupt}\big(f_U(\tilde{X}_U,\,\varepsilon_U)\big), &
  W &= \mathrm{Corrupt}\big(f_W(\tilde{U},\tilde{X}_W,\,\varepsilon_W)\big), &
  Z &= \mathrm{Corrupt}\big(f_Z(\tilde{U},\tilde{X}_Z,\,\varepsilon_Z)\big),
  \label{eq:uwz}
\end{align}

where $\tilde{(\cdot)}$ denotes standardisation to zero mean and unit variance and $\mathrm{Corrupt}(\cdot)$ applies the feature postprocessing of \Cref{app:prior_postprocess} (optional categorisation, outlier clamping, and scaling). The exogenous-noise blocks $\varepsilon_{(\cdot)}$ are sized proportionally to each mechanism's non-noise input width via a per-dataset share parameter (see \Cref{tab:prior_sampled_hp}), so the noise level does not shrink as signal width grows. Crucially, $Z$ is excluded from the inputs of $f_W$ and $f_Y$, and $W$ is excluded from those of $f_Z$ and $f_A$, so the proxy exclusion restrictions of \Cref{as:proxy_exclusion} hold by construction for any realisation of the random mechanisms.

Notation: $\tilde{X}_U, \tilde{X}_W, \tilde{X}_Z, \tilde{X}_A, \tilde{X}_Y$ denote per-mechanism random subsets of $\tilde{X}$'s columns, each drawn independently with its own per-dataset keep probability (see \Cref{app:prior_mechanisms}); every mechanism that reads $X$, including $f_U$, sees one of these subsets rather than the full block. Likewise, we'll use the notation $\tilde{W}_Y\subseteq\tilde{W}$ and $\tilde{Z}_A\subseteq\tilde{Z}$ to denote random column subsets for the direct proxy edges.

\paragraph{Mechanism wiring: confounding control and edge subsampling.}
Three per-dataset knobs then set which columns actually feed each mechanism. First, to control how strongly the confounder influences each of $A$ and $Y$ -- independently of $d_U$, which in a random MLP would otherwise pin the confounding to $U$'s accidental share of the input width -- we target a share $s_A,s_Y\sim\mathrm{Beta}(1.2,2.8)$ of each mechanism's inputs to be $U$-derived and convert it to a column count $k=\max(2,\operatorname{round}(\tfrac{s_c}{1-s_c}n))$ with $s_c=\min(s,0.75)$ and $n$ the width of the non-$U$ inputs. The mechanism is then fed either a subsample of $U$'s columns (when $k\le d_U$) or all of $U$ plus $k-d_U$ columns of a shared nonlinear re-encoding $U^{\mathrm{aug}}=f_{\mathrm{aug}}(\tilde U)$ (when $k>d_U$); we write $U^{[A]},U^{[Y]}$ for these confounding-controlled inputs. We detail this knob in \Cref{app:prior_mechanisms}.

Second and third, the direct covariate and proxy edges are made per-column optional, mirroring each other. Every mechanism downstream of $X$ -- i.e.\ $f_U, f_W, f_Z, f_A, f_Y$ -- receives an independent random subset $\tilde{X}_U,\tilde{X}_W,\tilde{X}_Z,\tilde{X}_A,\tilde{X}_Y\subseteq\tilde{X}$, each column entering with its own per-dataset keep probability  $\sim\mathrm{Unif}[0,1]$. Similarly, each column of $W$ (resp.\ $Z$) enters $f_Y$ (resp.\ $f_A$) with its own per-dataset keep probability $\sim\mathrm{Unif}[0,1]$, giving subsets $\tilde{W}_Y\subseteq\tilde{W}$, $\tilde{Z}_A\subseteq\tilde{Z}$. These masks alter the random-function prior without adding graph edges: they only remove direct \(X\), \(W\), or \(Z\) inputs, while the \(U\to W\) and \(U\to Z\) inputs and the completeness heuristics remain unchanged.

\paragraph{Potential outcomes, heterogeneity remix, and the observed dataset.}
The outcome mechanism returns both potential outcomes $Y_0^{\mathrm{raw}},Y_1^{\mathrm{raw}}$ from a shared forward pass, which are remixed via a per-dataset heterogeneity dial $\gamma\sim\mathrm{Unif}[0,1]$ into $Y_0,Y_1$, ranging continuously from the raw, fully heterogeneous effect ($\gamma=1$) to a single constant effect ($\gamma=0$); we detail and analyse this construction in \Cref{app:prior_outcome}. Averaging the remixed potential outcomes over the $K$ replicates of unit $i$ gives the CEPO estimates $\hat\mu_a(x_i)=\tfrac{1}{K}\sum_{k=1}^K y_{i,k}(a)$, and hence the CATE target $\hat Q(x_i)=\hat\mu_1(x_i)-\hat\mu_0(x_i)$. To form the observed dataset, one replicate $k_i\sim\mathrm{Unif}\{1,\dots,K\}$ is retained per unit, yielding the tuple $\big(x_i, w_{i,k_i}, z_{i,k_i}, a_{i,k_i}, y_{i,k_i}\big)$ with $y_{i,k_i}=y_{i,k_i}(a_{i,k_i})$. The transformer sees only these observational components; the CEPO targets $\hat\mu_0,\hat\mu_1$ are used purely as training targets and are never observed.

\paragraph{Proximal conditions.}
The construction enforces latent ignorability and the proxy exclusion restrictions of \Cref{as:latent_ignorability,as:proxy_exclusion} through the mechanism input restrictions and independent exogenous noises. Positivity (\Cref{as:positivity}) is enforced by the clipped treatment propensity (\Cref{app:prior_treatment}). Completeness (\Cref{as:completeness}) cannot be guaranteed for arbitrary nonlinear mechanisms; instead, the prior imposes numeric and categorical proxy-profile heuristics (\Cref{app:prior_dimensions,app:prior_postprocess}). Empirical diagnostics of the resulting overlap, exclusion restrictions, confounding strength, and proxy relevance are collected in \Cref{app:prior_audit}.

\begin{algorithm}[t]
\caption{Generation of one proximal dataset.}
\label{alg:prior}
\begin{algorithmic}[1]
\Require sequence length $T$, Monte Carlo count $K$, hyperparameters.
\State sample $d_U$ from \eqref{eq:dU_triangular}; sample $d_W,d_Z,d_X$ from \eqref{eq:proxy_dims}
\State instantiate random MLP mechanisms $f_X,f_U,f_W,f_Z,f_A,f_Y,f_{\mathrm{aug}}$
\State sample per-dataset hyperparameters: noise shares, confounding shares $s_A,s_Y$, edge keep-probabilities, overlap knobs $s,b$, $\gamma\sim\mathrm{Unif}(0,1)$
\Statex\vspace{2pt}
\Repeat
  \State $X \gets \mathrm{Corrupt}(f_X(\varepsilon_X)) \in \R^{T\times d_X}$, tiled $K$ times and standardised to $\tilde X$
  \State subsample covariate inputs $\tilde X_U,\tilde X_W,\tilde X_Z,\tilde X_A,\tilde X_Y\subseteq\tilde X$
  \State draw independent noise $\varepsilon_U,\varepsilon_W,\varepsilon_Z,\varepsilon_A,\varepsilon_Y$ over all $TK$ rows
  \State $U \gets \mathrm{Corrupt}(f_U(\tilde X_U,\varepsilon_U))$, standardised to $\tilde U$
  \State $W \gets \mathrm{Corrupt}(f_W(\tilde U,\tilde X_W,\varepsilon_W))$, standardised to $\tilde W$ \Comment{$Z\not\to f_W$}
  \State $Z \gets \mathrm{Corrupt}(f_Z(\tilde U,\tilde X_Z,\varepsilon_Z))$, standardised to $\tilde Z$ \Comment{$W\not\to f_Z$}
  \State subsample direct-proxy inputs $\tilde Z_A\subseteq\tilde Z$ and $\tilde W_Y\subseteq\tilde W$
  \State form $U^{[A]},U^{[Y]}$ from $\tilde U$ and $U^{\mathrm{aug}}=f_{\mathrm{aug}}(\tilde U)$ via confounding shares $s_A,s_Y$
  \State $\ell \gets f_A(U^{[A]},\tilde X_A,\tilde Z_A,\varepsilon_A)$;\quad $\tilde\ell \gets (\ell-\bar\ell)/(\operatorname{sd}(\ell)+10^{-8})$
  \State $\pi \gets \operatorname{clip}\big(\sigma(s\,\tilde\ell+b),\,\epsilon,\,1-\epsilon\big)$
  \State $A \sim \mathrm{Bernoulli}(\pi)$
  \State $(Y_0^{\mathrm{raw}},Y_1^{\mathrm{raw}}) \gets f_Y(U^{[Y]},\tilde X_Y,\tilde W_Y,\varepsilon_Y)$
  \State draw $\alpha\sim\mathrm{Unif}(0,1)^{TK}$; remix to $(Y_0,Y_1)$ via \eqref{eq:remix}
  \State $\hat\mu_a(x_i) \gets \tfrac1K\sum_{k=1}^K y_{i,k}(a)$ for $a\in\{0,1\}$
  \State select $k_i\sim\mathrm{Unif}\{1,\dots,K\}$ per unit; keep observed $\{x_i,w_i,z_i,a_i,y_i\}$
  \State remove constant columns from $X,U,W,Z$; update active dims $d_X,d_U,d_W,d_Z$ \label{line:constant_column_filter}
\Until{$d_X,d_W,d_Z>0$,  $\min(d_W,d_Z)\ge d_U>0$ and categorical dominance of $W,Z$ over $U$.}
\Statex\vspace{2pt}
\State \Return observed $\{x_i,w_i,z_i,a_i,y_i\}_{i=1}^T$ and targets $\{\hat\mu_0(x_i),\hat\mu_1(x_i)\}_{i=1}^T$
\end{algorithmic}
\end{algorithm}

\begin{figure}[ht]
  \centering
  \begin{adjustbox}{max totalsize={1.2\linewidth}{0.78\textheight}}
    \input{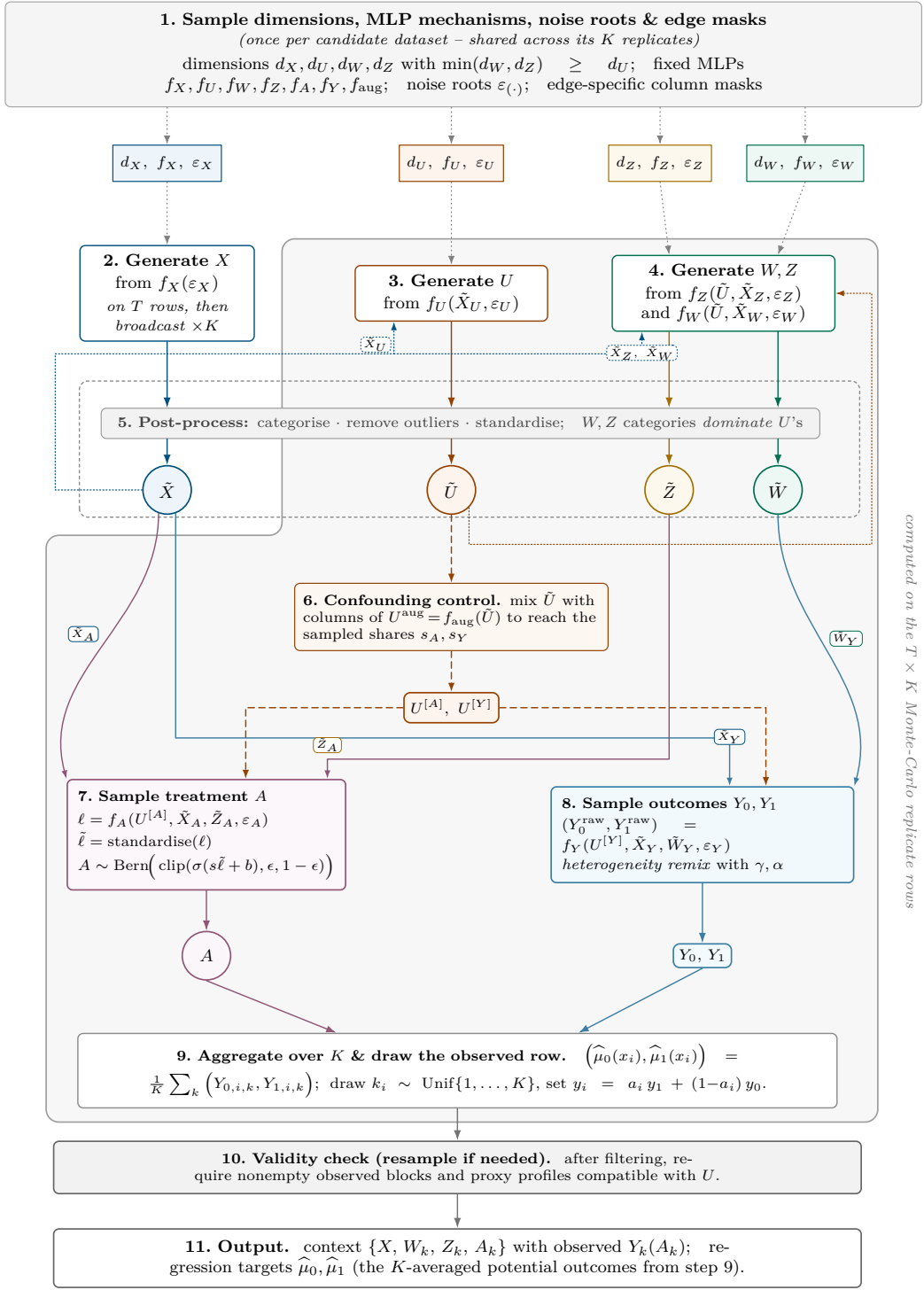}
  \end{adjustbox}
  \caption{Detailed diagram of the proximal prior's per-dataset generation pipeline.}
  \label{fig:prior_pipeline}
\end{figure}

\FloatBarrier
\subsection{Generative model components}
\label{app:prior_model}

\subsubsection{Parameters of the prior}
\label{app:prior_hyperparams}

Three sets of parameters govern the prior. \emph{Structural} parameters (\Cref{tab:prior_structural_hp}) fix the shape of a dataset -- the sampling laws of the four block dimensions, the number of rows, and the number of Monte Carlo replicates behind the $\mu_0/\mu_1$ targets -- and are the reference point for the dimension analysis of \Cref{app:prior_dimensions}. \emph{Fixed} hyperparameters (\Cref{tab:prior_fixed_hp}) are shared constants, identical across the whole prior. \emph{Sampled} hyperparameters (\Cref{tab:prior_sampled_hp}) are instead drawn once per dataset from a meta-distribution -- typically a log-scaled truncated normal whose own centre is itself uniform on a log scale -- so that mechanism depth, width, activation, and noise scale, as well as the causal knobs governing overlap and confounding strength, vary \emph{across} the prior rather than being pinned to a single value.

The fixed table lists only the constants that pin down behaviour described qualitatively elsewhere in this appendix; the generator additionally exposes on/off switches for the mechanisms of the following subsections -- the dimension heuristic (\Cref{app:prior_dimensions}), categorical completeness (\Cref{app:prior_postprocess}), the confounding-share subsampling of $U$'s inputs to $f_A$ and $f_Y$, and the per-column subsampling of the direct and covariate edges (\Cref{app:prior_mechanisms}). All are enabled throughout the prior.

\begin{table}
\centering
\caption{Structural parameters of the prior: the block dimensions and the dataset shape. Each dimension is drawn per dataset from the law shown, stated in terms of the ceilings listed below it. The proxy dimensions are drawn conditionally on $d_U$ so that the completeness heuristic of \Cref{app:prior_dimensions}.}
\label{tab:prior_structural_hp}
\small
\begin{tabularx}{\textwidth}{>{\raggedright\arraybackslash}X>{\raggedright\arraybackslash}X>{\raggedright\arraybackslash}X}
\toprule
hyperparameter & value & description \\
\midrule
\multicolumn{3}{l}{\textbf{Block dimensions}} \\
\texttt{dim\_\allowbreak X} & $\mathrm{Unif}\{0, \dots, \texttt{max\_dim\_X}\}$ & observed baseline covariate dimension \\
\texttt{dim\_\allowbreak U} & triangular on $\{1, \dots, \texttt{max\_dim\_U}\}$ & latent confounder dimension \\
\texttt{dim\_\allowbreak W} & $\mathrm{Unif}\{d_U, \dots, \texttt{max\_dim\_W}\}$ & outcome-inducing proxy dimension \\
\texttt{dim\_\allowbreak Z} & $\mathrm{Unif}\{d_U, \dots, \texttt{max\_dim\_Z}\}$ & treatment-inducing proxy dimension \\
\midrule
\multicolumn{3}{l}{\textbf{Dimension bounds}} \\
\texttt{max\_\allowbreak dim\_\allowbreak X} & $10$ & covariate-dimension ceiling \\
\texttt{max\_\allowbreak dim\_\allowbreak U} & $10$ & confounder-dimension ceiling ($d^{\max}$ of \Cref{app:prior_dimensions}) \\
\texttt{max\_\allowbreak dim\_\allowbreak W} & $10$ & outcome-confounding proxy dimension ceiling \\
\texttt{max\_\allowbreak dim\_\allowbreak Z} & $10$ & treatment-confounding proxy dimension ceiling \\
\midrule
\multicolumn{3}{l}{\textbf{Dataset}} \\
\texttt{context\_\allowbreak rows} & $\mathrm{Unif}\{64, \dots, 2047\}$ & context rows per episode \\
\texttt{n\_\allowbreak query} & $128$ & fixed query rows per episode \\
\texttt{mc\_\allowbreak cate} & $250$ & MC replicates $K$ per row behind the $\mu_0/\mu_1$ targets \\
\texttt{mc\_\allowbreak target\_\allowbreak rows} & query rows only & rows carrying MC CEPO targets \\
\bottomrule
\end{tabularx}
\end{table}

\begin{table}
\centering
\caption{Fixed hyperparameters of the prior, held constant across every dataset. Listed are the constants that pin down behaviour the surrounding text describes only qualitatively; switches whose shipped value is that behaviour itself are stated in the text instead.}
\label{tab:prior_fixed_hp}
\small
\begin{tabularx}{\textwidth}{>{\raggedright\arraybackslash}Xl>{\raggedright\arraybackslash}X}
\toprule
hyperparameter & value & description \\
\midrule
\multicolumn{3}{l}{\textbf{Mechanism}} \\
\texttt{noise\_\allowbreak std} & $0.001$ & tiny fixed per-layer dither so continuous columns realise distinct values \\
\midrule
\multicolumn{3}{l}{\textbf{Feature corruption}} \\
\texttt{max\_\allowbreak categories} & $50$ & cap on categories per feature \\
\texttt{proxy\_\allowbreak cat\_\allowbreak cover\_\allowbreak prob} & $0.5$ & chance a proxy mirrors a $U$ category (vs.\ a continuous wildcard) \\
\midrule
\multicolumn{3}{l}{\textbf{Positivity}} \\
\texttt{propensity\_\allowbreak epsilon} & $0.001$ & positivity clamp on the propensity \\
\midrule
\multicolumn{3}{l}{\textbf{Completeness}} \\
\texttt{completeness\_\allowbreak max\_\allowbreak attempts} & $50$ & resample budget before giving up \\
\bottomrule
\end{tabularx}
\end{table}

\begin{table}
\centering
\caption{Per-dataset sampled hyperparameters of the prior. Each is drawn once per dataset from the listed meta-distribution; log-scaled truncated normals draw their centre log-uniformly over the stated mean range and add the lower bound as an offset.}
\label{tab:prior_sampled_hp}
\small
\begin{tabularx}{\textwidth}{>{\raggedright\arraybackslash}X>{\raggedright\arraybackslash}X>{\raggedright\arraybackslash}X}
\toprule
hyperparameter & sampling law & description \\
\midrule
\multicolumn{3}{l}{\textbf{Mechanism}} \\
\texttt{num\_\allowbreak layers} & log-scaled trunc.\ normal, mean $\sim\log\mathrm{Unif}[1, 6]$, int $\ge 2$ & mechanism MLP depth \\
\texttt{hidden\_\allowbreak dim} & log-scaled trunc.\ normal, mean $\sim\log\mathrm{Unif}[5, 130]$, int $\ge 4$ & mechanism MLP width \\
\texttt{mlp\_\allowbreak activations} & categorical over 54 random activation functions & activation, drawn from a random library \\
\texttt{init\_\allowbreak std} & log-scaled trunc.\ normal, mean $\sim\log\mathrm{Unif}[0.01, 10]$, real $\ge 0$ & weight-initialisation scale \\
\texttt{block\_\allowbreak wise\_\allowbreak dropout} & categorical $\{True, False\}$ (random weights) & block-sparse vs.\ dense weight init \\
\texttt{mlp\_\allowbreak dropout\_\allowbreak prob} & scaled Beta on $[0, 0.9]$, shapes $\sim\mathrm{Unif}[0.1, 5]$ & dropout rate (dense init only) \\
\midrule
\multicolumn{3}{l}{\textbf{Noise}} \\
\texttt{sampling} & categorical $\{\texttt{normal}, \texttt{mixed}, \texttt{uniform}\}$ (random weights) & family of the root cause variables \\
\texttt{pre\_\allowbreak sample\_\allowbreak cause\_\allowbreak stats} & categorical $\{True, False\}$ (random weights) & pre-sample per-cause mean/std (normal sampling) \\
\texttt{noise\_\allowbreak share\_\allowbreak of\_\allowbreak input\_\allowbreak in\_\allowbreak U} & $\mathrm{Beta}(1.2, 2.8)$ & target noise share of $f_U$ inputs ($\varepsilon_U$ width) \\
\texttt{noise\_\allowbreak share\_\allowbreak of\_\allowbreak input\_\allowbreak in\_\allowbreak W} & $\mathrm{Beta}(1.2, 2.8)$ & target noise share of $f_W$ inputs ($\varepsilon_W$ width) \\
\texttt{noise\_\allowbreak share\_\allowbreak of\_\allowbreak input\_\allowbreak in\_\allowbreak Z} & $\mathrm{Beta}(1.2, 2.8)$ & target noise share of $f_Z$ inputs ($\varepsilon_Z$ width) \\
\texttt{noise\_\allowbreak share\_\allowbreak of\_\allowbreak input\_\allowbreak in\_\allowbreak A} & $\mathrm{Beta}(1.2, 2.8)$ & target noise share of $f_A$ inputs ($\varepsilon_A$ width) \\
\texttt{noise\_\allowbreak share\_\allowbreak of\_\allowbreak input\_\allowbreak in\_\allowbreak Y} & $\mathrm{Beta}(1.2, 2.8)$ & target noise share of $f_Y$ inputs ($\varepsilon_Y$ width) \\
\midrule
\multicolumn{3}{l}{\textbf{Causal}} \\
\texttt{propensity\_\allowbreak shift} & $\mathrm{Unif}[-2.5, 2.5]$ & treatment-logit bias (marginal treated fraction) \\
\texttt{propensity\_\allowbreak scale} & log-scaled trunc.\ normal, mean $\sim\log\mathrm{Unif}[0.5, 3]$, real $\ge 0.1$ & treatment-logit gain (overlap vs.\ confounding) \\
\texttt{confounding\_\allowbreak share\_\allowbreak of\_\allowbreak U\_\allowbreak in\_\allowbreak Y} & $\mathrm{Beta}(1.2, 2.8)$ & target share of $f_Y$ inputs drawn from $U$ \\
\texttt{confounding\_\allowbreak share\_\allowbreak of\_\allowbreak U\_\allowbreak in\_\allowbreak A} & $\mathrm{Beta}(1.2, 2.8)$ & target share of $f_A$ inputs drawn from $U$ \\
\texttt{direct\_\allowbreak edge\_\allowbreak share\_\allowbreak of\_\allowbreak W\_\allowbreak in\_\allowbreak Y} & $\mathrm{Beta}(1, 1)$ & per-column keep probability of the direct $W \to Y$ edge \\
\texttt{direct\_\allowbreak edge\_\allowbreak share\_\allowbreak of\_\allowbreak Z\_\allowbreak in\_\allowbreak A} & $\mathrm{Beta}(1, 1)$ & per-column keep probability of the direct $Z \to A$ edge \\
\texttt{covariate\_\allowbreak share\_\allowbreak of\_\allowbreak X\_\allowbreak in\_\allowbreak U} & $\mathrm{Beta}(1, 1)$ & per-column keep probability of the direct $X \to U$ edge \\
\texttt{covariate\_\allowbreak share\_\allowbreak of\_\allowbreak X\_\allowbreak in\_\allowbreak W} & $\mathrm{Beta}(1, 1)$ & per-column keep probability of the direct $X \to W$ edge \\
\texttt{covariate\_\allowbreak share\_\allowbreak of\_\allowbreak X\_\allowbreak in\_\allowbreak Z} & $\mathrm{Beta}(1, 1)$ & per-column keep probability of the direct $X \to Z$ edge \\
\texttt{covariate\_\allowbreak share\_\allowbreak of\_\allowbreak X\_\allowbreak in\_\allowbreak A} & $\mathrm{Beta}(1, 1)$ & per-column keep probability of the direct $X \to A$ edge \\
\texttt{covariate\_\allowbreak share\_\allowbreak of\_\allowbreak X\_\allowbreak in\_\allowbreak Y} & $\mathrm{Beta}(1, 1)$ & per-column keep probability of the direct $X \to Y$ edge \\
\midrule
\multicolumn{3}{l}{\textbf{Feature corruption}} \\
\texttt{cat\_\allowbreak col\_\allowbreak prob} & $\mathrm{Beta}(1, 3)$ & per-column categorisation probability \\
\bottomrule
\end{tabularx}
\end{table}

\FloatBarrier
\subsubsection{Dimensions \& dimension heuristic}
\label{app:prior_dimensions}

The block dimensions are sampled according to the completeness heuristic introduced in \Cref{app:prior_procedure}. The confounder dimension is drawn from a triangular (linearly decreasing) law on $\{1,\dots,d_U^{\max}\}$,
\begin{equation}
  \Pr[d_U = d] \;\propto\; (d_U^{\max}+1-d),
  \label{eq:dU_triangular}
\end{equation}
which skews $d_U$ toward small values. Conditional on \(d_U\), the two proxy dimensions are sampled uniformly subject to \(d_W,d_Z\ge d_U\), while the covariate dimension \(d_X\) is sampled independently and uniformly:
\begin{equation}
  d_W \sim \mathrm{Unif}\{d_U,\dots,d_W^{\max}\},\quad
  d_Z \sim \mathrm{Unif}\{d_U,\dots,d_Z^{\max}\},\quad
  d_X \sim \mathrm{Unif}\{d_X^{\min},\dots,d_X^{\max}\}.
  \label{eq:proxy_dims}
\end{equation}

This conditional-uniform sampling has a convenient consequence in the common-bound configuration used here, where $d_U$, $d_W$, and $d_Z$ all range from $1$ to a shared maximum $m$. When $d_U$ follows the triangular law, the weight $(m+1-d)$ placed on $d_U=d$ exactly cancels the normaliser $1/(m-d+1)$ of the conditional uniform, so the induced marginal of each proxy dimension is the linear ramp
\begin{equation}
  \Pr[d_W = w] \;=\; \frac{w}{\sum_{j=1}^{m} j}
  \;=\; \frac{2w}{m(m+1)},
  \qquad w \in \{1, \dots, m\},
  \label{eq:proxy_marginal}
\end{equation}
and likewise for $d_Z$.

The triangular choice for $d_U$ is what buys this. Had $d_U$ been uniform, the same conditional sampling would induce a \emph{harmonic} proxy marginal, $\Pr[d_W=w] = (H_m - H_{m-w})/m$ with $H_n$ the $n$-th harmonic number, which piles proxy mass against the ceiling $m$: the value $w=m$ is reachable from every $d_U$, whereas small $w$ requires a small $d_U$, so probability accumulates at the top (for $m=10$, $\Pr[d_W=10]=0.29$). The triangular law flattens this to the linear ramp of \eqref{eq:proxy_marginal} ($\Pr[d_W=10]=0.18$), and lowers the mean confounder and proxy dimensions from $5.5$ and $7.75$ to $4.0$ and $7.0$ respectively. \Cref{fig:prior_dim_marginals} shows the delivered marginals of both schemes against these analytic laws.

\begin{figure}[h]
    \centering
    \includegraphics{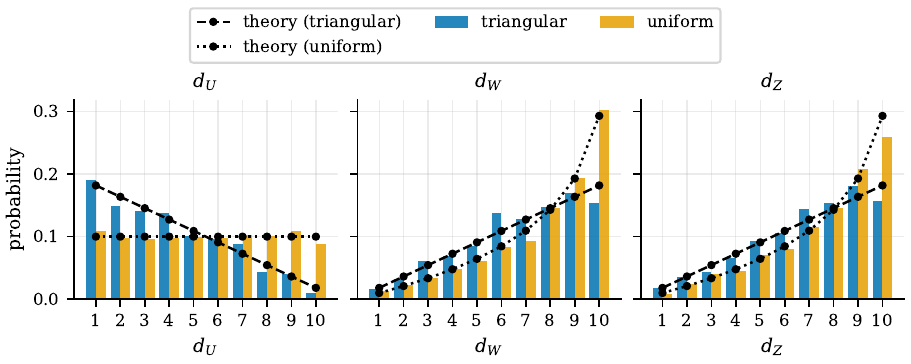}
    \caption{Delivered marginal distributions of the block dimensions $d_U, d_W, d_Z$ across the prior, for the shipped (triangular $d_U$) and the alternative (uniform $d_U$) sampling, overlaid with their analytic sampling laws (dashed: triangular; dotted: uniform). Bars are empirical frequencies over $1024$ datasets. Uniform $d_U$ pushes proxy mass against the ceiling $d^{\max}=10$; the triangular law spreads it into the linear ramp of \eqref{eq:proxy_marginal}. Delivered bars track the theoretical curves, confirming the generator samples the intended laws.}
    \label{fig:prior_dim_marginals}
\end{figure}

\FloatBarrier
\subsubsection{MLP generation, noise \& edge sampling}
\label{app:prior_mechanisms}

TabICL builds a single random MLP that turns a block of root ``causes'' into one feature table, reading its columns from the pooled activations of the MLP's intermediate layers. We reuse that construction unchanged as the mechanism for each \emph{individual} node -- the layer stack, the random activation library, the block-sparse or dropout-masked Gaussian initialisation, and the pooled-output selection, all sampled once per dataset (\Cref{tab:prior_sampled_hp}); we refer to TabICL for these details \citep{quTabICLTabularFoundation2025}. In \modelname, the innovative step is that a proximal dataset is not one block but a \emph{chain} of them, $f_X, f_U, f_W, f_Z, f_A, f_Y$ along a proximal learning DAG. Wiring one block into the next raises two questions a single-MLP generator never faces:
\begin{itemize}
    \item How should a parent block enter the child's mechanism?
    \item How should we inject fresh per-node randomness while keeping it proportional to the signal?
\end{itemize}

\paragraph{Parents and noise.} The root mechanism has no parents and maps a sampled noise input block straight to the covariates, $X=f_X(\varepsilon_X)$. Every downstream mechanism instead concatenates its standardised parent blocks with a fresh input block of its own, e.g.\ $U=f_U(\tilde X_U,\varepsilon_U)$ as in \eqref{eq:uwz}. These per-node blocks $\varepsilon_{(\cdot)}$ play the role of exogenous noise: we draw them from the \emph{same} sampler TabICL uses for its root causes -- each column independently normal, uniform, multinomial, or clamped Zipf under per-dataset mixing weights -- and size each to a per-dataset share of its mechanism's signal width (\Cref{tab:prior_sampled_hp}), so the noise-to-signal ratio stays controlled rather than shrinking as the parent blocks grow. The MLP's own within-layer Gaussian perturbation is left small ($\texttt{noise\_std}=10^{-3}$), so in practice the input blocks carry essentially all of a node's exogenous randomness. Postprocessing $\mathrm{Corrupt}(\cdot)$ (\Cref{app:prior_postprocess}) then reshapes each block before it is consumed.

\paragraph{Edge sampling.} At the MLP level, the parent blocks above are passed through dataset-specific column masks. For each mechanism that depends on $X$, we draw a keep probability from $\mathrm{Unif}[0,1]$ and retain the columns of $\tilde X$ independently according to that probability. The same procedure is applied to the $W\to Y$ and $Z\to A$ links, producing the masked proxy inputs $\tilde W_Y$ and $\tilde Z_A$. Thus edge sampling can only remove potential parents from the fixed proximal DAG: the completeness links $Z\leftarrow U\rightarrow W$ remain intact, and no new dependencies are introduced. The resulting subgraphs therefore retain the dimensionality heuristic and the structural proximal restrictions. The separate control of how strongly $U$ enters $f_A$ and $f_Y$ is described next.

\paragraph{Confounding control.} In a random MLP the confounder's influence on $A$ and $Y$ would be pinned by $U$'s share of each mechanism's input width, hence by $d_U$; to decouple confounding strength from $d_U$ we target a share $s_A,s_Y\sim\mathrm{Beta}(1.2,2.8)$ of the inputs to $f_A,f_Y$ to be $U$-derived and translate it into a column count $k$. When $k$ does not exceed $d_U$ the mechanism sees a $k$-column subsample of $U$; when it does, it sees all of $U$ plus $(k-d_U)$ extra columns drawn from a shared nonlinear re-encoding $U^{\mathrm{aug}}=f_{\mathrm{aug}}(\tilde U)$, itself a random MLP, giving the confounding-controlled inputs $U^{[A]},U^{[Y]}$. Since $f_{\mathrm{aug}}$ is a deterministic function of $U$, these extra columns carry no information beyond $U$ -- they only let the prior push confounding above what $d_U$ alone would permit, without altering the DAG or the completeness edges.

\begin{figure}[t]
    \centering
    \includegraphics{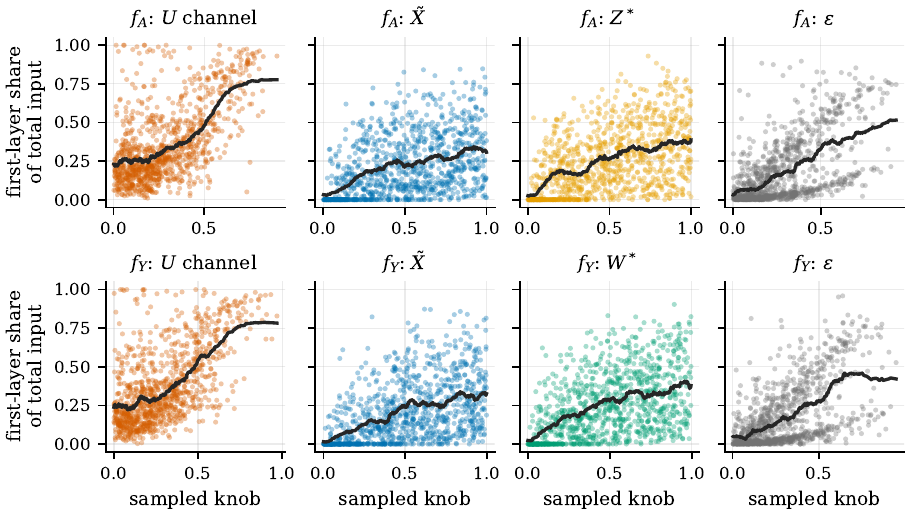}
    \caption{Every dial: each input channel's share of $f_A$'s (top) and $f_Y$'s (bottom) total input \emph{variance}, against its own sampled knob ($s$ for the $U$-channel, the covariate and proxy keep-probabilities for $\tilde X$ and $\tilde{Z}_A/\tilde{W}_Y$, the noise share for $\varepsilon$). Solid line: moving average. Each knob steers its channel's delivered variance share.}
    \label{fig:prior_mechanism_input_composition}
\end{figure}

\Cref{fig:prior_mechanism_input_composition} evaluates the effect on the mechanisms' input variance of the four signal knobs, one per channel\footnote{Confounder: \texttt{confounding\_share\_of\_U\_in\_*}. Covariate: \texttt{covariate\_share\_of\_X\_in\_*}. Proxy: \texttt{direct\_edge\_share\_of\_*\_in\_*}. Noise: \texttt{noise\_share\_of\_input\_in\_*}.} governing the $U$-, $X$-, proxy- and noise-channels respectively. We observe clearly the expected behaviour, as higher values for a knob correspond to a higher proportion of input variance carried by its own channel, on average. Then, \Cref{fig:prior_mechanism_output_dependence} considers the outputs of each mechanism, and evaluates their dependence on the input. On the delivered data we predict each mechanism's output from its parent blocks with a cross-fitted in-context regressor (TabICL, out-of-fold so every row is predicted exactly once), and read each block's \emph{unique} contribution as the leave-one-out drop $R^2(\text{all})-R^2(\text{all}\setminus b)$. For $f_A$ we predict the raw treatment logit $\mathrm{logit}(\pi)$ from $\{U,X,Z\}$; for $f_Y$ the potential outcome $Y_0$ from $\{U,X,W\}$; the noise channel, having no block to drop, is read as the unexplained share $1-R^2$. 

\Cref{fig:prior_mechanism_output_dependence} shows each channel's contribution rising with its own knob, mirroring \Cref{fig:prior_mechanism_input_composition} at the output level: the mechanisms are predicted with median $R^2\approx0.5$, and the noise knob visibly inflates the unexplained share. Two caveats with this procedure. First, the blocks are supersets of the subsampled columns each mechanism actually consumed, so a drop is a \emph{lower bound} on the true dependence. Second, because leave-one-out credits only \emph{unique} signal, the confounder channel reads low wherever $U$ is partly recoverable from its child proxy. On the $Y$ side, $Y_0$ is post heterogeneity-remix, so its $1-R^2$ also absorbs the exogenous mixing (which vanishes as $\gamma\to1$).

\begin{figure}[t]
    \centering
    \includegraphics{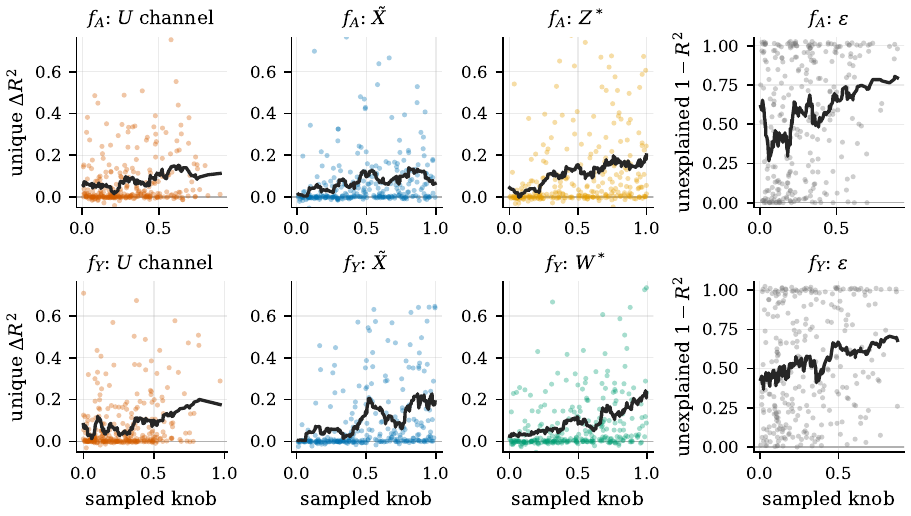}
    \caption{Output side of the composition: what each mechanism's output \emph{depends on}, measured by a cross-fitted TabICL regression on the delivered data. Columns 1--3 plot each parent channel's unique leave-one-out $\Delta R^2$ (drop in out-of-fold $R^2$ when that block is removed from the predictors) against its own knob, for $f_A$ (top, target $\mathrm{logit}(\pi)$) and $f_Y$ (bottom, target $Y_0$); column 4 plots the unexplained share $1-R^2$ against the noise knob. Solid line: moving average. Each channel's dependence rises with its knob, as at the input level. The $U$-channel reads modestly because leave-one-out credits only signal \emph{unique} to $U$, not the part shared with its child proxy.}
    \label{fig:prior_mechanism_output_dependence}
\end{figure}

\FloatBarrier
\subsubsection{Post-process: categorisation, outlier and scaling}
\label{app:prior_postprocess}

Each block passes through the $\mathrm{Corrupt}(\cdot)$ pipeline of TabICL~\citep{quTabICLTabularFoundation2025} before being consumed in \eqref{eq:uwz}: a random subset of columns is discretised into categories, outliers are clamped at $4$ standard deviations, and the result is standardised. Like TabICL, we apply this to each mechanism's output rather than \emph{within} its MLP layers, as \citet{stithCausalFoundationModels2026} instead suggest.

Categorisation interacts with the completeness assumptions of \Cref{app:proximal_prior}, which we preserve through two heuristics on the \emph{delivered} blocks. \emph{Numerically}, we keep $d_W,d_Z\ge d_U$. \emph{Categorically}, every column of $U$ must be dominated by a distinct proxy column: a numerical column of $U$ by a numerical proxy column, and a categorical column of $U$ by either a numerical proxy column or a categorical one with at least as many categories; as exemplified in \Cref{tab:prior_cat_heuristic_example}.

\begin{table}
    \centering
    \caption{Worked examples of the categorical-completeness heuristic for the fixed confounder profile $U=[\,\mathrm{num},\,5\text{-cat},\,3\text{-cat}\,]$ (one numerical column and two categorical columns, with $5$ and $3$ categories). A proxy block $P$ satisfies the heuristic when its columns can be matched one-to-one to those of $U$ so that each is at least as informative: a numerical column dominates anything, and a categorical column dominates a categorical requirement only if it has at least as many categories. A numerical proxy column may thus cover a categorical requirement, but a categorical proxy column can never cover a numerical one.}
    \label{tab:prior_cat_heuristic_example}
    \begin{tabularx}{\linewidth}{@{}l c >{\raggedright\arraybackslash}X@{}}
        \toprule
        Proxy profile $P$ & Heuristic? & Reason \\
        \midrule
        $[\,\mathrm{num},5\text{-cat},3\text{-cat}\,]$
            & \textcolor{green!50!black}{\textbf{ok}}
            & exact match, column for column \\
        $[\,\mathrm{num},8\text{-cat},4\text{-cat}\,]$
            & \textcolor{green!50!black}{\textbf{ok}}
            & categorical columns carry more categories than required \\
        $[\,\mathrm{num},\mathrm{num},3\text{-cat}\,]$
            & \textcolor{green!50!black}{\textbf{ok}}
            & a numerical column covers $U$'s $5$-cat column \\
        $[\,\mathrm{num},\mathrm{num},\mathrm{num}\,]$
            & \textcolor{green!50!black}{\textbf{ok}}
            & numerical columns dominate every requirement \\
        \midrule
        $[\,\mathrm{num},\mathrm{num}\,]$
            & \textcolor{red!70!black}{\textbf{no}}
            & too few columns ($d_P=3<d_U=4$) \\
        $[\,\mathrm{num},4\text{-cat},3\text{-cat}\,]$
            & \textcolor{red!70!black}{\textbf{no}}
            & the $4$-cat column cannot cover $U$'s $5$-cat column \\
        $[\,5\text{-cat},5\text{-cat},3\text{-cat}\,]$
            & \textcolor{red!70!black}{\textbf{no}}
            & only one numerical column for $U$'s two numerical columns \\ 
        $[\,\mathrm{num},\mathrm{num},2\text{-cat}\,]$
            & \textcolor{red!70!black}{\textbf{no}}
            & a numerical wildcard covers the $5$-cat column, leaving the
              $2$-cat column to cover $U$'s $3$-cat column, but $2<3$ \\
        \bottomrule
    \end{tabularx}
\end{table}

\paragraph{Numeric completeness.} The dimension heuristic is imposed at sampling time, but the constant-column filtering that follows (\Cref{alg:prior}, line~\ref{line:constant_column_filter}) can shrink a block afterwards, so we re-check it on the delivered dimensions. \Cref{fig:prior_dimension_joint_heatmaps} illustrates this heuristic enforcement.

\begin{figure}[t]
    \centering
    \includegraphics{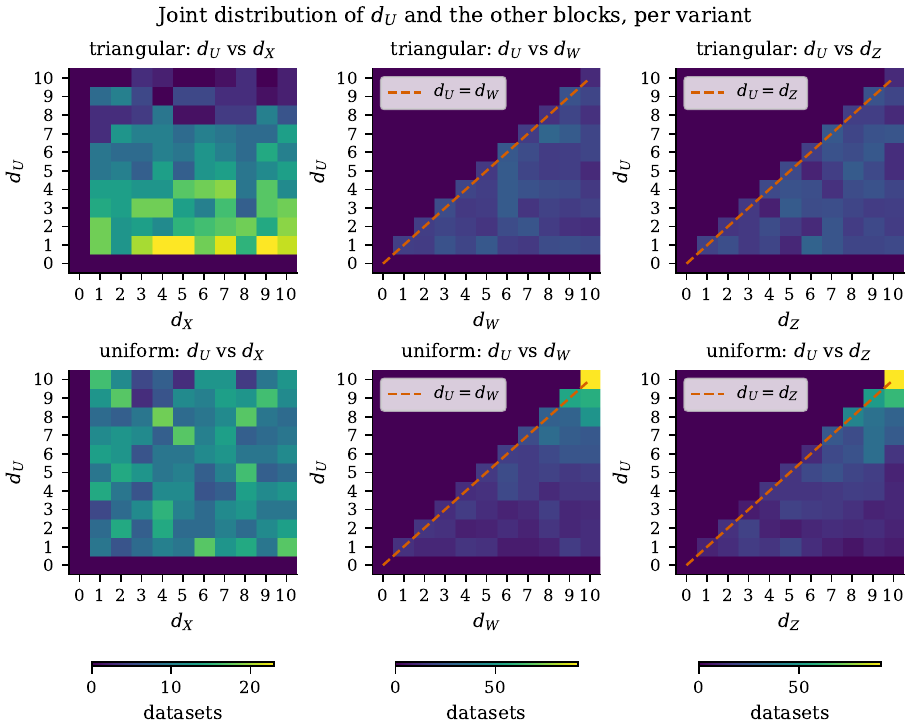}
    \caption{Joint distribution of the latent dimension $d_U$ against each feature block ($d_X$, $d_W$, $d_Z$), one row per prior variant (\emph{triangular}, \emph{uniform}). Colour encodes the number of delivered datasets falling in each $(d_U, \cdot)$ cell. The diagonal line on the $d_W$/$d_Z$ panels marks the completeness heuristic $d_W, d_Z \ge d_U$; both variants place their mass on or below it, i.e.\ constant-column filtering does not measurably break the heuristic imposed at sampling time. $d_X$ carries no such constraint and is shown for reference.}
    \label{fig:prior_dimension_joint_heatmaps}
\end{figure}

\paragraph{Categorical completeness.} We call $P$ \emph{categorically complete} w.r.t.\ $U$ when every column of $U$ matches a distinct, at-least-as-informative column of $P$ (``dominates''). Whenever $U$ is categorised, the prior builds each proxy to satisfy this by construction: it reserves one numerical column per numerical column of $U$, covers each categorical column of $U$ with either a numerical column or -- with probability \texttt{proxy\_cat\_cover\_prob} $=50\%$ -- a categorical column drawing strictly more categories, and fills any surplus proxy columns freely. \Cref{fig:prior_category_count_marginals} contrasts the block compositions this yields with and without this construction: \emph{default} (enforcement on) and \emph{no\_enforcement} ($U$ categorised, proxies left unconstrained) start from the same categorised $U$, but under \emph{no\_enforcement} the proxies end up no more numerical than $U$ itself, so nothing forces domination.

\begin{figure}[t]
    \centering
    \includegraphics{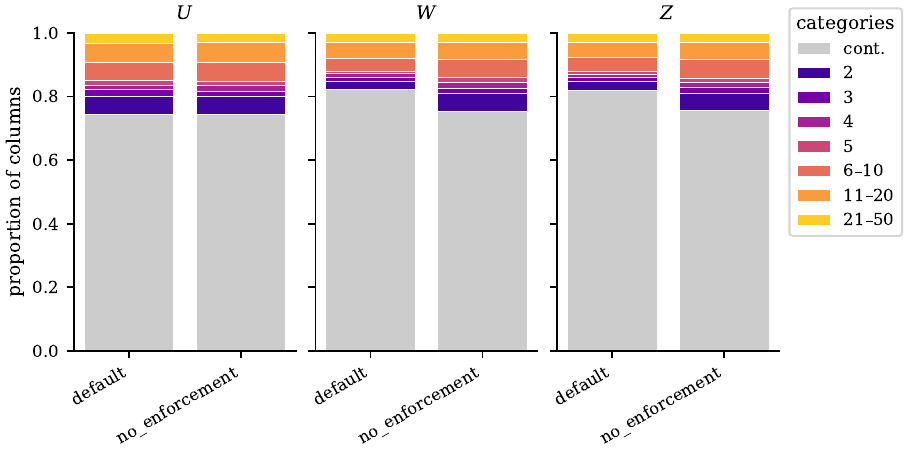}
    \caption{Column-type composition of each block ($U$, $W$, $Z$), one $100\%$-stacked bar per variant (\emph{default}: categorical completeness enforced; \emph{no\_enforcement}: proxies left unconstrained). Each bar partitions a block's columns into continuous (grey, ``cont.'') and categorical, the latter binned by category count ($2$--$5$ singletons, then coarser $6$--$10$, $11$--$20$, $21$--$d^{\max}_{\mathrm{cat}}$ tail bins). Continuous columns dominate every block; contrasting $W$/$Z$ across the two variants shows how much enforcement pushes proxy columns categorical relative to the unconstrained \emph{no\_enforcement} arm.}
    \label{fig:prior_category_count_marginals}
\end{figure}

\paragraph{Enforcement scales with $U$'s demand.} \Cref{fig:prior_proxy_enforcement_vs_u_demand} reads the covering rule off the delivered data from two angles. Conditioning on a single $U$ column's cardinality (panels a--b), the proxy block's numerical fraction is roughly flat in that cardinality and sits \emph{above} the unconstrained arm ($\approx0.70$ vs.\ $\approx0.60$), while the mean category count of the proxy's categorical columns grows with the $U$ column's cardinality under \emph{default} ($\approx12$ at $2$ categories to $\approx19$ in the $30{+}$ bin), exactly as the ``strictly more categories'' covering rule dictates, and stays flat at $\approx10$ under \emph{no\_enforcement}. Enforcement therefore acts on the \emph{depth} of the proxy's categorical columns, not on how many of them there are.

Panel (c) makes the same point negatively, at the dataset level. The categorical share of the proxies rises with $U$'s categorical demand $|\mathrm{cats}(U)|$ in \emph{both} arms, because the per-dataset categorisation propensity is shared across blocks: a dataset that discretises many columns of $U$ also discretises many columns of its proxies, enforcement or not. Enforcement in fact leaves the proxies \emph{less} categorical than the free \emph{no\_enforcement} arm at every demand level ($0.08\to0.43$ vs.\ $0.14\to0.58$), consistent with the reserved numerical columns of panel (a). So the categorical share is not the signature of enforcement; the per-column depth of panel (b), and the survival rate below, are.

\begin{figure}[t]
    \centering
    \includegraphics{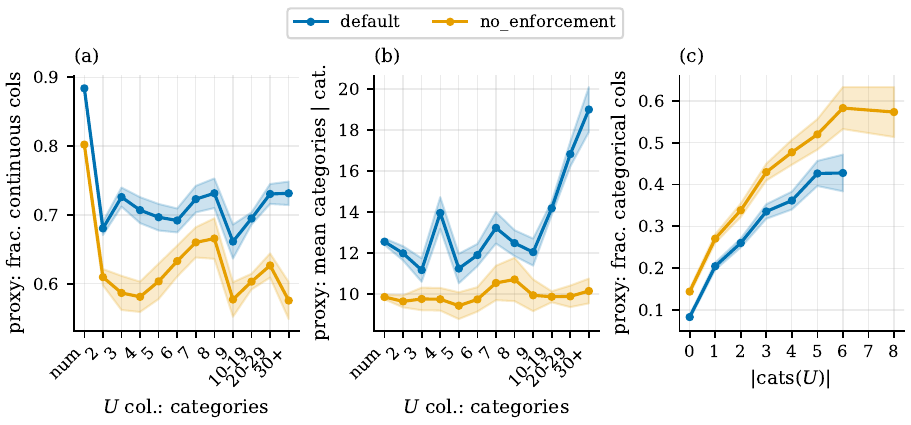}
    \caption{How enforcement makes proxy structure track $U$'s categorical demand, pooling $W$ and $Z$ across the \emph{default} and \emph{no\_enforcement} variants; shaded bands are $\pm$ one SEM. \textbf{(a, b)} condition on a single $U$ column's own cardinality: the $x$-axis bins a $U$ column by its category count (continuous ``\texttt{num}'', then $2,\dots,9$, and coarser tail bins up to ``$30{+}$''), and each point averages a statistic of the co-occurring proxy block over every dataset containing a $U$ column of that kind -- (a) the fraction of proxy columns that are continuous, essentially flat in $U$'s cardinality but higher under \emph{default}, and (b) the mean category count of the proxy's categorical columns, which rises with $U$'s cardinality under \emph{default} but stays low under \emph{no\_enforcement}. \textbf{(c)} moves to the dataset level: the mean fraction of categorical proxy columns against $U$'s total categorical demand $|\mathrm{cats}(U)|$, with bins of fewer than $5$ datasets dropped. Both arms rise (the per-dataset categorisation propensity is shared across blocks) and \emph{default} sits below \emph{no\_enforcement}: enforcement deepens the proxy's categorical columns rather than multiplying them. Because proxy statistics in (a, b) are block-level rather than tied to a matched column, those trends reflect association with ``the dataset contains a $U$ column of this kind,'' not a one-to-one mapping.}
    \label{fig:prior_proxy_enforcement_vs_u_demand}
\end{figure}

\FloatBarrier
\subsubsection{Treatment assignment}
\label{app:prior_treatment}

The treatment mechanism $f_A$ maps the confounding-controlled confounder input $U^{[A]}$, the covariate subset $\tilde X_A$, the direct proxy subset $\tilde{Z}_A$ and fresh noise $\varepsilon_A$ to a single scalar logit per row,
\begin{equation}
  \ell \;=\; f_A\big(U^{[A]},\,\tilde X_A,\,\tilde{Z}_A,\,\varepsilon_A\big)\;\in\;\R^{TK},
  \label{eq:A_logits}
\end{equation}
with $W$ excluded from the inputs so that the proxy exclusion restriction of \Cref{as:proxy_exclusion} holds by construction (Equation~\ref{eq:uwz}). The treatment $A$ is binary; the steps below turn $\ell$ into a propensity score $\pi=\Pr[A=1\mid\cdot]$ and draw $A$.

\paragraph{Logit normalisation.} A random MLP fixes neither the location nor the spread of its output, both of which drift with depth, width and input scale; left raw, the induced treatment prevalence and the amount of overlap would be an uncontrolled by-product of the sampled architecture. We therefore standardise the logits across the $T\cdot K$ rows of the dataset to zero mean and unit variance,
\begin{equation}
  \tilde\ell \;=\; \frac{\ell - \bar\ell}{\operatorname{sd}(\ell) + 10^{-8}},
  \label{eq:A_normalise}
\end{equation}
so that the assignment distribution is decoupled from the accidental scale of $f_A$ and every dataset starts from a common footing before the overlap knobs act.

\paragraph{Shift and scale.} Two per-dataset propensity knobs then set the assignment regime through an affine map of the standardised logits,
\begin{equation}
  \ell' \;=\; s\,\tilde\ell + b,\qquad
  b \sim \mathrm{Unif}(-2.5,\,2.5),\quad
  s \sim \text{log-scaled truncated normal},
  \label{eq:A_affine}
\end{equation}
both drawn once per dataset (see \Cref{tab:prior_sampled_hp}). The shift $b$ moves the marginal treatment prevalence away from the balanced split -- negative $b$ makes treatment rare, positive $b$ makes it common -- while the scale $s$ controls how sharply assignment depends on the covariates and confounder: small $s$ flattens the logits toward a coin flip (strong overlap, weak confounding channelled through $A$), whereas large $s$ sharpens them toward near-deterministic assignment (weaker overlap).

\paragraph{Propensity and sampling.} A sigmoid turns the shifted logits into a propensity score, from which the treatment is drawn Bernoulli,
\begin{equation}
  \pi \;=\; \sigma(\ell'),\qquad A \sim \mathrm{Bernoulli}(\pi).
  \label{eq:A_sigmoid}
\end{equation}

\paragraph{Overlap clamp.} Because a large $s$ can push $\pi$ arbitrarily close to $0$ or $1$, we clamp the propensity into $[\epsilon,1-\epsilon]$ with $\epsilon=10^{-3}$ before sampling,
\begin{equation}
  \pi \;\leftarrow\; \operatorname{clip}(\pi,\,\epsilon,\,1-\epsilon),
  \label{eq:A_clamp}
\end{equation}
guaranteeing a floor on the probability of each treatment arm and hence enforcing positivity/overlap ($0 < \Pr[A=1\mid\cdot] < 1$) for every row by construction (\Cref{as:positivity}). We audit the delivered overlap across the prior in \Cref{app:prior_overlap}.

\Cref{fig:prior_treatment_walkthrough} shows this five-step pipeline on a single representative dataset, following one unit (in red) from its raw logit through to its treatment draw.

\begin{figure}[t]
    \centering
    \includegraphics{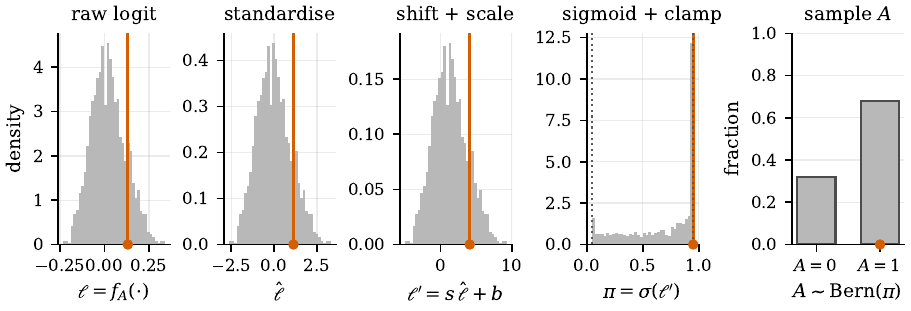}
    \caption{The treatment-assignment pipeline of \eqref{eq:A_logits}--\eqref{eq:A_clamp} on one representative dataset, tracking a single unit (red line) across all five stages. (a) The raw logit $\ell=f_A(\cdot)$, whose scale and offset are an uncontrolled by-product of the sampled architecture. (b) The standardised $\tilde\ell$ of \eqref{eq:A_normalise}, pinned to zero mean and unit variance. (c) The affine $\ell'=s\,\tilde\ell+b$ of \eqref{eq:A_affine}, shifted and sharpened by the two sampled overlap knobs. (d) The propensity $\pi=\sigma(\ell')$ of \eqref{eq:A_sigmoid}, confined to $[\epsilon,1-\epsilon]$ by the clamp of \eqref{eq:A_clamp} (dotted bounds, drawn at an enlarged $\epsilon$ for visibility; production $\epsilon=10^{-3}$). (e) The realised treatment $A\sim\mathrm{Bernoulli}(\pi)$ across the batch (fraction in each arm), with the tracked unit's realised arm marked in red.}
    \label{fig:prior_treatment_walkthrough}
\end{figure}

\FloatBarrier
\subsubsection{Potential outcomes \& heterogeneity remix}
\label{app:prior_outcome}

The outcome mechanism $f_Y$ emits both potential outcomes from a single forward pass, returning the pair $\big(Y_0^{\mathrm{raw}},Y_1^{\mathrm{raw}}\big)$ from a common input and exogenous draw $\varepsilon_Y$; the two arms differ only in their read-out, so the raw effect $\tau^{\mathrm{raw}}=Y_1^{\mathrm{raw}}-Y_0^{\mathrm{raw}}$ stays coupled through the shared latent state.

Left as is, a dataset's effect heterogeneity would be an uncontrolled by-product of the sampled MLP. To let the prior span the full range from strongly modulated effects to a single constant effect, we remix the raw outcomes before forming the targets. With a per-dataset dial $\gamma\sim\mathrm{Unif}(0,1)$, a per-row weight $\alpha\sim\mathrm{Unif}(0,1)$, and the raw average effect $\bar\tau=\tfrac{1}{T\cdot K}\sum(Y_1^{\mathrm{raw}}-Y_0^{\mathrm{raw}})$,
\begin{equation}
\begin{aligned}
  Y_1 &= \big(\alpha+(1-\alpha)\gamma\big)\,Y_1^{\mathrm{raw}}
        + (1-\gamma)(1-\alpha)\big(Y_0^{\mathrm{raw}}+\bar\tau\big),\\
  Y_0 &= \big((1-\alpha)+\alpha\gamma\big)\,Y_0^{\mathrm{raw}}
        + (1-\gamma)\alpha\big(Y_1^{\mathrm{raw}}-\bar\tau\big).
\end{aligned}
\label{eq:remix}
\end{equation}
This construction follows the heterogeneity knob of RealCause \citep{nealRealCauseRealisticCausal2021}, later re-used in the causal-foundation-model prior of CausalPFN \citep{balazadehCausalPFNAmortizedCausal2025}.

Substituting \eqref{eq:remix} collapses the per-row weights to $Y_1-Y_0=\gamma\,\tau^{\mathrm{raw}}+(1-\gamma)\bar\tau$, so $\gamma$ linearly interpolates each unit's effect between its raw value ($\gamma=1$, fully heterogeneous) and the dataset mean $\bar\tau$ ($\gamma=0$, fully homogeneous). Anchoring at $\bar\tau$ rather than zero preserves the ATE at every $\gamma$; the per-row $\alpha$ simply shares the required shift between $Y_0$ and $Y_1$ for each unit; averaged across rows, this allocation and the compensating $\pm\bar\tau$ terms leave the marginal means of $Y_0,Y_1$ (hence the outcome levels and the ATE) unchanged in expectation, so the dial reshapes only the \emph{spread} of the effect.

Because the collapse also gives $\mathrm{sd}(Y_1-Y_0)=\gamma\,\mathrm{sd}(\tau^{\mathrm{raw}})$, the dial should act on the delivered data as a pure, proportional spread control. \Cref{fig:prior_heterogeneity_spread} checks this across the prior: $\gamma$ is drawn independently per dataset and recorded, so the shipped prior is itself a $\gamma$ sweep, and the (outcome-scaled) effect dispersion tracks it closely, while the outcome coupling $\mathrm{corr}(Y_0,Y_1)$ approaches its exact $\gamma\to0$ value of $1$.

\begin{figure}
    \centering
    \includegraphics{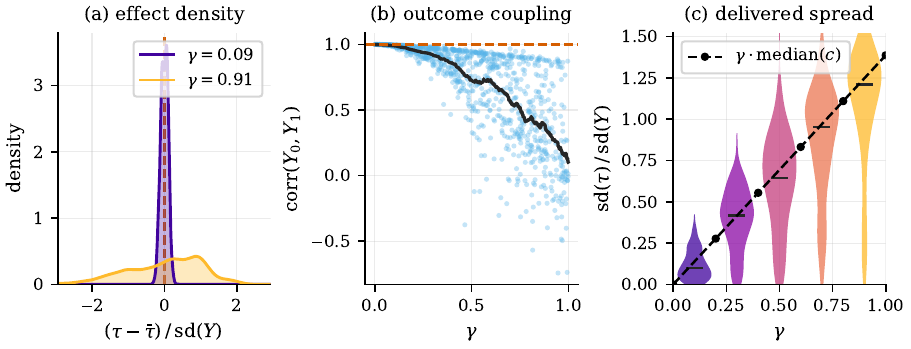}
    \caption{The heterogeneity dial on the delivered prior, over $1024$ datasets. All statistics are ratios, since each SCM's outcome carries its own arbitrary scale. (a) Two datasets' effect densities on a shared outcome-scaled axis, each centred on its own $\bar\tau$ (dashed line), one gathered ($\gamma=0.09$) and one spread ($\gamma=0.91$); the pair is chosen with \emph{matched} raw dispersion $c=\mathrm{sd}(\tau^{\mathrm{raw}})/\mathrm{sd}(Y)$ ($1.127$ vs.\ $1.125$), so the contrast is attributable to $\gamma$ rather than to SCM variation. (b) The outcome coupling $\mathrm{corr}(Y_0,Y_1)$ against $\gamma$ (solid line: moving average). At $\gamma=0$ the effect is the constant $\bar\tau$, so $Y_1=Y_0+\bar\tau$ and the coupling is \emph{exactly} $1$ (dashed reference). (c) Distribution of the delivered dispersion $\mathrm{sd}(\tau)/\mathrm{sd}(Y)$ per $\gamma$ bin (Spearman $0.78$ overall), against the proportionality $\gamma\cdot\mathrm{median}(c)$ implied by the collapse; within-bin spread is dominated by variation in $c$. Because the reference line is anchored at zero and uses the overall median $c$ for its slope, the panel verifies strict proportionality to $\gamma$—linear scaling starting from the origin—rather than predicting the absolute magnitude of dispersion.}
    \label{fig:prior_heterogeneity_spread}
\end{figure}

\FloatBarrier
\subsubsection{Monte Carlo CATE supervision}
\label{app:prior_mc_cate}
For a query covariate row \(x_i\), \texttt{mc\_cate} specifies the number
\(K\) of independent downstream SCM draws used to approximate its CATE.
The covariates $X$ remain fixed, while the latent confounder $U$, proxies, and
exogenous outcome noise are redrawn. Within draw \(k\), both potential
outcomes are evaluated on the same latent state, yielding the paired effect
\(\tau_{i,k}=Y_{i,k}(1)-Y_{i,k}(0)\). We supervise on
\(\widehat Q_K(x_i)=K^{-1}\sum_k\tau_{i,k}\), an unbiased estimate of
\(Q(x_i)\), with estimated sampling variance
\(S_i^2=\widehat{\mathrm{Var}}_k(\tau_{i,k})/K\). Thus \(K\) controls
target precision, not the SCM or estimand: its standard error decreases as
\(K^{-1/2}\), at generation cost proportional to \(K\).

\FloatBarrier
\subsubsection{Visualize the prior}
\label{app:prior_visualize}

We close \Cref{app:prior_model} with three galleries of the delivered data, each a fixed-seed cross-section of the prior's diversity: \Cref{fig:prior_feature_distributions} shows individual \emph{feature marginals}, one column per block and one sampled dataset per row (the propensity $\pi$ for $A$, both potential outcomes $Y_0,Y_1$ for $Y$); \Cref{fig:prior_feature_pairs} shows \emph{within-block} dependence, scattering two features of the same block ($\pi$ against the draw $A$, and $Y_0$ against $Y_1$ for the scalar blocks); and \Cref{fig:prior_feature_pairs_cross} shows \emph{cross-block} dependence for a chosen few block pairs -- the parent$\to$child edges $U$--$W$, $U$--$Z$, the proxy pair $W$--$Z$ (associated only through the shared latent $U$), and the confounder against the propensity $U$--$\pi$.

\begin{figure}[h]
    \centering
    \includegraphics{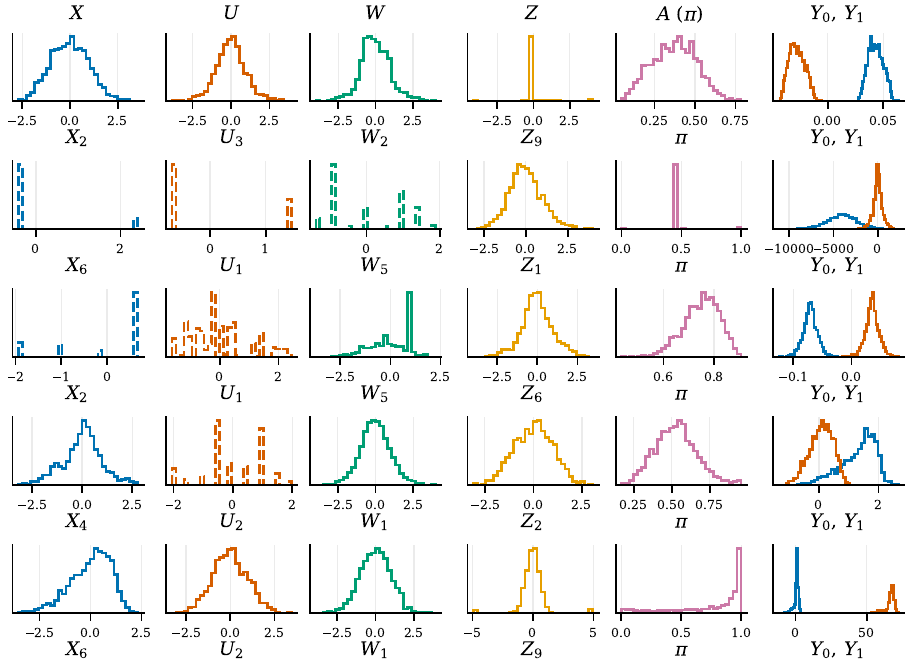}
    \caption{Gallery of feature marginals: one column per block ($X$, $U$, $W$, $Z$, the propensity $\pi=P(A{=}1\mid\cdot)$ for $A$, both potential outcomes for $Y$), one row per randomly sampled dataset. Each panel is one feature column's histogram, solid for continuous columns and dashed for categorical ones; the $Y$ column overlays $Y_0$ (blue) and $Y_1$ (red). Feature blocks share the standardised axis $\mathrm{Corrupt}$ leaves; $\pi\in[0,1]$ and $Y$
    keep their natural scale.}
    \label{fig:prior_feature_distributions}
\end{figure}

\begin{figure}
    \centering
    \includegraphics{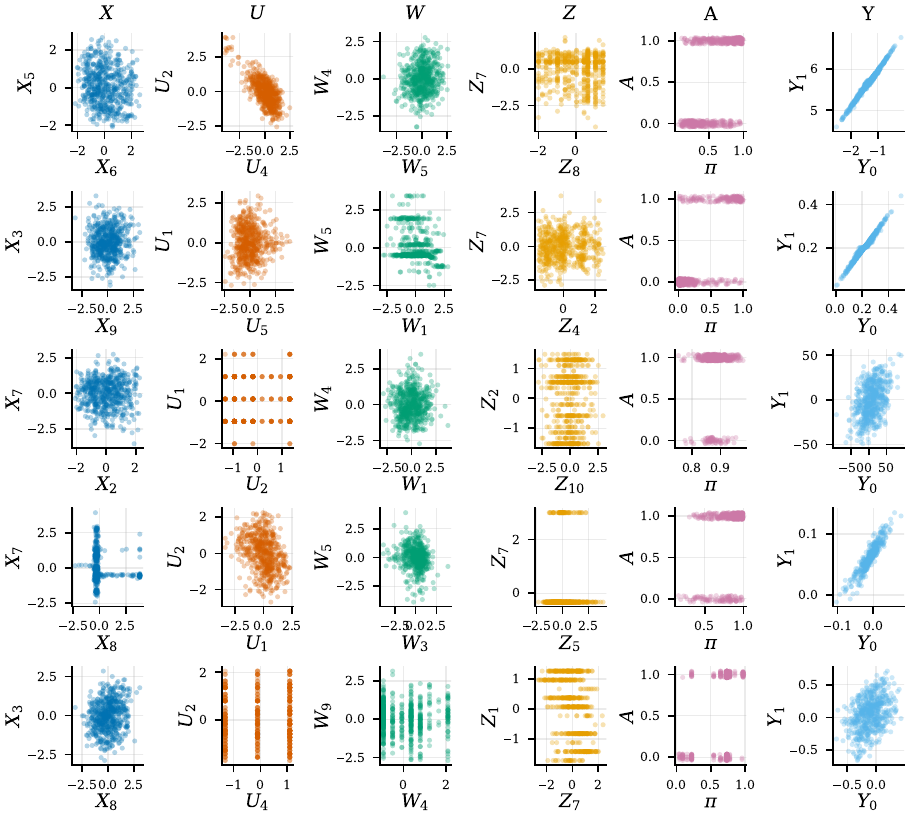}
    \caption{Within-block dependence: one column per block, one randomly sampled dataset per row, each panel a scatter of two features from that block. For $X$, $U$, $W$, $Z$ the two features are random columns of the block (sampled from datasets of width $\ge 2$); for $A$ the panel plots the propensity $\pi$ against the treatment draw $A$; for $Y$ the two potential outcomes $Y_0$ vs $Y_1$. Rows are subsampled for legibility; feature axes inherit the standardised scale.}
    \label{fig:prior_feature_pairs}
\end{figure}

\begin{figure}
    \centering
    \includegraphics{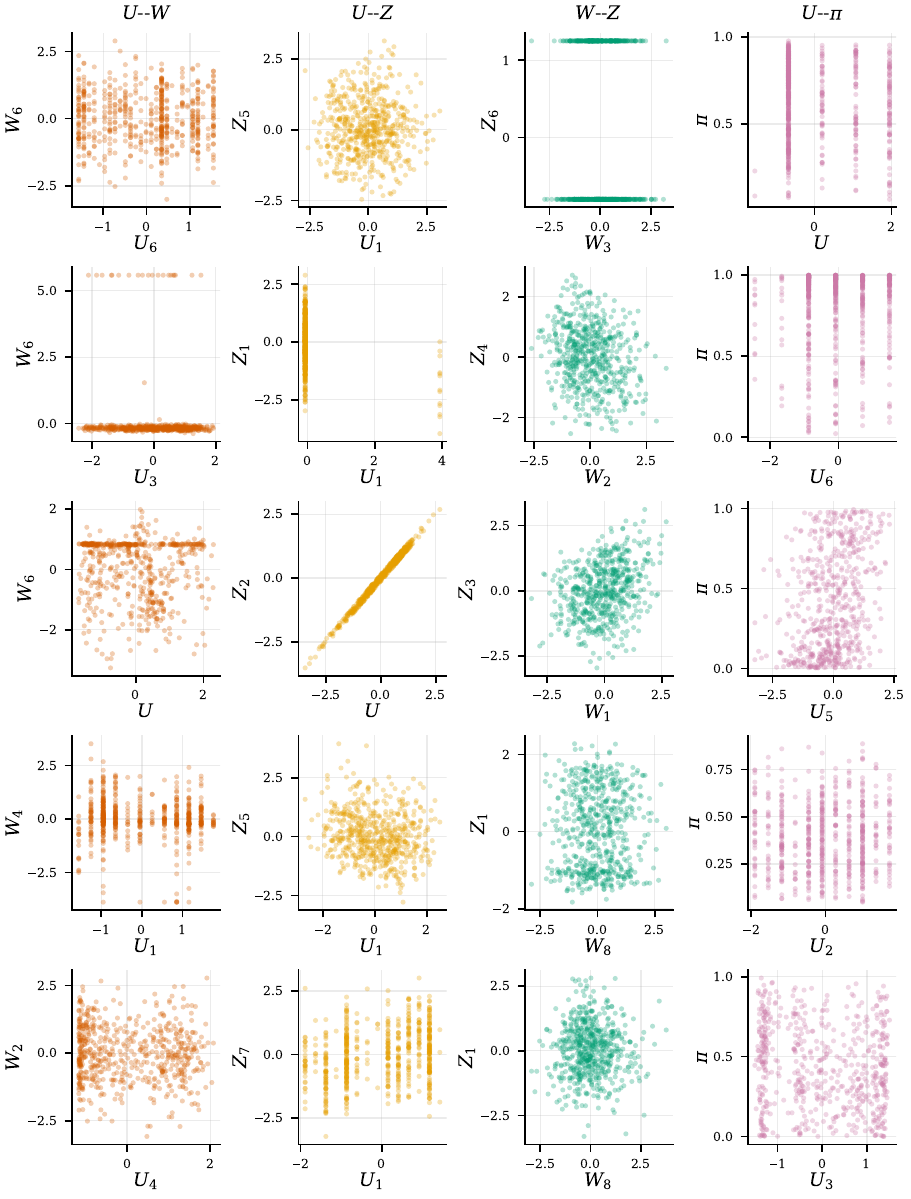}
    \caption{Cross-block dependence: one column per chosen block pair, one
    sampled dataset per row, each panel a scatter of one feature from each
    block. $U$--$W$ and $U$--$Z$ are the parent$\to$child edges; $W$--$Z$ is the
    proxy pair, associated only through the shared latent $U$ (no direct edge);
    $U$--$\pi$ is the confounder against the propensity $\pi=P(A{=}1\mid\cdot)$.
    Rows are subsampled for legibility; feature axes inherit the standardised
    scale.}
    \label{fig:prior_feature_pairs_cross}
\end{figure}

\clearpage
\subsection{Proximal assumption audit}
\label{app:prior_audit}

\FloatBarrier
\subsubsection{Overlap}
\label{app:prior_overlap}

Positivity (\Cref{as:positivity}) cannot fail in this prior: the clamp of \eqref{eq:A_clamp} confines $\pi$ to $[\varepsilon,1-\varepsilon]$ with $\varepsilon=10^{-3}$ chosen deliberately small (see \Cref{tab:prior_fixed_hp}), so $\pi$ is never exactly $0$ or $1$ and every row carries nonzero mass in both treatment arms by construction. Strictly, the stored $\pi$ is realised conditional on the unit's full set of mechanism inputs $(X,U,Z,\epsilon_A)$ rather than the marginal $\Pr[A=1\mid X,U]$ that \Cref{as:positivity} is stated over, since $A$'s logit also depends on the proxy $Z$ and the mechanism's own exogenous noise; but the marginal is the expectation of this finer-grained $\pi$ over $Z,\epsilon_A$, so keeping every realised $\pi$ away from $0,1$ is a sufficient, if anything conservative, condition for the marginal to be as well, and auditing the stored values directly is valid. The question worth auditing is therefore not whether positivity holds but how close the sampled overlap knobs push $\pi$ toward that clamp in practice, since a propensity pinned near $0$ or $1$ leaves little residual randomness in treatment assignment for any estimator to exploit even though positivity technically holds. The audit needs no additional estimation: \Cref{fig:prior_overlap} is a direct summary of these stored propensities.

Because $\varepsilon$ is small, the clamp only bites at the very extremes and the vast majority of rows sit well within it: across the prior, $84.0\%$ of all rows fall inside the conventional band $[0.05,0.95]$, and only $42.6\%$ of datasets contain even a single row on the clamp, with the median dataset containing none. Overlap is nonetheless deliberately stressed away from that safe interior: the median dataset still places $13.6\%$ of its rows outside $[0.05,0.95]$, since the sharpness $s$ of \eqref{eq:A_affine} is sampled precisely to span easy and hard assignment regimes.

\begin{figure}[h]
    \centering
    \includegraphics{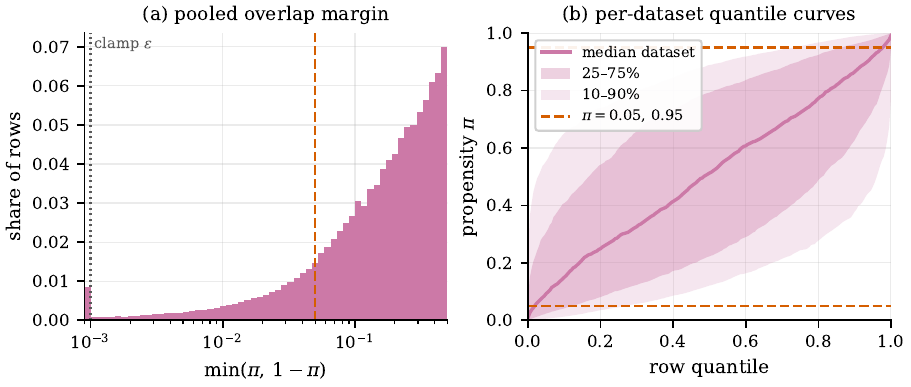}
    \caption{Delivered overlap across the prior, from the stored true propensity $\pi=\Pr[A=1\mid X,U]$ (no estimation involved). (a) Pooled distribution of the overlap margin $\min(\pi,1-\pi)$ over all rows of $1024$ datasets, on a logarithmic axis: the margin measures each unit's distance from a degenerate assignment, folding both tails into one axis, and the log scale is what makes the mass piled on the clamp $\varepsilon=10^{-3}$ (dotted) visible at all. Bars are per-bin shares of all rows on a logarithmic grid, so bin width grows to the right. (b) Per-dataset quantile function of $\pi$, summarised across datasets by the median curve with the $25$--$75\%$ and $10$--$90\%$ inter-dataset bands. Dashed lines mark the practical-overlap band $[0.05,0.95]$ in both panels.}
    \label{fig:prior_overlap}
\end{figure}

\FloatBarrier
\subsubsection{Exclusion restrictions}
\label{app:prior_exclusion}

The proxy exclusion restriction of \Cref{as:proxy_exclusion}, $Y\indep Z\mid(A,X,U)$ and $W\indep(A,Z)\mid(X,U)$, holds \emph{by construction}: \eqref{eq:uwz} never wires $Z$ into $f_W$ or $f_Y$, nor $W$ into $f_Z$ or $f_A$. We audit this with a cross-fitted regression: on the $Y$ side, the incremental $\Delta R^2_Z=R^2(Y_{\mathrm{obs}}\mid A,X,U,Z)-R^2(Y_{\mathrm{obs}}\mid A,X,U)$ of adding $Z$ to a conditioning set that already contains everything the clause names; on the $A$ side, the analogous $\Delta R^2_W$ for $W$ against the raw treatment logit $\mathrm{logit}(\pi)$ given $(X,U,W)$. Both are expected to be $\approx0$.

A near-zero $\Delta R^2$ is only informative if the estimator could have detected a leak had one been present. We therefore pair each check with a positive control that repeats the same regression with $U$ dropped from the conditioning set (Y-side: $\{A,X,Z\}$; A-side: $\{X,W\}$): since $U$ is the only path linking $Z$ to $Y$ and $W$ to $A$, removing it makes the excluded block genuinely predictive, and the control must light up clearly nonzero. Because the control's power scales with how much confounding a given draw actually carries, the figure also colours and ranks each point by \emph{u\_dependence} $=R^2(\mathrm{with}\ U)-R^2(\mathrm{without}\ U)$, the unique contribution of $U$ to the target: the control should fan upward with warmer colour, while the check stays flat and near zero regardless of colour.

\Cref{fig:prior_exclusion} confirms this reading. Both checks sit on zero (median $\Delta R^2_Z=-0.0019$, $\Delta R^2_W=-0.0032$), while both controls are slightly positive and rise further on the third of datasets with the largest $U$-dependence, with a positive Spearman correlation between $u\_dependence$ and the control ($0.375$ Y-side, $0.415$ A-side) visible as the colour gradient climbing alongside the control column. The zero-leak reading is therefore not an artefact of a blind test: the same regression detects a real, confounder-mediated dependence whenever $U$ is removed, and detects none when the exclusion restriction is left intact.

\begin{figure}
    \centering
    \includegraphics{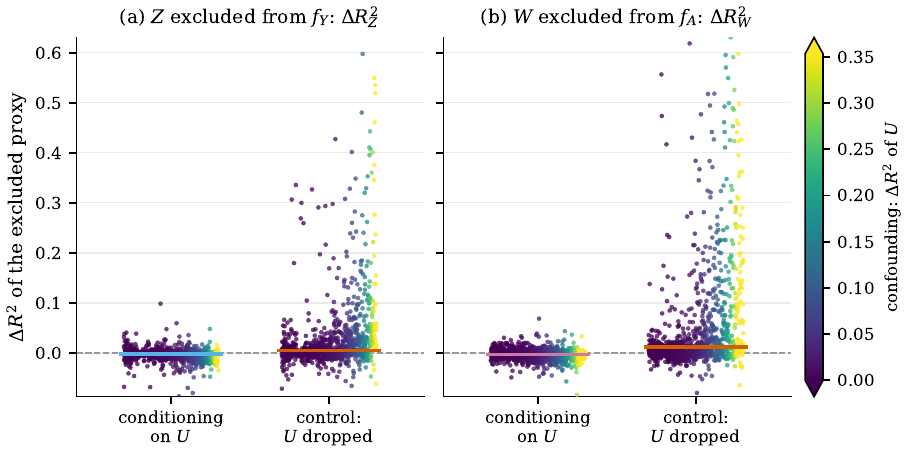}
    \caption{Exclusion-restriction audit (\Cref{as:proxy_exclusion}) over $1024$ datasets. Each panel is one clause -- (a) $Z$ excluded from $f_Y$, (b) $W$ excluded from $f_A$ -- and each column a cross-fitted incremental $\Delta R^2$ of the excluded block: \emph{conditioning on $U$} is the check (expected $\approx0$), \emph{control: $U$ dropped} the positive control (expected clearly positive). Colour and horizontal position both encode $u\_dependence=R^2(\mathrm{with}\ U)-R^2(\mathrm{without}\ U)$, i.e.\ how much confounding that draw carries (linear ramp, floored at $0$, capped at the $95$th percentile). Both checks stay on zero regardless of colour; both controls fan upward with warmer colour, most visibly in the top tercile of $u\_dependence$ (\Cref{tab:prior_exclusion}).}
    \label{fig:prior_exclusion}
\end{figure}

\begin{table}
\centering
\caption{Exclusion-restriction audit (\Cref{as:proxy_exclusion}): cross-fitted incremental $\Delta R^2$ of the excluded block over the stated conditioning set, for the check (conditions on $U$) and the positive control ($U$ dropped). \emph{Top confounded} is the third of datasets with the largest $u\_dependence$ ($R^2$ with $U$ minus $R^2$ without it).}
\label{tab:prior_exclusion}
\begin{tabular}{lllrrr}
\toprule
conditioning set & role & datasets & median $\Delta R^2$ & 90th pct & share $> 0.02$ \\
\midrule
\multicolumn{6}{l}{\textbf{$Z$ excluded from $f_Y$}} \\
$\{A, X, U, Z\}$ & conditions on $U$ & all & -0.002 & 0.007 & 0.021 \\
$\{A, X, U, Z\}$ & conditions on $U$ & top confounded & -0.005 & 0.004 & 0.012 \\
$\{A, X, Z\}$ & positive control & all & 0.004 & 0.099 & 0.313 \\
$\{A, X, Z\}$ & positive control & top confounded & 0.029 & 0.222 & 0.598 \\
\midrule
\multicolumn{6}{l}{\textbf{$W$ excluded from $f_A$}} \\
$\{X, U, W\}$ & conditions on $U$ & all & -0.003 & 0.009 & 0.020 \\
$\{X, U, W\}$ & conditions on $U$ & top confounded & -0.004 & 0.004 & 0.021 \\
$\{X, W\}$ & positive control & all & 0.012 & 0.218 & 0.415 \\
$\{X, W\}$ & positive control & top confounded & 0.060 & 0.395 & 0.677 \\
\bottomrule
\end{tabular}
\end{table}

\FloatBarrier
\subsubsection{Latent exchangeability and confounding strength}
\label{app:prior_confounding}

Latent exchangeability (\Cref{as:latent_ignorability}) holds by construction. Treatment is drawn from $f_A(X,U,Z,\varepsilon_A)$, while both potential outcomes are generated without treatment as an input and depend on $X,U,W$ and independent outcome noise (\Cref{eq:uwz}). Conditional on $(X,U)$, the exogenous draws generating $Z$ and treatment are independent of the ones that generate $W$ and the potential outcomes. So, $Y(a)\indep A\mid(X,U)$ for each $a$.This establishes the assumption; we investigate the strength of the confounding of $U$ below, but this does not test the conditional independence or prove finite-sample unbiasedness.

\Cref{fig:prior_confounding} compares five estimators (unadjusted; T-learners adjusting for $X$ or $(X,W,Z)$; and oracle T-learners adjusting for $(X,U)$ or $(X,U,W)$) against the generated oracle ATE $\mathrm{mean}(Y_1-Y_0)$, with absolute errors divided by $\mathrm{sd}(Y_{\mathrm{obs}})$ (\Cref{tab:prior_confounding}). Median errors are close (unadjusted $0.041$; oracle-$(X,U,W)$ $0.025$). Errors above $0.1\,\mathrm{sd}(Y_{\mathrm{obs}})$ occur in $24.1\%$ of datasets without adjustment but only $10.7$--$10.8\%$ with either oracle. This reduction is consistent with residual confounding in some draws, although finite-sample estimation error also contributes to each rate. \Cref{fig:prior_confounding_share} groups datasets by the delivered $U$ share of $f_A$ and $f_Y$: the backdoor-$X$ error is higher mainly when the treatment mechanism receives a larger share of $U$-derived inputs. The pattern is modest at this sample size; adjusting for $U$ removes this confounding pathway, explaining why the oracle error does not rise in the same way.

\begin{figure}
    \centering
    \includegraphics{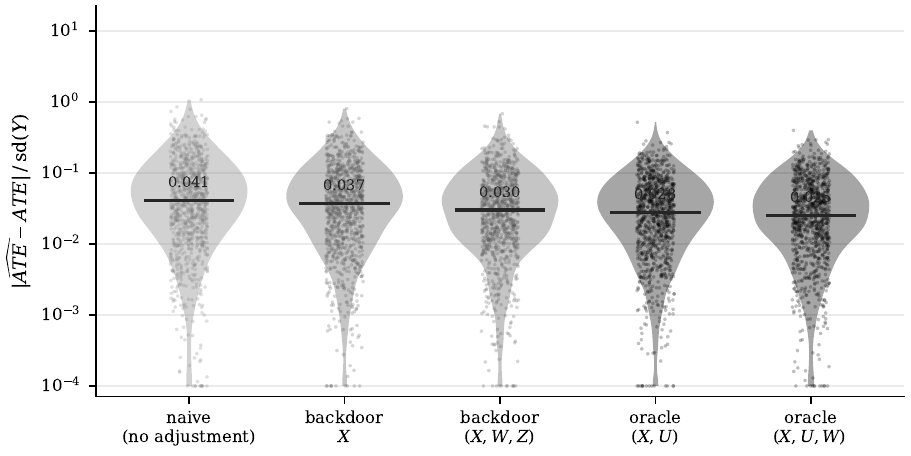}
    \caption{Confounding-strength comparison (\Cref{tab:prior_confounding}) over $1024$ datasets. Each violin shows normalised absolute ATE error $|\widehat{\mathrm{ATE}}-\mathrm{ATE}|/\mathrm{sd}(Y_{\mathrm{obs}})$ on a log axis, scored against $\mathrm{mean}(Y_1-Y_0)$: \emph{naive} (no adjustment), \emph{backdoor} $X$ and $(X,W,Z)$ (T-learners on measured covariates), and \emph{oracle} $(X,U)$ and $(X,U,W)$ (T-learners also given the latent confounder).}
    \label{fig:prior_confounding}
\end{figure}

\begin{figure}
    \centering
    \includegraphics{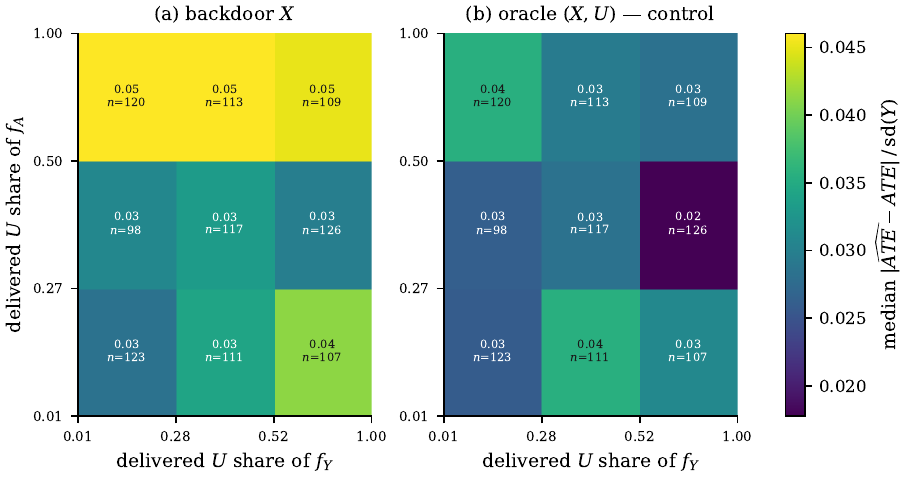}
    \caption{Backdoor-$X$ and oracle-$(X,U)$ error over the delivered confounding-share grid: $1024$ datasets binned $3\times3$ on the quantiles of the delivered $U$ share of $f_A$ (rows) and $f_Y$ (columns) (\Cref{fig:prior_mechanism_input_composition}). Cells report the median normalised error on a shared colour scale. The backdoor panel's top row (delivered $U$ share of $f_A$ above $0.495$) is elevated at $0.045$--$0.046$ against $0.028$--$0.041$ in the two rows below, and is nearly flat across $f_Y$'s share. The oracle panel shows no comparable pattern.}
    \label{fig:prior_confounding_share}
\end{figure}

\begin{table}
\centering
\caption{ATE error across adjustment sets. Errors are absolute deviations from the in-sample $\mathrm{mean}(Y_1-Y_0)$, divided by $\mathrm{sd}(Y_{\mathrm{obs}})$. All adjusted estimates use cross-fitted T-learners; the oracle rows include the latent $U$.}
\label{tab:prior_confounding}
\small
\begin{tabularx}{\textwidth}{ll>{\raggedright\arraybackslash}Xlll}
\toprule
rung & adjustment set & role & median error & 90th pct & share $> 0.1$ \\
\midrule
naive & none & unadjusted & 0.041 & 0.195 & 0.241 \\
backdoor X & $X$ & correct if no latent confounder & 0.037 & 0.163 & 0.211 \\
backdoor XWZ & $(X, W, Z)$ & DAG-blind: everything measured & 0.030 & 0.121 & 0.148 \\
oracle XU & $(X, U)$ & identification: minimal sufficient set & 0.028 & 0.104 & 0.108 \\
oracle XUW & $(X, U, W)$ & efficiency: adds $W$, a parent of $f_Y$ & 0.025 & 0.105 & 0.107 \\
\bottomrule
\end{tabularx}
\end{table}

\FloatBarrier
\subsubsection{Proxy relevance}
\label{app:prior_relevance}

Completeness (\Cref{as:completeness}) asks that $W$ and $Z$ carry \emph{enough variation} in $U$ for the bridge functions to be identified, not merely that they are wide enough. \Cref{app:prior_dimensions,app:prior_postprocess} enforce only a counting surrogate -- $d_W,d_Z\ge d_U$ and its categorical analogue $\mathrm{dominates}(P,U)$ -- which holds on $100\%$ of delivered datasets by construction and is silent on whether satisfying it buys \emph{functional} recoverability of $U$. We audit that directly with a cross-fitted regression: for each column $u_j$ of $U$ and each proxy $P\in\{W,Z\}$, the incremental $\Delta R^2_P(u_j)=R^2(u_j\mid X,P)-R^2(u_j\mid X)$, incremental over $X$ because $X$ is itself a parent of $U$ (Equation~\ref{eq:uwz}) and a proxy that merely re-encoded it would be useless for identification however high its raw $R^2$. We aggregate each proxy's $\Delta R^2$ over $U$'s columns two ways: the mean (treating every column as equally important) and the minimum, since completeness is a claim about \emph{all} of $U$ and one unrecoverable direction is enough to break it.

\Cref{fig:prior_relevance} reads the same quantity three ways. Panel (a) shows both proxies are relevant on the typical draw (median $\Delta R^2\approx0.28$, over $93\%$ of datasets positive) but the weakest $U$ column is a materially harder story (median $\Delta R^2\approx0.05$, roughly $60\%$ of datasets below the $0.1$ weak line). Panel (b) conditions on each proxy's own dimension slack $d_P-d_U$: relevance rises with slack, validating the heuristic's spirit. Panel (c) conditions on the proxy mechanism's sampled noise share, and relevance falls as that share grows, as the dilution story predicts.  \Cref{tab:prior_relevance} separates the two components of slack. Mean relevance has a modest positive correlation with proxy width ($\rho=0.166$ for both proxies). Weakest-column relevance rises with slack ($\rho\approx0.28$) and falls with $d_U$ ($\rho\approx-0.26$), but has almost no association with proxy width alone ($\rho=0.034$ for $W$, $0.017$ for $Z$). Because the prior samples proxy width conditional on $d_U$, wider proxies tend to accompany higher-dimensional $U$; recovering the weakest column can therefore remain difficult as proxy width increases, especially since taking a minimum over more $U$ columns tends to lower that statistic.

\begin{figure}[t]
    \centering
    \includegraphics{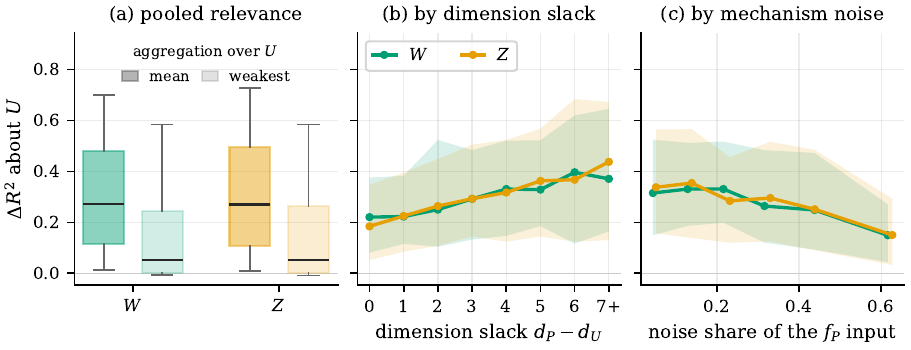}
    \caption{Proxy relevance measured by the cross-fitted increment $\Delta R^2_P(u_j)=R^2(u_j\mid X,P)-R^2(u_j\mid X)$, shown separately for $W$ and $Z$. (a) Distribution of the mean and weakest-column $\Delta R^2$ over $U$; boxes show the IQR and whiskers the $10$th--$90$th percentiles. (b) Median $\Delta R^2$ by proxy dimension slack $d_P-d_U$, with slack $\ge7$ pooled. (c) Median $\Delta R^2$ by sampled noise share of $f_P$, in six equal-frequency bins per proxy. Bands in (b) and (c) show the IQR.}
    \label{fig:prior_relevance}
\end{figure}

\begin{table}
\centering
\caption{Proxy-relevance audit (\Cref{as:completeness}): cross-fitted incremental $\Delta R^2 = R^2(u_j \mid X, P) - R^2(u_j \mid X)$ of each proxy $P$ for the latent columns, aggregated over $U$ by the mean or by the weakest column. The counting heuristic (\Cref{app:prior_postprocess}) passes for every dataset, so the \emph{share} column is the rate at which it and functional recoverability disagree. $\rho$ is Spearman's correlation with each candidate driver: the proxy's own width $d_P$, $d_U$, the slack $d_P - d_U$ that conflates the two, and the sampled noise share of the $f_P$ input.}
\label{tab:prior_relevance}
\small
\begin{tabular}{lrrrrrr}
\toprule
aggregation over $U$ & median $\Delta R^2$ & share $< 0.1$ & $\rho$ $d_P$ & $\rho$ $d_U$ & $\rho$ $d_P - d_U$ & $\rho$ noise \\
\midrule
\multicolumn{7}{l}{\textbf{$W$}} \\
mean & 0.273 & 0.226 & 0.166 & -0.002 & 0.208 & -0.202 \\
weakest & 0.053 & 0.593 & 0.034 & -0.249 & 0.277 & -0.140 \\
\midrule
\multicolumn{7}{l}{\textbf{$Z$}} \\
mean & 0.271 & 0.235 & 0.166 & -0.023 & 0.226 & -0.208 \\
weakest & 0.052 & 0.588 & 0.017 & -0.272 & 0.281 & -0.095 \\
\bottomrule
\end{tabular}
\end{table}

\clearpage
\FloatBarrier

\clearpage
\section{Model architecture, objective, training, and validation}
\label{app:model}

\subsection{Architecture}
\label{app:model_arch}

\modelname encodes a dataset in the three sequential stages of TabICLv2 \citep{quTabICLv2BetterFaster2026}, which we use unchanged. A \emph{column embedder} maps each scalar cell to an embedding through a shared inducing-point set transformer applied along the unit axis, so that a cell is represented relative to the empirical distribution of its own column. A \emph{row interactor} then treats the embedded cells of a unit as a sequence of tokens along the column axis, applies a transformer with rotary positional encoding, and concatenates the outputs of $n_{\mathrm{cls}}$ prepended learnable CLS tokens into a single vector per unit. Finally, an \emph{in-context learner} runs a transformer over the sequence of unit vectors in which query units attend to context units but not to one another; no positional encoding is applied along the unit axis, so the model is equivariant to the ordering of the dataset. All blocks are pre-norm and use the query-aware scalable softmax of \citet{quTabICLv2BetterFaster2026}. Our only removals are the classification-specific components, which a continuous causal estimand does not use; the remaining departures are described below.

\paragraph{A target that is never observed.}
Every departure below follows from one change of task. In tabular in-context learning the query is a context row with a missing label: the model predicts $y^\star$, the same variable it has just seen realized in the context, and the query differs from a context unit only in that one entry is withheld. \modelname predicts $Q(x^\star) = \E[Y(1) - Y(0) \mid X = x^\star]$, which is realized for no unit in the table. The observed columns are evidence about the mechanism that generates it, and supervision comes from the prior's oracle targets rather than from held-out labels.

Two context columns consequently have no query counterpart. The outcome $y$ is context-only in TabICL as well, but there it is withheld because it is the answer; here it is withheld because a query unit is a covariate profile, not a realized unit, and the same holds for the treatment $a$: the estimand averages over both arms, so no value of $a$ can sensibly be attached to the query vector $x^\star$. The proxies are a different case. A query unit can have a $w^\star$ and a $z^\star$, and the prior generates them, so withholding them is a design choice rather than a structural necessity: conditioning on them would target $\E[Y(1) - Y(0) \mid X, W, Z]$ instead. We follow the proximal literature in conditioning on the baseline covariates alone (\Cref{sec:bridge_functions}) and marginalizing over the proxies. At inference the model therefore sees the context table in full and a query reduced to $x^\star$ alone. The modifications below adapt the backbone to exactly this interface.

\paragraph{Structural role embeddings.}
The backbone from TabICL consumes a single undifferentiated feature matrix, whereas proximal identification treats the three observed blocks asymmetrically: the roles of $Z$, $W$ and $X$ are necessary to be known in order to infer the causal estimand of interest. Nothing in the values themselves reveals that role. Column order alone cannot supply it either: the block sizes $(d_X, d_Z, d_W)$ are resampled for every dataset, so no fixed column index marks where one block ends and the next begins, and the model would have to infer the two boundaries from the data before it could use them. We therefore tag the role explicitly. Writing $\mathrm{lin}(\cdot)$ for the shared cell projection into $\R^E$, a cell $v$ of block $g \in \{X, Z, W\}$ enters the column embedder as
\begin{equation}
    \mathrm{lin}(v) + r_g, \qquad r_X,\, r_Z,\, r_W \in \R^{E} \text{ learnable},
    \label{eq:app_role_embedding}
\end{equation}
after which the three blocks are concatenated along the column axis into a single sequence of $d_X + d_Z + d_W$ columns. Role identity is thus carried by the embedding rather than by position. The treatment $A$ is not passed as a column at all: it enters only through the target encoder described next.

\paragraph{Treatment-aware target encoder.}
The remaining two observed variables, $y$ and $a$, are not handled the way $X$, $Z$ and $W$ are. They are present for context units and absent for queries, so instead of being concatenated to the input row as columns, they are added to the tokens of context units. This also settles how $a$ is encoded: it qualifies $y$ rather than standing beside it, since an outcome means one thing under treatment and another under control. We therefore hold two projections and select between them with the observed treatment indicator,
\begin{equation}
    \mathrm{enc}(y_i, a_i) = (1 - a_i)\,\mathrm{lin}_{y_0}(y_i) + a_i\,\mathrm{lin}_{y_1}(y_i),
    \qquad \mathrm{lin}_{y_0},\, \mathrm{lin}_{y_1} : \R \to \R^{E}.
    \label{eq:app_treatment_aware_encoder}
\end{equation}
Outcome and arm are thus encoded jointly, and the treated and control response surfaces may occupy different directions in embedding space from the first layer onwards. The encoder is applied at two points, the column embedder and the in-context learner, with untied parameters.

\paragraph{Standardization and query-row proxy masking.}
Each of $X$, $Z$, $W$ and $y$ is standardized column-wise using the empirical mean and standard deviation of the \emph{context} units only, so that no query information reaches the model through the scale and predictions are invariant to affine rescaling of the inputs.

A query unit is then presented without its proxies. The $Z$ and $W$ positions of query rows are filled with a fixed value of $0$ and then masked out of the attention in the row interactor, so those tokens do not merely carry a neutral value but are invisible to the query representation altogether. The same protocol applies during training and at inference time.

\paragraph{Output head.}
For each query unit the in-context learner emits $B$ logits, passed through a softmax to give the predictive distribution $q_\theta[\,\cdot \mid x^\star, \mathcal{D}_N]$ over a fixed grid of $B$ equal-width bins spanning $[v_{\min}, v_{\max}]$, expressed in units of the context standard deviation of $y$. Discretizing the predictive distribution in this way is standard practice for PFNs \citep{hollmannTabPFNTransformerThat2023} and is what allows a single forward pass to return a full distribution without committing to a parametric family; amortized causal models adopt it as well \citep{balazadehCausalPFNAmortizedCausal2025}. A single head is placed directly over the CATE rather than one head per arm. Although $Q(x)=\mu_1(x)-\mu_0(x)$, the two CEPOs are functionals of the same unknown SCM and are therefore generally dependent under its posterior given $\mathcal{D}_N$. Their marginal posterior distributions alone do not determine the posterior of their difference: convolving them would implicitly impose posterior independence and omit their covariance, which can misstate the CATE uncertainty. Directly learning the distribution of $Q(x)$ avoids this unsupported joint-posterior assumption.

Two properties of the discretization matter for what follows. First, we read the bin probabilities as a piecewise-uniform \emph{density} rather than as point masses at the bin centers. This is what makes the noise convolution of \eqref{eq:app_blurred_histogram} exact, and it puts a floor of $w/\sqrt{12}$ --- the standard deviation of a single bin, of width $w = (v_{\max} - v_{\min})/B$ --- on any posterior standard deviation the head can report. Second, the grid is a hard support constraint: a target outside $[v_{\min}, v_{\max}]$ cannot be represented at all, so the bounds must be wide enough to cover the CATE distribution the prior produces. Chosen $v_{\max}$, $v_{\min}$ and $B$ are reported in \Cref{tab:model_hparams}.

\paragraph{Hyperparameters.}
\Cref{tab:model_hparams} lists the configuration of the three stages. Two entries are worth reading together. The column embedder and the row interactor operate at the embedding dimension $E$, but the in-context learner receives the concatenation of the $n_{\mathrm{cls}}$ CLS tokens and therefore runs at $E \cdot n_{\mathrm{cls}} = 512$; since it is also the deepest stage, it accounts for $28.3$M of the $29.6$M parameters, the column embedder and row interactor contributing $0.9$M and $0.4$M respectively. These parameters are kept identical to those of TabICLv2, so that its pre-trained weights can be used to initialize every component we did not modify.

\begin{table}[t]
    \centering
    \caption{Architectural hyperparameters of \modelname. The three stages are those of TabICLv2, kept at its values so that its pre-trained weights initialize every component we did not modify; the output head is ours. Parameter counts are per stage, the in-context learner's including the head.}
    \label{tab:model_hparams}
    \begin{tabular}{llr}
        \toprule
        Stage & Hyperparameter & Value \\
        \midrule
        \multirow{5}{*}{Shared}
            & Embedding dimension $E$ & 128 \\
            & Feedforward expansion factor & 2 \\
            & Activation & GELU \\
            & Normalization & pre-norm \\
            & Dropout & $0.0$ \\
        \midrule
        \multirow{4}{*}{Column embedder}
            & Blocks & 3 \\
            & Attention heads & 8 \\
            & Inducing points & 128 \\
            & Scalable softmax & query-aware, elementwise \\
        \midrule
        \multirow{4}{*}{Row interactor}
            & Blocks & 3 \\
            & Attention heads & 8 \\
            & CLS tokens $n_{\mathrm{cls}}$ & 4 \\
            & RoPE base & $10^{5}$ \\
        \midrule
        \multirow{3}{*}{In-context learner}
            & Blocks & 12 \\
            & Attention heads & 8 \\
            & Model dimension $E \cdot n_{\mathrm{cls}}$ & 512 \\
        \midrule
        \multirow{2}{*}{Output head}
            & Bins $B$ & 2\,000 \\
            & Grid $[v_{\min}, v_{\max}]$ & $[-5, 5]$ \\
        \midrule
        \multirow{4}{*}{Parameters}
            & Column embedder & $0.9$\,M \\
            & Row interactor & $0.4$\,M \\
            & In-context learner & $28.3$\,M \\
            & Total & $29.6$\,M \\
        \bottomrule
    \end{tabular}
\end{table}

\FloatBarrier
\subsection{The noise-aware histogram loss}
\label{app:noise_aware_loss}

This section gives the closed form of the training objective of \Cref{sec:model_training}, together with the motivation for scoring a \emph{noise-convolved} predictive distribution rather than fitting the Monte Carlo targets directly.

\paragraph{Notation and standardization.}
All quantities below live on the context-standardized scale. For a dataset with $N$ context rows, let $m_N$ and $s_N$ denote the empirical mean and standard deviation of the observed outcomes $\{y_i\}_{i=1}^N$. Outcomes are standardized as $(y - m_N)/s_N$, so the CATE target and its Monte Carlo variance enter as
\begin{equation}
    t_i \coloneqq \frac{\hat Q_K(x_i)}{s_N}, \qquad \varsigma_i^2 \coloneqq \frac{S_i^2}{s_N^2},
    \label{eq:app_standardized_target}
\end{equation}
the location shift $m_N$ cancelling in the difference $\hat\mu_1 - \hat\mu_0$. Because $m_N$ and $s_N$ are computed on context rows only, no query information leaks into the scale.

The head emits $B$ logits over a fixed uniform grid $e_1 < e_2 < \dots < e_{B+1}$ spanning $[v_{\min}, v_{\max}]$ with bin width $w = (v_{\max} - v_{\min})/B$, and $q_\theta[b \mid x_i, \mathcal{D}_N]$ denotes the predicted probability of bin $b = [e_b, e_{b+1})$. We write $\Phi$ for the standard normal CDF. The quantity $\sigma^2$ denotes the variance of an auxiliary Gaussian \emph{smoothing kernel}, expressed on this standardized scale; it is distinct from the row-specific Monte Carlo variance $\varsigma_i^2$ and is annealed during training according to the schedule described in \Cref{app:model_training}.

\paragraph{Closed form.}
Whereas $q_\theta$ is the model's predictive distribution over the estimand $Q(x_i)$, the quantity actually observed during training is the noisy measurement $t_i$; $\tilde{q}_\theta$ is the predictive density that $q_\theta$ induces over that measurement. Write the per-row variance as $\kappa_i^2 \coloneqq \varsigma_i^2 + \sigma^2$, combining the known target noise with the annealed smoothing scale (also known). Conditioning on which bin contains the estimand, the law of total probability gives
\begin{align}
    \tilde{q}_\theta\bigl(t \mid x_i, \mathcal{D}_N\bigr)
    &= \sum_{b=1}^{B}
        \vphantom{\Bigl[}\prob\bigl(Q \in [e_b, e_{b+1}) \mid x_i, \mathcal{D}_N\bigr) \;\cdot\;
        \vphantom{\Bigl[}p\bigl(t \mid Q \in [e_b, e_{b+1})\bigr)
    \nonumber \\[0.6em]
    &= \sum_{b=1}^{B} q_\theta[b \mid x_i, \mathcal{D}_N] \; \frac{1}{w}\left[ \Phi\!\left(\frac{e_{b+1} - t}{\kappa_i}\right) - \Phi\!\left(\frac{e_{b} - t}{\kappa_i}\right) \right].
    \label{eq:app_blurred_histogram}
\end{align}
The first factor is read straight off the head. The second follows from the piecewise-uniform continuous reading of the histogram: given that $Q$ lies in bin $b$, the head places it uniformly within that bin, so
\begin{equation}
    p\bigl(t \mid Q \in [e_b, e_{b+1})\bigr) = \int_{e_b}^{e_{b+1}} \frac{1}{w} \, \mathcal{N}\bigl(t;\, u,\, \kappa_i^2\bigr) \diff u
    = \frac{1}{w}\left[ \Phi\!\left(\frac{e_{b+1} - t}{\kappa_i}\right) - \Phi\!\left(\frac{e_{b} - t}{\kappa_i}\right) \right],
    \label{eq:app_bin_kernel}
\end{equation}
where the integral runs over the \emph{mean} of the Gaussian. The training objective is the negative log-likelihood of the observed target under \eqref{eq:app_blurred_histogram}, evaluated at $t = t_i$:
\begin{equation}
    \mathcal{L}_{\mathrm{noise\_aware}} = -\frac{1}{n_q} \sum_{i=1}^{n_q} \log \left( \sum_{b=1}^{B} q_\theta[b \mid x_i, \mathcal{D}_N] \; \frac{1}{w} \left[ \Phi\!\left(\frac{e_{b+1} - t_i}{\kappa_i}\right) - \Phi\!\left(\frac{e_{b} - t_i}{\kappa_i}\right) \right] \right).
    \label{eq:app_noise_aware_loss}
\end{equation}

\paragraph{Why convolve the model rather than fit the target.}
Consider the population risk of \eqref{eq:app_noise_aware_loss} at fixed context. The target is distributed as $t \sim p(\,\cdot \mid \mathcal{D}_N) \ast \mathcal{N}(0, \varsigma^2)$, the posterior over the estimand blurred by Monte Carlo error. A negative log-likelihood is minimized when the scored density matches the density of the data, i.e. when
\begin{equation}
    q_\theta \ast \mathcal{N}(0, \varsigma^2 + \sigma^2) = p(\,\cdot \mid \mathcal{D}_N) \ast \mathcal{N}(0, \varsigma^2)
    \quad \Longleftrightarrow \quad
    q_\theta \ast \mathcal{N}(0, \sigma^2) = p(\,\cdot \mid \mathcal{D}_N).
    \label{eq:app_fixed_point}
\end{equation}
The Monte Carlo variance cancels on both sides: the correction is exact for \emph{any} value of $\sigma$, and the recovered $q_\theta$ converges to the posterior as $\sigma$ is annealed. Fitting the same head directly to $t_i$ instead---the noise-unaware objective---has population optimum $q_\theta = p(\,\cdot \mid \mathcal{D}_N) \ast \mathcal{N}(0, \varsigma^2 + \sigma^2)$, over-dispersed by exactly the target noise. The distortion is not a constant offset that could be calibrated away post hoc: $\varsigma_i^2 = \Var(\tau \mid X = x_i)/K$ varies across query points, and is largest precisely where the individual treatment effect is most heterogeneous, so the inflation is heaviest in the regions where posterior width is most informative.

\paragraph{Full objective.}
To align the point prediction with the estimand, a squared-error penalty on the posterior mean is added,
\begin{equation}
    \mathcal{L}_{\mathrm{mean}} = \frac{1}{n_q} \sum_{i=1}^{n_q} \Bigl( \E_{q_\theta}\bigl[Q \mid x_i, \mathcal{D}_N\bigr] - t_i \Bigr)^2,
    \qquad
    \E_{q_\theta}\bigl[Q \mid x_i, \mathcal{D}_N\bigr] = \sum_{b=1}^{B} q_\theta[b \mid x_i, \mathcal{D}_N]\, c_b,
    \label{eq:app_mean_penalty}
\end{equation}
with $c_b = (e_b + e_{b+1})/2$ the bin centers, giving $\mathcal{L} = \mathcal{L}_{\mathrm{noise\_aware}} + \lambda_{\mathrm{mean}} \mathcal{L}_{\mathrm{mean}}$. The penalty does not compete with the distributional objective: because the Monte Carlo error is zero-mean, the minimizer of \eqref{eq:app_noise_aware_loss} already carries the posterior mean and hence minimizes \eqref{eq:app_mean_penalty} as well, leaving the population optimum of $\mathcal{L}$ unchanged for any $\lambda_{\mathrm{mean}} \ge 0$. At inference time the predictive distribution is reported \emph{un-blurred}, as $q_\theta(\cdot \mid x^\star, \mathcal{D}_N)$.

\FloatBarrier
\subsection{Training the model}
\label{app:model_training}

All results in this paper come from a single training run, and every number below is a property of that one run.

\paragraph{Configuration.}
\Cref{tab:training_hparams} reports the configuration in full: the prior the datasets are drawn from, the objective and its smoothing schedule, the optimizer, the pre-trained backbone the run is initialized from, and the validation sets used for checkpoint selection. Training runs for $75\,000$ steps at an effective batch of $64$ datasets, accumulated in micro-batches of $16$, for $4.8$M datasets consumed in total --- each drawn fresh from the prior and seen exactly once, so there is no notion of an epoch and no train/validation split on the prior itself.

\begin{table}[t]
    \centering
    \caption{Training configuration of \modelname: the prior the model is trained on, the objective and its smoothing schedule, the optimizer, the pre-trained backbone it is initialized from, and the validation sets used for checkpoint selection.}
    \label{tab:training_hparams}
    \begin{tabular}{llr}
        \toprule
        Stage & Hyperparameter & Value \\
        \midrule
        \multirow{8}{*}{Prior}
            & Confounder dimension $d_U$ & $1$--$10$ \\
            & Treatment proxy dimension $d_Z$ & $1$--$10$ \\
            & Outcome proxy dimension $d_W$ & $1$--$10$ \\
            & Covariate dimension $d_X$ & $0$--$10$ \\
            & Context units $N$ & $64$--$2\,048$ (uniform) \\
            & Query units $n_q$ & 128 \\
            & Monte Carlo worlds $K$ & 250 \\
            & MC targets on context units & no \\
        \midrule
        \multirow{9}{*}{Objective}
            & Loss & \texttt{noise\_aware} \\
            & Mean penalty $\lambda_{\mathrm{mean}}$ & $0.5$ \\
            & Bins $B$ & 2\,000 \\
            & Grid $[v_{\min}, v_{\max}]$ & $[-5, 5]$ \\
            & Bin width $w$ & $0.005$ \\
            & Initial smoothing $\sigma_0$ & $0.5$ \\
            & $\sigma$ held constant for & 500 steps \\
            & $\sigma$ linear decay over & 3\,000 steps \\
            & Final smoothing $\sigma$ & $0.01$ ($2w$) \\
        \midrule
        \multirow{10}{*}{Optimization}
            & Optimizer & AdamW \\
            & Learning rate & $1 \times 10^{-4}$ \\
            & Schedule & \texttt{cosine\_warmup} \\
            & Warmup & 2\,500 steps \\
            & Weight decay & $0$ \\
            & Gradient clipping & $1$ \\
            & Training steps & 75\,000 \\
            & Datasets per step & 64 (16 $\times$ 4 accumulated) \\
            & Precision & \texttt{bfloat16} (AMP) \\
            & Seed (numpy / torch) & 1234 / 1234 \\
        \midrule
        \multirow{3}{*}{Initialization}
            & Init. checkpoint & \texttt{ckpt\_tabicl/tabicl-regressor-v2-20260212.ckpt} \\
            & Transfer rules & \texttt{ckpt\_tabicl/transfer\_rules.yaml} \\
            & Weights only (fresh optimizer) & yes \\
        \midrule
        \multirow{5}{*}{Validation}
            & Prior set & \texttt{prior\_val\_v1} \\
            & Prior set: episodes / interval & 500 / 200 steps \\
            & Realistic set & \texttt{synt\_val\_v1} \\
            & Realistic set: episodes / interval & 480 / 100 steps \\
            & Selection metric & \texttt{synt\_dist/crps} (min) \\
        \bottomrule
    \end{tabular}
\end{table}

\paragraph{Schedules and training curves.}
\Cref{fig:training_curves} shows the objective and the two schedules. The learning rate follows a linear warmup over $2\,500$ steps into a cosine decay; the smoothing scale $\sigma$ of \eqref{eq:app_noise_aware_loss} is held at its initial value for $500$ steps and then annealed linearly over $3\,000$ steps to two bin widths, its final value for the remaining of the training (panel~c). The anneal is the reason the loss level is not comparable across the shaded window: changing $\sigma$ changes the training objective, so movement during the anneal cannot be attributed solely to learning.

Panels (a) and (b) show that, after the initial phase, the total objective becomes dominated by the distributional term, the mean term staying a small regularizer term. Once the $\sigma$ anneal is complete, both the distributional term and the total objective continue to decrease through step $75\,000$, but at a diminishing rate. Panel~(d) zooms on the mean term of the loss. We observe that it follows the behaviour of the distributional term closely.

\begin{figure}
    \centering
    \includegraphics{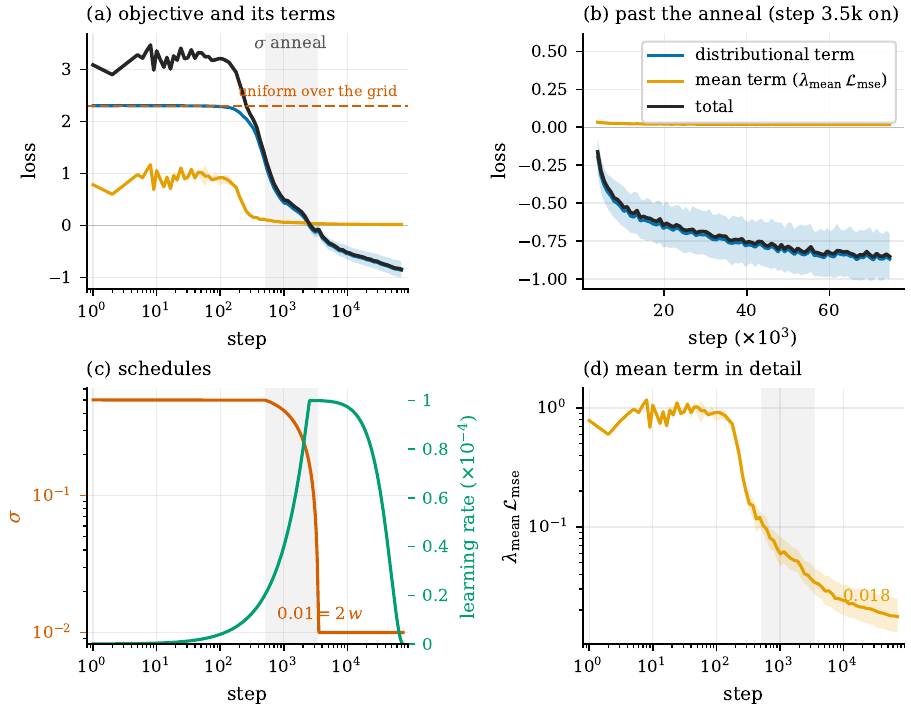}
    \caption{Training curves for \modelname. \textbf{(a)} The objective and its two terms against log step, with the $\sigma$-annealing window shaded; the dashed line is the loss of a uniform density over the CATE grid, matching the value at initialisation. \textbf{(b)} The same curves on a linear step axis after the anneal. \textbf{(c)} The $\sigma$ and learning-rate schedules. \textbf{(d)} A magnified view of the weighted mean term, $\lambda_{\mathrm{mean}}\mathcal{L}_{\mathrm{mse}}$, with the $\sigma$-annealing window again shaded. Stochastic curves show the median and interquartile band within log-spaced step bins in panels (a) and (d), and linearly spaced bins in panel (b); the deterministic schedules in panel (c) are shown without binning.}    \label{fig:training_curves}
\end{figure}

\paragraph{Validity of the objective.}
Three logged diagnostics establish that the loss the curves above report was well posed throughout. First, the grid of the output head is a hard support constraint: a target outside $[v_{\min}, v_{\max}]$ cannot be represented by any forecast on the grid, so a non-negligible out-of-grid fraction would invalidate every number scored on it. Averaged over the CATE targets in the on-the-fly synthetic training prior, the share of targets falling outside the grid is $2.0 \times 10^{-4}$, i.e.\ $0.02\%$, exceeding $1\%$ on $275$ of $75\,000$ steps. 

Second, the Monte Carlo targets carried their full replicate budget. Each query row's CATE target is an average over $K$ replicate draws, and the realized $K$ can be capped below the requested value by the generator's memory budget. The mean realized $K$ was $250$ --- the requested value --- at every logged step, so the cap never bound.

Third, sampling noise remained material at $K=250$ Monte Carlo replicates. The ratio of the mean target-noise standard deviation $\varsigma_i$ to the mean predicted posterior standard deviation averaged $0.20$, rising from $0.15$ shortly after the anneal to $0.22$ over the final third as the posterior sharpened while $\varsigma_i \propto K^{-1/2}$ remained fixed. Ignoring this noise would therefore add the heterogeneous per-row variance $\varsigma_i^2$ to the learned posterior variance, motivating the noise-aware correction.

\paragraph{Cost.}
\Cref{tab:training_cost} summarizes the computational cost of the run. Interestingly, model optimization was not the dominant expense: generating the synthetic prior consumed $43.46$ GPU-hours ($65.7\%$), approximately $2.4\times$ the cost of the forward and backward passes ($27.4\%$). Validation accounted for a further $5.4\%$, while checkpointing, logging, and startup contributed only $1.5\%$. In total, training consumed $66.15$ H100 GPU-hours over $75\,000$ steps and $4.8$ million synthetic datasets, averaging $3.17$ seconds per step.

\begin{table}
    \centering
    \caption{Cost of the \modelname training run, as measured GPU-hours. Validation is the two frozen evaluation datasets, scored at their logging intervals throughout training.}
    \label{tab:training_cost}
    \begin{tabular}{llrr}
        \toprule
        & & Hours & \% \\
        \midrule
        \multirow{5}{*}{Compute}
            & Prior generation & $43.46$ & $65.7$ \\
            & Forward / backward & $18.12$ & $27.4$ \\
            & Validation, prior set (376 passes) & $1.27$ & $1.9$ \\
            & Validation, realistic set (751 passes) & $2.34$ & $3.5$ \\
            & Checkpointing, logging, startup & $0.96$ & $1.5$ \\
        \midrule
        Total
            & GPU-hours ($1 \times$ H100 NVL) & $66.15$ & $100.0$ \\
        \midrule
        \multirow{3}{*}{Scale}
            & Training steps & 75\,000 &  \\
            & Synthetic datasets consumed & 4\,800\,000 &  \\
            & Seconds per step & $3.17$ &  \\
        \bottomrule
    \end{tabular}
\end{table}

\subsubsection{Post-training on larger context sizes}
\label{app:model_posttraining}

The main training run represents contexts of at most $2\,048$ observations.  We therefore perform a post-training experiment to test whether the limited support at larger context sizes contributes to the plateau observed in \Cref{sec:results_performance}.  Starting from the completed main-training checkpoint, we optimize for a total of $30\,000$ steps on the same causal prior and noise-aware objective, but draw context sizes uniformly from $64$ to $5\,000$ observations (with $128$ query rows).  Uniform, rather than log-uniform, sampling ensures that the adaptation gives substantial weight to the newly introduced large contexts while retaining smaller contexts for rehearsal.  We retain the effective meta-batch size of $64$ datasets, use a lower learning rate of $2\times10^{-5}$ with a $500$-step warmup, and enable activation recomputation to fit the longer sequences.

To make this adaptation inexpensive and limit changes to the learned representation, we load the complete main-training checkpoint and freeze all parameters except the in-context predictor.  Within that predictor, its first six transformer blocks are frozen as well, so only its later layers and unfrozen predictor parameters are optimized.  We evaluate the resulting checkpoints on separate frozen prior and realistic validation suites constructed at large context sizes.

\FloatBarrier
\subsection{Validation and posterior diagnostics}
\label{app:model_validation}

\subsubsection{Validation datasets}
\label{app:model_val_datasets}

We monitor the training of \modelname on a fixed suite of $480$ datasets, balanced across $12$ hand-written realistic proximal mechanisms. These mechanisms cover clinical severity and frailty, latent patient subtypes, threshold-based treatment assignment, utilization counts, censored survival, marketing uplift, educational selection, wage returns, hierarchical site effects, nonlinear ordinal proxies, interaction-driven confounding, and near-null effects under strong confounding. These settings span a diverse range of datasets: continuous, binary, ordinal and categorical variables, the treatment is binary, the outcomes may be continuous, skewed, count-valued, zero-inflated, or censored.

For each dataset, the dimensions of $U$, $X$, $W$ and $Z$ are sampled with the same procedure as in the prior (\Cref{app:prior_dimensions}). Each episode contains exactly $2048$ context rows and $128$ query points. The query target CATEs are computed by averaging $10\,000$ Monte Carlo replicates of ITE draws. At this replicate count, the Monte Carlo sampling noise is approximately $0.025$ times the predictive posterior standard deviation (of the last model checkpoint). The same fixed validation set is evaluated at every training checkpoint.

\subsubsection{Point-estimate validation}
\label{app:model_val_point}

We measure the point-estimate performance at both ATE and CATE levels. The ATE prediction is estimated as the average of the CATE predictions from the query set,
$$
\widehat{\mathrm{ATE}} = \frac{1}{n_q}\sum_{i=1}^{n_q}\hat{\tau}_i,
$$
and so is the target ATE also computed from the references CATEs. To evaluate the ATE estimation, we use the root mean squared error (RMSE) and the bias of the ATE normalized by the standard deviation of the episode's CATEs.

Regarding CATE evaluation, we use normalized precision in estimation of heterogeneous effect (nPEHE), 
\begin{equation}
    \mathrm{nPEHE}=\sqrt{\frac{\frac{1}{n_q}\sum_{i=1}^{n_q}(\hat{\tau}_i-\tau_i)^2}{\Var_i(\tau_i)}}.
\label{eq:npehe}
\end{equation}
We also measure the shape-$R^2$, which isolates the recovery of treatment-effect heterogeneity from
constant level error,
\begin{equation}
    R^2_{\mathrm{shape}} = 1- \frac{\Var_i(\hat{\tau}_i-\tau_i)} {\Var_i(\tau_i)}.
    \label{eq:r2_shape}
\end{equation}
A constant-ATE prediction gives $\mathrm{nPEHE}=1$ and $R^2_{\mathrm{shape}}=0$.

\Cref{fig:point_validation} shows the evolution of point-estimate performance on the fixed realistic-mechanism validation suite. In the initial stage of training, the model performs worse than the constant-ATE baseline (CATE nPEHE above $1$ and the shape $R^2$ is close to or below zero). After approximately $3\,500$ training steps, the median nPEHE falls below the constant-ATE baseline and the shape $R^2$ becomes positive. Performance then continues to improve throughout training: ATE RMSE decreases, normalized ATE bias moves towards zero, nPEHE falls further, and the shape $R^2$ increases. The narrowing interquartile bands indicate that performance also becomes more consistent across validation episodes.

\begin{figure}
    \centering
    \includegraphics{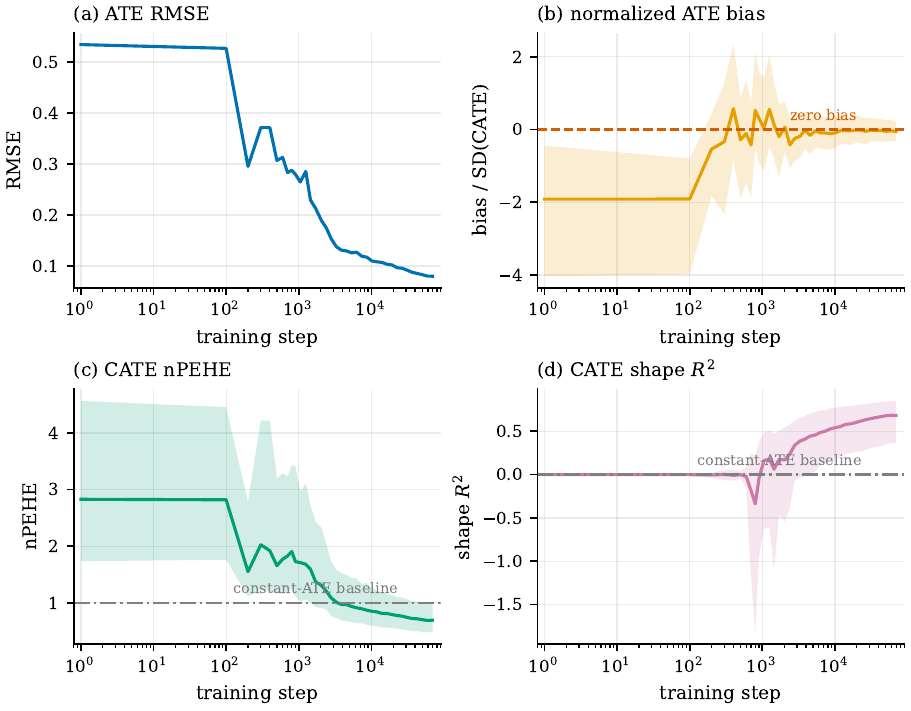}
    \caption{Evolution of point-estimate performance during training on the fixed realistic-mechanism validation suite (\Cref{app:model_val_datasets}). Curves show the validation metric at each evaluation step. The shaded regions in panels (b)--(d) indicate the interquartile range across validation episodes. \textbf{(a)} ATE RMSE, a pooled metric and therefore shown without an episode-level band. \textbf{(b)} normalized ATE bias, where zero denotes no bias. \textbf{(c)} CATE nPEHE, with the dashed line marking the constant-ATE baseline ($\mathrm{nPEHE}=1$). \textbf{(d)} CATE shape $R^2$, with the dashed line marking the corresponding constant-ATE baseline ($R^2_{\mathrm{shape}}=0$). Lower values are preferred in panels (a) and (c), while panel (b) is centered at zero and higher values are preferred in panel (d). For readability, curves and shaded bands are lightly smoothed by taking medians within log-spaced training-step bins.}    
    \label{fig:point_validation}
\end{figure}

\subsubsection{Distributional validation}
\label{app:model_val_distri}

We evaluate the predictive posterior using the continuous ranked probability score (CRPS), the predictive standard deviation relative to the prediction error, and empirical coverage of central credible intervals. For query row $i$, let $q_i$ denote the predicted distribution over $Q(x_i)$, with CDF $F_i$, mean $\mu_i$, and variance $v_i^2$. Let $Q_{K,i}$ denote the Monte Carlo CATE target and let $\varsigma_i^2$ denote its estimated sampling variance.

The CRPS evaluates the full predictive distribution against the target:
\begin{equation}
    \operatorname{CRPS}(q_i, Q_{K,i})
    =
    \int
    \left[
        F_i(t) - \1\{t \geq Q_{K,i}\}
    \right]^2 \,\mathrm{d}t .
    \label{eq:validation_crps}
\end{equation}

We also report the ratio of predictive standard deviation to prediction error,
\begin{equation}
    \mathrm{SD/RMSE}
    =
    \sqrt{
        \frac{
            \frac{1}{M}\sum_{i=1}^{n_q} v_i^2
        }{
            \frac{1}{n_q}\sum_{i=1}^{n_q}
            \left(\mu_i-Q_{K,i}\right)^2
            -
            \frac{1}{M}\sum_{i=1}^{n_q}\varsigma_i^2
        }
    },
    \label{eq:validation_sd_over_rmse}
\end{equation}
where $n_q$ is the total number of validation query rows. The subtraction of $\varsigma_i^2$ removes the Monte Carlo noise from the observed prediction error. This denominator equals the average predictive variance under calibration, in expectation. Values below one indicate overconfident predictions, whereas values above one indicate underconfidence.

\Cref{fig:distri_validation} displays the evolution of these distributional metrics during training. CRPS decreases as the model learns. Both the ratio SD/RMSE and the empirical coverages converge to overconfident states.

\begin{figure}
    \centering
    \includegraphics{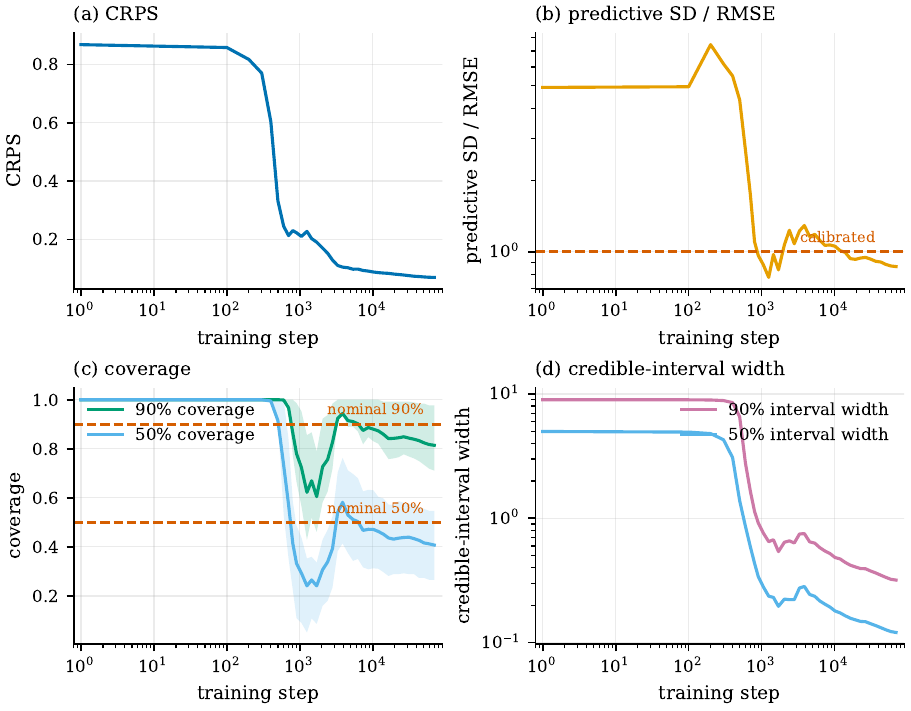}
    \caption{Evolution of distributional predictive performance during training on the fixed realistic-mechanism validation suite (\Cref{app:model_val_datasets}). \textbf{(a)} CRPS, for which lower values indicate better probabilistic predictions. \textbf{(b)} predictive standard deviation divided by RMSE, with the dashed line at one denoting calibrated predictive spread. \textbf{(c)} empirical 90\% and 50\% coverages, with the dashed lines marking the nominal coverage levels and the shaded regions showing the interquartile ranges across validation episodes. \textbf{(d)} widths of the central 90\% and 50\% credible intervals, pooled over validation rows. Panels (b) and (d) use logarithmic y-axes. For readability, curves are lightly smoothed by taking medians within log-spaced training-step bins.}    
    \label{fig:distri_validation}
\end{figure}

\subsubsection{Posterior representation \& smoothing}
\label{app:model_val_smoothing}

The histogram representation of the predictive posterior can produce visually spiky densities. Small changes in the allocation of probability across neighbouring bins may have little effect on the training objective while producing noticeably different density shapes. We therefore use the CDF as an intermediate representation for visualization. A Savitzky--Golay filter is applied to the raw CDF, and the corresponding density is obtained by numerical differentiation. As illustrated in \Cref{fig:posterior_spikes}, the smoothed CDF stays very close to the original CDF, but its implied density is easier to interpret. This postprocessing is applied only for visual inspection and does not alter the predictive distribution used to compute validation metrics. Future work may investigate regularization techniques during training to encourage smoother posterior estimates.

\begin{figure}
    \centering
    \includegraphics{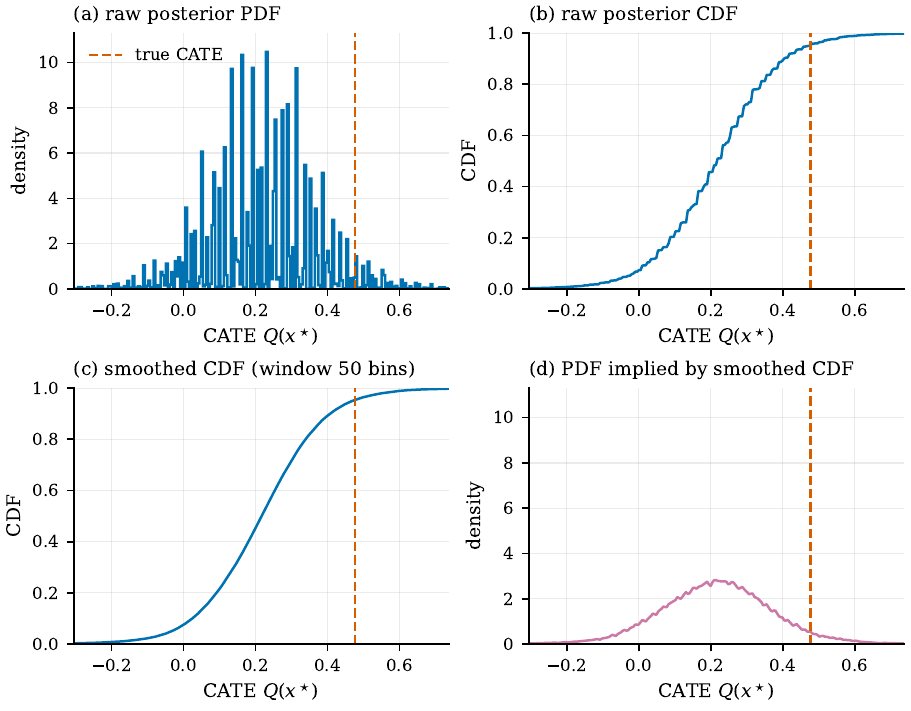}
    \caption{Illustration of the spiky posterior-density representation for one selected query point from the validation suite. \textbf{(a)} Raw posterior density induced by the model's histogram output. \textbf{(b)} Corresponding raw CDF. \textbf{(c)} CDF after Savitzky--Golay smoothing with a 50-point window and polynomial degree three. \textbf{(d)} Density obtained by differentiating the smoothed CDF. The smoothed CDF closely follows the raw CDF while yielding a more visually interpretable density. The dashed line marks the true CATE. Smoothing is used for visualization only; all validation metrics are computed from the original posterior representation.}
    \label{fig:posterior_spikes}
\end{figure}

\subsubsection{Posterior dynamics}
\label{app:model_val_posterior_dynamics}

\Cref{fig:posterior_contraction} shows that posterior width is initially independent of context size: at the first checkpoints, all context sizes produce similarly wide predictions. Over the first thousands of training steps, the width stays independent of the context size, but becomes narrower. Later in training, context-dependent contraction becomes evident, with larger context sizes producing narrower posteriors while smaller contexts remain wider. At the final checkpoint, the mean predictive standard deviation decreases from approximately $0.54$ at $n_{\mathrm{ctx}}=16$ to $0.10$ at $n_{\mathrm{ctx}}=2\,048$.

\begin{figure}
    \centering
    \includegraphics{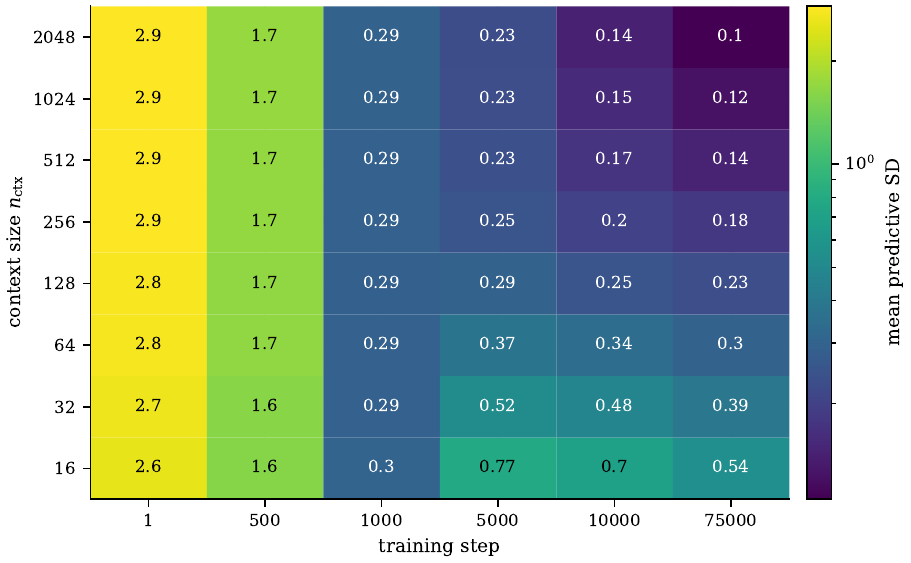}
    \caption{Posterior-width contraction during training on the fixed realistic-mechanism validation suite (\Cref{app:model_val_datasets}). Each cell reports the mean predictive posterior standard deviation averaged over all validation query rows, for a given training checkpoint and context size. Widths are expressed in standardized units and use the piecewise-uniform posterior variance. The shared colour scale is logarithmic.}
    \label{fig:posterior_contraction}
\end{figure}

\Cref{fig:posterior_evolution} illustrates the evolution of the predictive posterior for a fixed query point. Early in training, the posterior is diffuse and weakly dependent on context size. As training progresses, it concentrates around the true CATE, with the larger context producing a noticeably sharper final posterior.

\begin{figure}
    \centering
    \includegraphics{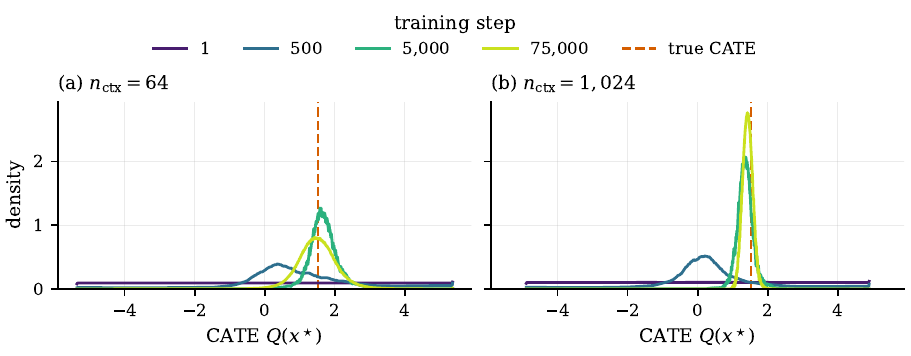}
    \caption{Evolution of the predictive CATE posterior for one fixed query point from the realistic-mechanism validation suite (\Cref{app:model_val_datasets}). Curves show the density obtained from the posterior CDF at four training checkpoints, with training steps indicated by colour. The same query point is evaluated with \textbf{(a)} $n_{\mathrm{ctx}}=64$ and \textbf{(b)} $n_{\mathrm{ctx}}=1\,024$. The dashed line marks the true CATE. Densities are shown in standardized units. The CDFs are smoothed for visualization, following \Cref{app:model_val_smoothing}.}
    \label{fig:posterior_evolution}
\end{figure}

\clearpage
\section{Details on results}
\label{app:results}

\subsection{Baseline methods}
\label{app:results_baselineMethods}

The baseline methods are either newly implemented or adjusted from existing codebases. All relevant code and configurations are available in our codebase.

\setlength{\LTleft}{0pt}
\setlength{\LTright}{\fill}
\setlength{\LTcapwidth}{\textwidth}
\begin{longtable}{@{}p{0.25\textwidth}p{0.16\textwidth}p{0.35\textwidth}p{0.15\textwidth}@{}}
  \caption{\raggedright Catalogue of CATE predictors selected in the main benchmark.}
  \label{tab:results_method_catalogue}\\
    \toprule
    Predictor & Family & Fitted estimator and prediction rule & Details \\
    \midrule
    \endfirsthead
    \caption[]{\raggedright Catalogue of CATE predictors selected in the main benchmark (continued).}\\
    \toprule
    Predictor & Family & Fitted estimator and prediction rule & Details \\
    \midrule
    \endhead
    \midrule
    \multicolumn{4}{r}{Continued on next page} \\
    \endfoot
    \bottomrule
    \endlastfoot
    S-learner (ridge) & Backdoor meta-learner & S-learner with ridge regression. & \Cref{app:results_backdoor} \\
    S-learner (RF) & Backdoor meta-learner & S-learner with random forest. & \Cref{app:results_backdoor} \\
    S-learner (TabICL) & Backdoor meta-learner & S-learner with zero-shot TabICL. & \Cref{app:results_backdoor} \\
    S-learner (TabICL) + kNN marg. & Backdoor meta-learner & S-TabICL fit on $(X,W,Z)$, then 16-neighbour proxy marginalisation. & \Cref{app:results_backdoor} \\
    S-learner (TabICL) + kernel marg. & Backdoor meta-learner & S-TabICL fit on $(X,W,Z)$, then 16-neighbour kernel proxy marginalisation (scale 1). & \Cref{app:results_backdoor} \\
    T-learner (ridge) & Backdoor meta-learner & T-learner with ridge regression. & \Cref{app:results_backdoor} \\
    T-learner (RF) & Backdoor meta-learner & T-learner with random forest. & \Cref{app:results_backdoor} \\
    T-learner (TabICL) & Backdoor meta-learner & T-learner with zero-shot TabICL. & \Cref{app:results_backdoor} \\
    T-learner (ridge) + kNN marg. & Backdoor meta-learner & T-OLS fit on $(X,W,Z)$, then 16-neighbour proxy marginalisation. & \Cref{app:results_backdoor} \\
    T-learner (ridge) + kernel marg. & Backdoor meta-learner & T-OLS fit on $(X,W,Z)$, then 16-neighbour kernel proxy marginalisation (scale 1). & \Cref{app:results_backdoor} \\
    X-learner (RF) & Backdoor meta-learner & X-learner with random forest. & \Cref{app:results_backdoor} \\
    X-learner (TabICL) & Backdoor meta-learner & X-learner with zero-shot TabICL. & \Cref{app:results_backdoor} \\
    P-learner (ridge) & Proximal & Kernel-bridge P-learner with ridge final-stage CATE regression. & \Cref{app:results_plearner} \\
    P-learner (RF) & Proximal & Kernel-bridge P-learner with random-forest final-stage CATE regression. & \Cref{app:results_plearner} \\
    P-learner (TabICL) & Proximal & Kernel-bridge P-learner with TabICL final-stage CATE regression. & \Cref{app:results_plearner} \\

    KPV (ind. $W$) & Proximal & KPV bridge with $W$ integrated over the context empirical marginal. & \Cref{app:results_kpv} \\
    KPV (cond.\ $W$) & Proximal & KPV bridge with locally conditional $W\mid X=x$ integration. & \Cref{app:results_kpv} \\
    PMMR (ind. $W$) & Proximal & PMMR bridge with $W$ integrated over the context empirical marginal. & \Cref{app:results_pmmr} \\
    PMMR (cond.\ $W$) & Proximal & PMMR bridge with locally conditional $W\mid X=x$ integration. & \Cref{app:results_pmmr} \\
    NMMR-V (fixed; ind. $W$) & Proximal & V-statistic NMMR, fixed RBF width, with marginal $W$ integration. & \Cref{app:results_nmmr} \\
    NMMR-V (fixed; cond.\ $W$) & Proximal & V-statistic NMMR, fixed RBF width, with conditional $W\mid X=x$ integration. & \Cref{app:results_nmmr} \\
    NMMR-U (fixed; ind. $W$) & Proximal & U-statistic NMMR, fixed RBF width, with marginal $W$ integration. & \Cref{app:results_nmmr} \\
    NMMR-U (fixed; cond.\ $W$) & Proximal & U-statistic NMMR, fixed RBF width, with conditional $W\mid X=x$ integration. & \Cref{app:results_nmmr} \\
    NMMR-V (median; ind. $W$) & Proximal & V-statistic NMMR, median-heuristic RBF width, with marginal $W$ integration. & \Cref{app:results_nmmr} \\
    NMMR-V (median; cond.\ $W$) & Proximal & V-statistic NMMR, median-heuristic RBF width, with conditional $W\mid X=x$ integration. & \Cref{app:results_nmmr} \\
    NMMR-U (median; ind. $W$) & Proximal & U-statistic NMMR, median-heuristic RBF width, with marginal $W$ integration. & \Cref{app:results_nmmr} \\
    NMMR-U (median; cond.\ $W$) & Proximal & U-statistic NMMR, median-heuristic RBF width, with conditional $W\mid X=x$ integration. & \Cref{app:results_nmmr} \\
    CausalFM & Causal ICL & Pre-trained causal foundation model, $X$ only. & \Cref{app:results_backdoor} \\
    Do-PFN & Causal ICL & Pre-trained causal foundation model, $X$ only. & \Cref{app:results_backdoor} \\
    CausalPFN & Causal ICL & Pre-trained causal foundation model, $X$ only. & \Cref{app:results_backdoor} \\
    CausalPFN + kNN marg. & Causal ICL & CausalPFN fit on $(X,W,Z)$, then 16-neighbour proxy marginalisation. & \Cref{app:results_backdoor} \\
    CausalPFN + kernel marg. & Causal ICL & CausalPFN fit on $(X,W,Z)$, then 16-neighbour kernel proxy marginalisation (scale 1). & \Cref{app:results_backdoor} \\
    \modelname & Proximal foundation model & Amortized proximal transformer; query proxies are masked. & \Cref{sec:proximalFM} \\
\end{longtable}

\subsubsection{Backdoor meta-learners, in-context learners and proxy marginalization}
\label{app:results_backdoor}

We evaluate some methods estimating the CATE under the backdoor adjustment assumptions. First, we evaluat standard causal meta-learners from the backdoor adjustment framework \citep{kunzelMetalearnersEstimatingHeterogeneous2019}, specifically the S-learner, T-learner, and X-learner. We also implement some in-context learners: CausalPFN \citep{balazadehCausalPFNAmortizedCausal2025}, CausalFM \citep{maFoundationModelsCausal2026} and Do-PFN \citep{robertsonDoPFNInContextLearning2025}. In our proximal causal setting, the presence of the unobserved confounder $U$ directly violates the conditional ignorability assumption these methods rely on ($Y(a) \indep A \mid X$). Consequently, all backdoor methods are misspecified. They serve as empirical baselines.

Because the causal estimand of interest is the covariate-conditional effect $Q(x) = \E[Y(1) - Y(0) \mid X=x]$, yet observational data contain proxy variables $(W, Z)$ that carry information about $U$, we evaluate backdoor methods under two distinct options:
\begin{enumerate}
    \item \textbf{Discard proxies ($X$-only):} The learners are trained exclusively on baseline covariates $X$, discarding $(W, Z)$ entirely. This reflects the standard practice of ignoring unobserved confounding and assuming $Y(a) \indep A \mid X$.
    
    \item \textbf{Proxy conditioning with marginalization ($(X, W, Z) \to X$):} To avoid disadvantaging backdoor methods by withholding proxy features, we also fit learners on the full feature set $(X, W, Z)$. Because this yields a proxy-conditional effect $\hat\tau(x, w, z)$ rather than $Q(x)$, we post-process the predictions by integrating out $(W, Z)$ with respect to a local empirical estimate of $P(W, Z \mid X=x)$. 
\end{enumerate}
Note that even with proxy features and marginalization, Option 2 does not identify $Q(x)$: naive regression is insufficient under the proximal assumptions (\Cref{sec:problem_setup}).

\paragraph{Meta-learner formulations.}

We evaluate three meta-learning architectures \citep{kunzelMetalearnersEstimatingHeterogeneous2019}:
\begin{itemize}
    \item \textbf{S-Learner:} Fits a single regressor $\hat\mu$ on the pooled data predicting $Y$ from features and treatment $A$. The effect is estimated as $\hat\tau_{\mathrm{S}}(x) = \hat\mu(x, 1) - \hat\mu(x, 0)$. When proxies are included, it estimates $\hat\tau_{\mathrm{S}}(x, w, z) = \hat\mu(x, 1, w, z) - \hat\mu(x, 0, w, z)$.
    
    \item \textbf{T-Learner:} Fits separate regressors $\hat\mu_0$ on the control cohort ($A=0$) and $\hat\mu_1$ on the treated cohort ($A=1$), predicting $\hat\tau_{\mathrm{T}} = \hat\mu_1 - \hat\mu_0$. This allows separate response surfaces for each arm, at the cost of not sharing information across arms.
    
    \item \textbf{X-Learner:} A four-stage crossover architecture designed for imbalanced treatments. After fitting base models $\hat\mu_0, \hat\mu_1$, it imputes counterfactual pseudo-effects $D_1 = Y_1 - \hat\mu_0(X_1)$ and $D_0 = \hat\mu_1(X_0) - Y_0$, trains second-stage models $\hat\tau_1$ and $\hat\tau_0$ to predict them, and combines them using an estimated propensity score $\hat{e}$:
    \begin{equation}
        \hat\tau_{\mathrm{X}} = \hat{e}\,\hat\tau_0 + (1 - \hat{e})\,\hat\tau_1.
        \label{eq:x_learner}
    \end{equation}
    In our implementation, $\hat{e}$ is estimated via logistic regression.
\end{itemize}

Each meta-learner accepts three distinct base model classes:
\begin{itemize}
    \item \textbf{Ridge regression:} Standardized linear model with penalty $\alpha \in \{0.1, 1, 10, 100, 1000\}$ selected via 4-fold cross-validation.
    \item \textbf{Random forest:} 100 trees with maximum depth $\in \{1, 3, 5\}$ and minimum leaf size $\in \{1, 5, 10\}$ chosen via 4-fold cross-validation.
    \item \textbf{TabICLv2 \citep{quTabICLv2BetterFaster2026}:} In-context tabular foundation model evaluated zero-shot.
\end{itemize}

\paragraph{In-context learners.}

We evaluate CausalPFN, CausalFM and Do-PFN in their pretrained, zero-shot form: each model uses the episode's observed data as context, with no episode-specific parameter updates. CausalPFN and CausalFM directly return a CATE prediction for each query. Do-PFN instead predicts outcomes under each treatment arm, and we take their difference to form the CATE.

\paragraph{Proxy marginalization.}

When trained on $(X, W, Z)$, a meta-learner predicts a proxy-conditional contrast $\hat\tau(x, w, z)$. Because the evaluation target is the covariate-conditional effect $Q(x) = \E[Y(1) - Y(0) \mid X=x]$, we recover a function of $x$ alone by integrating out the proxies against a local empirical estimate of $P(W, Z \mid X=x)$:
\begin{equation}
    \hat\tau(x) = \sum_{i \in \mathcal{N}_k(x)} q_i(x)\,\hat\tau(x, W_i, Z_i).
    \label{eq:proxy_marginalization}
\end{equation}
Crucially, the integration uses the aligned context pairs $(W_i, Z_i)$ rather than a Cartesian product of marginals, preserving the empirical joint dependence between the two proxies. The neighborhood $\mathcal{N}_k(x)$ comprises the $k$ context points closest to $x$ under a feature-standardized Euclidean distance:
\begin{equation}
    d(x, X_i) = \sqrt{\sum_{d=1}^{d_X} \frac{(x_d - X_{i,d})^2}{\sigma_d^2}}, \qquad \sigma_d = \sqrt{2}\,\mathrm{median}_{j,l}\big(|X_{j,d} - X_{l,d}|\big).
\end{equation}
We evaluate two local weighting schemes for $q_i(x)$:
\begin{itemize}
    \item \textbf{$k$-NN weighting:} Uniform weights over the $k$ nearest context atoms, $q_i(x) = 1/k$.
    \item \textbf{Gaussian kernel weighting:} Nadaraya-Watson exponential weights with bandwidth scale $s > 0$:
    \begin{equation}
        q_i(x) = \frac{\exp\big(-\big(d(x, X_i)^2 - \min_{j \in \mathcal{N}_k(x)} d(x, X_j)^2\big) / s\big)}{\sum_{l \in \mathcal{N}_k(x)} \exp\big(-\big(d(x, X_l)^2 - \min_{j \in \mathcal{N}_k(x)} d(x, X_j)^2\big) / s\big)}.
    \end{equation}
\end{itemize}
In our benchmarks, we use $k =16$ and fixed kernel scale $s = 1.0$.

\subsubsection{P-learner}
\label{app:results_plearner}

We implement the P-learner of \citet{sverdrupProximalCausalLearning2023}. The cross-fitted score and final-stage regression follow their procedure. For bridge estimation, we use a Nystr\"om-approximation of the minimax RKHS formulation, described below. 

For each data point, we compute the doubly robust proximal score:

\begin{equation}
  \Gamma_i = (-1)^{1-A_i}\, q(Z_i,A_i,X_i)\big(Y_i - h(W_i,A_i,X_i)\big)
             + h(W_i,1,X_i) - h(W_i,0,X_i),
  \label{eq:plearner_score}
\end{equation}

The outcome and treatment bridges are characterized, for $a\in\{0,1\}$, by the conditional moments
\begin{equation}
  \begin{split}
  \E\!\left[\1\{A=a\}\{Y-h(W,a,X)\}\mid Z,X\right]&=0,\\
  \E\!\left[\1\{A=a\}q(Z,a,X)-1\mid W,X\right]&=0.
  \end{split}
  \label{eq:plearner_bridge_moments}
\end{equation}

We obtain the CATE estimate by regressing the cross-fitted scores on $X$, $\hat Q(x)=\hat\E[\hat\Gamma\mid X=x]$. The mean of the same scores gives the corresponding proximal ATE estimate.

\paragraph{Kernel bridge estimation.} We estimate both bridges using the regularised minimax RKHS formulation of \citet{dikkalaMinimaxEstimationConditional2020}, following \citet{sverdrupProximalCausalLearning2023}. The outcome bridge $h$ is represented by RBF features $\Phi_H$ of $(W,X)$, and the treatment bridge $q$ by $\Phi_Q$ of $(Z,X)$. Because of computational and memory constraints, we use the Nystr\"om method with $r$ landmark points sampled from the training fold. The kernel bridge estimators have closed-form solutions with penalties $\lambda_h$ and $\lambda_q$ on the two bridges.

\paragraph{Hyperparameter tuning.} Several components of the method is subject to a choice of parameter. First of all, the CATE predictor $\hat{Q}(.)$ can be made using any general predictor. We use a ridge regularized linear regression, a random forest or TabICL \citep{quTabICLv2BetterFaster2026}. The parameters of the random forest and ridge regressor are chosen by cross-validation, TabICL does not require parameter tuning.

Then, regarding the bridge functions $h$ and $q$, the regularizer $\lambda_h$ and $\lambda_q$ are chosen by cross-validation to minimize their moment loss. The Nystr\"om rank $r$ is fixed. The kernel bandwidth is chosen by the median heuristic, and not searched. \Cref{tab:methods_plearner} summarizes the hyperparameters.

\paragraph{Score stabilisation.} The estimated treatment-bridge weight can have large magnitude: we cap it at $q_{\max}$ before
forming \eqref{eq:plearner_score}.

\begin{table}[h]
  \centering
  \caption{P-learner settings. Bridge regularization parameters are selected
           separately within each cross-fitting fold using held-out moment loss.
           Final-stage hyperparameters are selected per episode using
           four-fold cross-validated prediction MSE.}
  \label{tab:methods_plearner}
  \begin{tabular}{lll}
    \toprule
    Component & Setting & Value \\
    \midrule
    Cross-fitting
      & Folds (stratified on $A$)
      & 4 \\

    Bridge class
      & Kernel
      & RBF, blockwise median heuristic \\
      & Nystr\"om rank $r$
      & 256 \\

    Bridge regularization
      & $\lambda_h,\lambda_q$
      & \emph{sweep}:
        $\{10^{-6},10^{-5},10^{-4},10^{-3},10^{-2},10^{-1}\}$ \\
      & Selection criterion
      & Held-out proximal moment loss \\
      & Validation fraction
      & $0.25$ \\
      & $\lambda_g$ (critic loss)
      & $10^{-3}$, fixed \\

    Proximal score
      & Weight clipping $q_{\max}$
      & 20 \\

    Final-stage CV
      & Folds
      & 4 \\
      & Selection criterion
      & Mean squared error \\

    Ridge final stage
      & Model
      & Standardized ridge regression \\
      & Ridge penalty $\alpha$
      & \emph{searched}: $\{0.1,1,10,100,1000\}$ \\

    Random-forest final stage
      & Number of trees
      & 100, fixed \\
      & Maximum depth
      & \emph{searched}: $\{1,3,5\}$ \\
      & Minimum leaf size
      & \emph{searched}: $\{1,5,10\}$ \\
      & Maximum feature fraction
      & 1, fixed \\

    TabICL final stage
      & Model
      & Zero-shot prediction, no tuning\\
    \bottomrule
  \end{tabular}
\end{table}

\subsubsection{Neural moment matching regression}
\label{app:results_nmmr}

We adapted the implementation of the neural moment matching regression (NMMR) method of \citet{kompaDeepLearningMethods2022}\footnote{Original implementation available at \url{https://github.com/beamlab-hsph/Neural-Moment-Matching-Regression}.} to our setting. Their code estimates one population-level quantity per dataset, $\E[Y \mid \mathrm{do}(A=a)]$, evaluated on a grid of treatment values, which for a binary treatment is simply the two arms whose difference is the ATE. Our benchmark instead asks for the CATE $\hat Q(x)$ at each query row of every episode.

\paragraph{The method.} NMMR replaces the outcome bridge $h_0$ of \eqref{eq:bridge} by a neural network $g_\phi(x,w,a) \in \R$ and fits it directly, with no two-stage regression. This neural network is trained on an observational dataset, with simple gradient descent. It uses the following loss function, determined by an RBF kernel matrix on the conditioning variables $A$, $Z$ and $X$:

\begin{equation}
  r_i = Y_i - g_\phi(X_i, W_i, A_i),
  \qquad
  K_{ij} = \exp\!\Big(-\tfrac{\norm{\kappa_i-\kappa_j}^{2}}{2\ell^{2}}\Big),
  \quad
  \kappa_i = \big(A_i,\; Z_i/\hat\sigma_Z,\; X_i/\hat\sigma_X\big),
  \label{eq:nmmr_resid}
\end{equation}
which determines the quadratic forms to minimize
\begin{equation}
  \hat R_V = \frac{r^\T K\, r}{n^{2}}
  \qquad\text{or}\qquad
  \hat R_U = \frac{r^\T \big(K - \mathrm{diag}(K)\big)\, r}{n(n-1)},
  \label{eq:nmmr_uv}
\end{equation}
the V- and U-statistics. The U-statistic drops the diagonal: it is unbiased but noisier, while the V-statistic is biased but lower variance. We report both as separate entries, as in the original paper.

\paragraph{Deviations at prediction time.} The reference code never produces a conditional ATE, even though the bridge it fits supports one. It returns a single scalar per arm, $\E_{W,X}[g_\phi(a,W,X)]$, whose difference is an ATE. Recovering $Q(x)$ of \eqref{eq:proximal_g_formula} instead means marginalising $W$ \emph{at each query point}; we consider two possibilities.

\begin{enumerate}
  \item \textbf{We assume $W \indep X$} (``ind. W''). The direct transposition of the reference code: average the fitted bridge over the empirical marginal of $W$ in the context, $\hat Q(x) = \frac{1}{n}\sum_i g_\phi(x, W_i, 1) - g_\phi(x, W_i, 0)$. This is a valid strategy only if one assumes that $W \indep X$.
  
  \item \textbf{We do not assume $W \indep X$} (``cond. W''). We approximate $\E[g_\phi(x,W,a) \mid X = x]$ by averaging over the marginal of $W$ implied by the context, reweighted by a Gaussian kernel on $X$. This computes the proximal g-formula, without independence assumption, but at a cost of an extra smoothing bandwidth.
\end{enumerate}

\paragraph{Implementation corrections \& computational notes.} We do the following changes in the code compared to original code. First, in the MLP used to model the bridge function, the original implementation applied a ReLU transformation on the last layer. This made the prediction of negative outcomes impossible, and was removed in our implementation. Second, for computational optimization, we avoid recomputing the kernel matrix at each iteration mini-batch, but compute the complete matrix once and sub-sample from it.

\paragraph{Kernel bandwidth.} The original implementation uses a fixed bandwidth of $l = 1.0$ in the \Cref{eq:nmmr_resid}. This has obvious drawbacks as the average distance between data points grow with the dimensionality of these points. In non-trivial dimensions, the kernel becomes the identity matrix, and all information from the proxy $Z$ is lost: the bridge function becomes a direct estimator of $Y$ in the V-statistic case, and very unstable in the U-statistic one. Therefore, we implement an alternative bandwidth which is defined by the median heuristic and does not suffer from this degeneration in higher dimension. Both alternatives are evaluated.

\paragraph{Hyperparameter selection.} We follow the tuning from \citet{kompaDeepLearningMethods2022} and choose hyperparameters of the MLP based on the V- and U-statistics on a held-out validation split. The model is then refitted on the whole context using the final parameters. The parameter space searched is described in \Cref{tab:methods_nmmr}. Note that we do not evaluate every configuration, but sample without repetitions $10$ of them. 

\begin{table}
  \centering
  \caption{NMMR settings. Values marked \emph{searched} are selected per episode
           by the procedure described below; the remainder are fixed across all
           episodes. The U- and V-statistic variants are reported as separate
           methods, so the loss is fixed within a run rather than searched.}
  \label{tab:methods_nmmr}
  \begin{tabular}{lll}
    \toprule
    Component & Setting & Value \\
    \midrule
    Bridge class    & MLP, ReLU, linear output        & \\
                    & Hidden width                    & \emph{searched}: $\{16, 32, 64\}$ \\
                    & Depth (hidden layers)           & \emph{searched}: $\{2, 4, 6\}$ \\
    Moment kernel   & RBF over $(A, Z/\hat\sigma_Z, X/\hat\sigma_X)$ & \\
                    & Length scale $\ell$             & fixed at $1.0$ or median heuristic \\
    Loss            & U- or V-statistic \eqref{eq:nmmr_uv} & one method each \\
    Optimisation    & Adam, learning rate             & \emph{searched}: $\{10^{-3}, 5\cdot10^{-3}\}$ \\
                    & Weight decay ($L_2$)            & $10^{-4}$ \\
                    & Epochs                          & \emph{sweep}: $\{25, 50, 75, 100\}$ \\
                    & Minibatch size                  & $256$ \\
    Prediction      & $W$ strategy                    & ind. W / cond. W \\
                    & $X$ smoothing length scale (\texttt{Wcond}) & median heuristic \\
                    & Neighbours $k$ (\texttt{Wcond})  & $512$ \\
    Tuning          & Validation fraction             & $0.2$ \\
                    & Search budget                   & $10$ of $18$ configurations \\
    \bottomrule
  \end{tabular}
\end{table}

\subsubsection{Proxy maximum moment restriction}
\label{app:results_pmmr}

We reimplemented the proxy maximum moment restriction (PMMR) method of \citet{mastouriProximalCausalLearning2021}, instead of reusing the authors' experimental code\footnote{Original implementation available at \url{https://github.com/yuchen-zhu/kernel_proxies}.}. This allowed us to adjust the code for our usecases (prediction of CATE, with binary treatments). We preserved the PMMR's RKHS bridge class and objective, but adapted treatment kernel to binary setting, and changed the prediction heads to the same two cases as described for NMMR (\Cref{app:results_nmmr}).

\paragraph{The method.} PMMR fits the same outcome bridge $h_0$ as NMMR, but takes it in an RKHS rather than in a neural network class, which makes the minimiser available in closed form. We use column-wise Gaussian kernels for $W$, $Z$ and $X$, and a binary indicator kernel for $A$. With their products kernels over $(A,X,W)$ and $(A,Z,X)$, respectively, write the two matrices:
\begin{equation}
  L = K_A \odot K_W \odot K_X,
  \qquad
  \Phi = K_A \odot K_Z \odot K_X,
  \label{eq:pmmr_grams}
\end{equation}
so that $h(\cdot) = \sum_i \alpha_i k_L(\cdot, (A_i, W_i, X_i))$ and the V-statistic \eqref{eq:nmmr_uv} of the residuals $Y - L\alpha$ under $\Phi$ is minimised by
\begin{equation}
  \alpha = \big(L \Phi L + \lambda n^2 L + \epsilon n^2 I\big)^{-1} L \Phi Y .
  \label{eq:pmmr_alpha}
\end{equation}
Here $\lambda$ penalises $\norm{h}_{\mathcal H_L}$ and is the regulariser of the original method; $\epsilon$ is added for numerical stability.

\paragraph{Deviations at prediction time.} As for NMMR, the reference code returns only $\E[Y \mid \mathrm{do}(A=a)]$, which here is $\hat y(a) = \alpha^\T [\,k_A(A, a) \odot \overline{wx}\,]$ with $\overline{wx}_i = \frac 1n \sum_j k_W(W_i, W_j)\, k_X(X_i, X_j)$. We obtain $\hat Q(x)$ through the same two strategies as NMMR (``ind. W'', ``cond. W''), with the same caveats. Unlike NMMR, \texttt{Wcond} does not need extra hyperparameters: the weights used to approximate $W\mid X=x$ are determined by the $k_X(X_j,x)$ from the model.

\paragraph{Hyperparameter selection.} We select per episode on the V-statistic of a held-out fraction of the context, with the scoring kernel fixed at the median heuristic of the held-out rows, so that all candidates are compared on one metric rather than each on its own. The parameter space is described in \Cref{tab:methods_pmmr}. It splits in two: the bandwidth multipliers enter the matrices $L$ and $\Phi$, while $\lambda$ enters only the matrix assembled in \eqref{eq:pmmr_alpha}. We therefore sample the former and enumerate the latter exhaustively for each draw, re-solving \eqref{eq:pmmr_alpha} at every $\lambda$ off a single pair of $L$, $\Phi$.

\paragraph{Computational notes.} For binary $A$ the kernel $k_A$ is an indicator, so \eqref{eq:pmmr_alpha} is block diagonal by arm and we assemble and solve it one arm at a time; this is exact and reduces computational and memory costs.

\begin{table}[h]
  \centering
  \caption{PMMR settings. Bandwidth multipliers are \emph{sampled}; the
           regularisers are \emph{swept}, i.e.\ every combination is evaluated
           for each bandwidth draw. Both are selected per episode on the
           held-out V-statistic.}
  \label{tab:methods_pmmr}
  \begin{tabular}{lll}
    \toprule
    Component & Setting & Value \\
    \midrule
    Kernels         & Bandwidths baseline                           & per-column median heuristic \\
    & Kernel bandwidth multiplier (W,X) & \emph{searched}: $\{0.8,1,1.2\}$\\
    & Kernel bandwidth multiplier (Z) & \emph{searched}: $\{0.8,1,1.2\}$\\
    Regularisation  & $\lambda$ (RKHS norm)               & \emph{swept}: $\{10^{-5}, 10^{-4},10^{-3}, 10^{-2}, 10^{-1}\}$ \\
                    & Cholesky jitter                       & $10^{-8}$\\
    Prediction      & $W$ strategy                          & ``ind. W'', ``cond. W'' \\
    Tuning          & Validation fraction                   & $0.2$ \\
                    & Search budget                         & all $9$ bandwidth multipliers $\times$ $5$ regularisers \\
    \bottomrule
  \end{tabular}
\end{table}

\subsubsection{Kernel proxy variable}
\label{app:results_kpv}

We reimplemented the kernel proxy variable (KPV) method of \citet{mastouriProximalCausalLearning2021}, presented in the same paper as PMMR and sharing its released code. Again, the reimplementation allowed us to adjust the code for our usage: binary kernel and CATE predictions.

\paragraph{The method.} KPV fits the same outcome bridge $h_0$ as PMMR and in the same RKHS, but reaches it through two ridge regressions rather than one moment restriction. The theoretical formulation uses two samples, $D_1$ of size $m_1$ and $D_2$ of size $m_2$. Following the released implementation, we reuse the complete context in both stages, so $D_1=D_2$ and $m_1=m_2=n_{\mathrm{context}}$. The first stage regresses the feature map of $W$ on $(A, X, Z)$, giving the conditional mean embedding $\hat\mu_{W\mid a,x,z}$; writing $\mathcal{K} = K_A \odot K_X \odot K_Z$ for the product kernel over the conditioning set, its solution is
\begin{equation}
  \Gamma = \big(\mathcal{K}_{D_1 D_1} + \lambda_1 m_1 I\big)^{-1} \mathcal{K}_{D_1 D_2}.
  \label{eq:kpv_gamma}
\end{equation}
The second stage regresses $Y$ on that embedding, which yields the bridge coefficients in closed form,
\begin{equation}
  \Sigma = \big(\Gamma^\T K_W \Gamma\big) \odot K_A \odot K_X,
  \qquad
  \alpha_{ij} = \Gamma_{ij} \big[(\Sigma + \lambda_2 m_2 I)^{-1} Y_{D_2}\big]_j ,
  \label{eq:kpv_alpha}
\end{equation}
with $K_W$ taken on $D_1$ and $K_A, K_X$ on $D_2$. As done in PMMR, we use columnwise Gaussian kernels for $K_X, K_W$ and $K_Z$, and a binary indicator one for $K_A$. Unlike PMMR, the conditional distribution of $W$ given the conditioning set is estimated explicitly, in the first stage, rather than left implicit in a moment condition.

\paragraph{Hyperparameter selection.} We use the authors' closed-form leave-one-out criterion for both regularisers, run inside each fit. On the other hand, we modify the hyperparameter tuning of the kernel bandwidth. As suggested by the authors, we initialize the search by the median heuristic. We then search around this initialization by minimizing the stage-2 leave-one-out criterion (instead of the stage-1 one, used in their implementation). We prefer this as it compares bandwidths using prediction error against the same outcome $Y$, whereas stage-1 LOOCV changes both the input kernel and the $W$-kernel defining the loss, making scores across bandwidths less directly comparable. \Cref{tab:methods_kpv} summarises the tuning parameters.

\paragraph{Deviations at prediction time.} As for NMMR and PMMR, the reference returns only $\E[Y \mid \mathrm{do}(A=a)]$, here $\hat\beta(a) = \frac 1n \sum_{i,j,k} \hat\alpha_{ij}\, k_A(\tilde a_j, a)\, k_X(\tilde x_j, x_k)\, k_W(w_i, w_k)$. We obtain $\hat Q(x)$ through the same two strategies (``ind. W'', ``cond. W''), with the same caveats. As in PMMR, the marginalisation is over the \emph{joint} $(X, W)$, whereas the released code averages the two grams separately, which agrees only if $X \indep W$ in sample. The reference estimates the population effect by averaging the fitted bridge over a separate test sample; we instead use the context rows, which also serve as the fitting sample.

\paragraph{Implementation corrections \& computational notes.} In the original implementation, the median bandwidth is recomputed at every kernel call, and prediction uses the default scale rather than using the training one. Therefore, training and prediction do not use the same fitted kernel. We instead estimate each bandwidth once from the context and reuse it consistently. The released driver also fits a separate input scaler to the test data; we reuse the transformation from the context. Finally, we leverage the binary aspect of the treatment kernel to solve the closed-form formulas arm by arm, reducing computational and memory costs without changing the estimator.

\begin{table}[h]
  \centering
  \caption{KPV settings. Both stages use all context rows. The regularizers
           $\lambda_1$ and $\lambda_2$ are selected within each fit by their
           closed-form leave-one-out criteria. The shared kernel-scale
           multiplier is selected per episode using the stage-2 LOOCV loss.}
  \label{tab:methods_kpv}
  \begin{tabular}{lll}
    \toprule
    Component & Setting & Value \\
    \midrule
    Data split
      & $m_1$ (stage-1 rows)
      & $n_{\mathrm{context}}$ \\
      & $m_2$ (stage-2 rows)
      & $n_{\mathrm{context}}$ \\

    Kernels
      & Kernel family
      & Column-wise Gaussian RBF for $X,W,Z$; binary for $A$ \\
      & Base bandwidths
      & Per-column median heuristic, estimated on the context \\
      & Shared multiplier $c$ on $\sigma_d^2$
      & Shared across $X,W,Z$; \emph{searched:} $\{0.8,1.0,1.2\}$ \\

    Regularization
      & $\lambda_1,\lambda_2$
      & \emph{searched} by LOOCV; 10-point log grid in $[10^{-5},10]$ \\

    Prediction
      & $W$ strategy
      & ``ind.\ $W$'', ``cond.\ $W$'' \\

    Tuning
      & Selection score
      & Stage-2 leave-one-out loss \\
      & Bandwidth candidates
      & 3 per episode; no separate validation split \\
    \bottomrule
  \end{tabular}
\end{table}

\FloatBarrier
\subsection{Semi-synthetic benchmarks from real covariates}
\label{app:results_semisynt}

We construct proximal benchmarks from twelve numerical datasets in the INRIA-SODA tabular benchmark collection \citep{grinsztajn2022tabular}.\footnote{\url{https://huggingface.co/datasets/inria-soda/tabular-benchmark}} Real source rows supply the joint distribution of observed and hidden variables. We generate treatment, proxies, and outcomes with known causal mechanisms. Crossing linear/nonlinear mechanisms, weak/strong confounding, and weak/strong proxy information gives $12\times2\times2\times2=96$ configurations. Three seeded realizations per configuration give $288$ realizations. Figure~\ref{fig:benchmark_construction}(a) shows how the real covariate distribution enters the construction.

\begin{figure}[t]
\centering
\includegraphics[width=\linewidth]{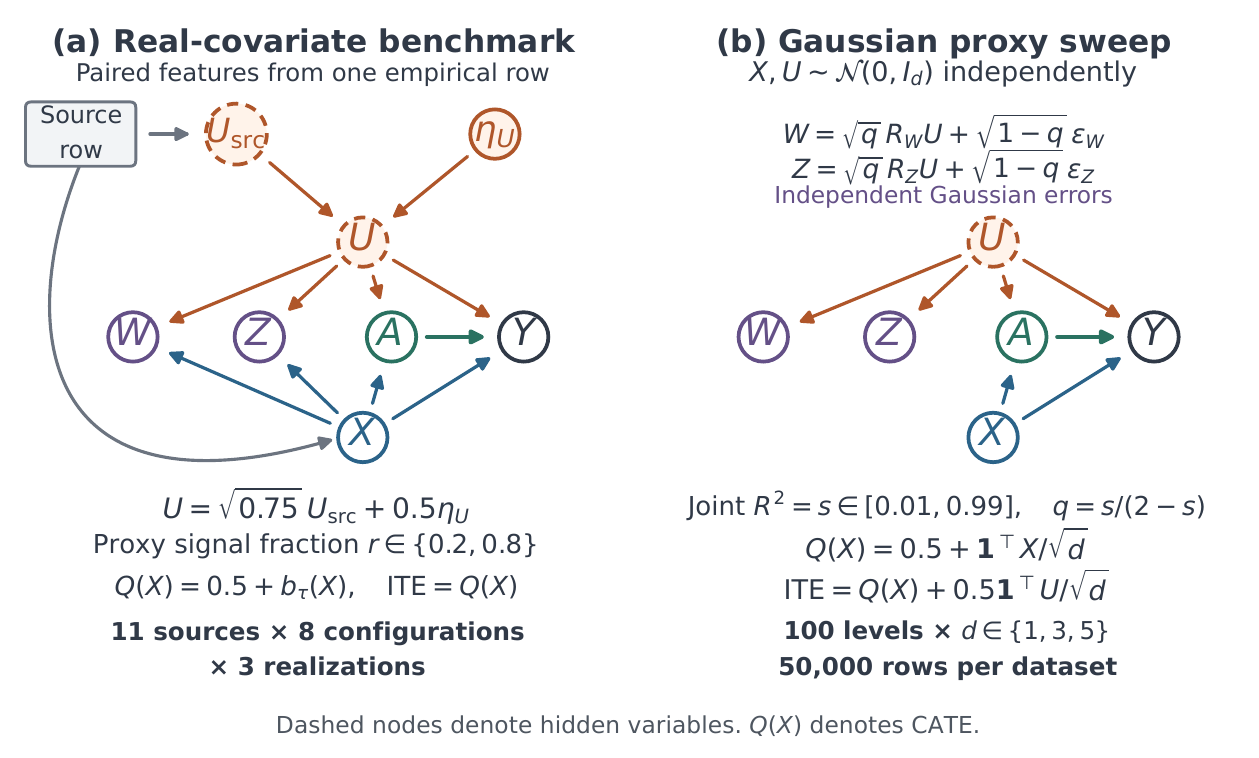}
\caption{Construction of the generated proximal benchmarks. (a) A shared real source row supplies dependent observed covariates and hidden source features. Independent Gaussian innovation augments the hidden features. Both $X$ and $U$ enter the proxy measurements and the treatment and outcome mechanisms. (b) Independent Gaussian $X$ and $U$ define the controlled synthetic study. The joint proxy information parameter changes the measurement channels while preserving the causal variables within each dimension. The diagrams display the generated causal structure, with independent disturbances suppressed where indicated.}
\label{fig:benchmark_construction}
\end{figure}

\paragraph{Sources and calibration.}

We select twelve datasets from the INRIA-SODA tabular benchmark collection with at least $20$ features and $10\,000$ rows. For each source, we reserve a fixed set of rows for calibration and use the remaining rows to construct evaluation episodes. Calibration rows determine missing-value handling, clipping, and feature standardisation; the resulting transformations remain fixed across episodes. We exclude features that cannot be reliably calibrated.

We assign the retained source columns, which may include the dataset’s original prediction target, to observed covariates $X$ and hidden source variables $U_{\mathrm{src}}$ using a deterministic ordering. The two groups contain the same number of features; we use $d=5$ for our experiments. Each realization samples context and query rows from the evaluation pool without replacement. \Cref{tab:real_sources} reports the original sizes of the selected datasets.

\begin{table}[t]
\centering
\small
\caption{Original sizes of the twelve selected datasets. Feature counts exclude the original prediction target, although that column remains eligible for selection as a source variable in our construction.}
\begin{tabular}{lrr}
\toprule
Source & Rows & Features \\
\midrule
\texttt{pol}                            & $10\,082$  & 26  \\
\texttt{MiniBooNE}                      & $72\,998$  & 50  \\
\texttt{default-of-credit-card-clients} & $13\,272$  & 20  \\
\texttt{Higgs}                          & $940\,160$ & 24  \\
\texttt{jannis}                         & $57\,580$  & 54  \\
\texttt{heloc}                          & $10\,000$  & 22  \\
\texttt{covertype}                      & $423\,680$ & 54  \\
\texttt{albert}                         & $58\,252$  & 31  \\
\texttt{road-safety}                    & $111\,762$ & 32  \\
\texttt{Ailerons}                       & $13\,750$  & 33  \\
\texttt{superconduct}                   & $21\,263$  & 79  \\
\texttt{Allstate\_Claims\_Severity}     & $188\,318$ & 124 \\
\bottomrule
\end{tabular}
\label{tab:real_sources}
\end{table}

\paragraph{Latent and observed mechanisms.} For each row, we draw $\eta_U\sim\mathcal N(0,I_{d_U})$ independently and set
$$
U=\sqrt{0.75}\,U_{\mathrm{src}}+0.5\eta_U.
$$
Because $X$ and $U_{\mathrm{src}}$ come from the same source row, $U$ retains some of their dependence. The independent noise term adds variation in $U$ even when $X$ is fixed. Let $g$ be identity in the linear family and componentwise $\operatorname{asinh}$ in the nonlinear family. A separate calibration innovation stream determines
\[
m=\operatorname{mean}_{\mathrm{cal}}g(U),\qquad
C=\operatorname{Cov}_{\mathrm{cal}}g(U),\qquad
V=(C+0.05I)^{-1/2}\{g(U)-m\}.
\]
The covariance uses the population denominator and the inverse square root is symmetric. We draw $\widetilde c\sim\mathcal N(0,I_{d_U})$ and rescale it to $c$ so that $L=c^\top V$ has unit calibration standard deviation. Treatment and baseline outcome share this latent score.

The linear observed basis contains the coordinates of $X$. The nonlinear basis applies one uniformly selected transformation from $t$, $\operatorname{asinh}(t)$, $\sin(t)$, and $t^2$ to each coordinate, then adds products for up to eight distinct coordinate pairs selected uniformly without replacement. After standardising these basis columns on calibration rows, let $K$ denote their number. For each role $r\in\{a,y,\tau,w,z\}$, we draw independent coefficients from $\mathcal N(0,1/K)$ and standardise the resulting linear combinations on calibration rows to obtain $b_r(X)$. The functions $b_a,b_y,b_\tau$ are scalar and $b_w,b_z$ each have $d_U$ coordinates. Standard deviations used to normalise the score, basis, and functions have a floor of $10^{-8}$. We set
\[
f_A(X)=a_0+0.5b_a(X),\quad f_Y(X)=b_y(X),\quad
f_W(X)=0.5b_w(X),\quad f_Z(X)=0.5b_z(X),
\]
where $a_0\sim\operatorname{Unif}[-0.25,0.25]$, and define $Q(X)=0.5+b_\tau(X)$. The CATE function therefore has calibration mean $0.5$ and standard deviation one.

\paragraph{Proxy channels.} We obtain orthogonal $B_W,B_Z$ from independent Gaussian matrices using QR decomposition. We choose signs for a nonnegative diagonal in the triangular factor. For $J\in\{W,Z\}$, define
$$
\sigma_{J,j}=\max\{\operatorname{sd}_{\mathrm{cal}}[(B_JV)_j],10^{-8}\},
\qquad
D_J(r)=\sqrt{\frac{1-r}{r}}\operatorname{diag}(\sigma_{J,1},\ldots,\sigma_{J,d_U}).
$$
With independent standard-normal errors, the proxy measurements are
$$
W=\phi\{B_WV+f_W(X)+D_W(r)\epsilon_W\},\qquad
Z=\phi\{B_ZV+f_Z(X)+D_Z(r)\epsilon_Z\}.
$$
The map $\phi$ is identity in the linear family and componentwise $\operatorname{asinh}$ in the nonlinear family. Weak and strong proxy settings use $r=0.2$ and $r=0.8$. This parameter is the calibration latent-signal fraction in each measurement coordinate before the observed offset and observation map. Joint predictive information about $U$ also depends on the source distribution and both proxy channels.

\paragraph{Treatment, outcomes, and bridge.} Using independent treatment and outcome disturbances, we generate
\begin{align}
e(X,U)&=\operatorname{clip}\{\operatorname{sigmoid}(f_A(X)+\alpha L),0.02,0.98\},\\
A&\sim\operatorname{Bernoulli}\{e(X,U)\},\\
Y(0)&=f_Y(X)+\gamma L+\epsilon_Y,\qquad \epsilon_Y\sim\mathcal N(0,1),\\
Y(1)&=Y(0)+Q(X),\qquad Y=Y(0)+AQ(X).
\end{align}
Weak confounding uses $(\alpha,\gamma)=(0.35,0.5)$ and strong confounding uses $(1.25,1.5)$. The potential outcomes share their disturbance, giving the exact heterogeneous effect
\[
Y(1)-Y(0)=Q(X)=\E[Y(1)-Y(0)\mid X].
\]
An explicit outcome bridge is
\[
h(w,a,x)=f_Y(x)+\gamma c^\top B_W^{-1}
\{\phi^{-1}(w)-f_W(x)\}+aQ(x).
\]
Here $\phi^{-1}$ is identity or componentwise $\sinh$. Independent, zero-mean measurement error gives $\E[h(W,a,X)\mid X,U]=\E[Y(a)\mid X,U]$. The graph in Figure~\ref{fig:benchmark_construction}(a) also gives $W\perp(A,Z)\mid(X,U)$ and $Y\perp Z\mid(A,X,U)$, establishing
\[
\E[Y\mid A,Z,X]=\E[h(W,A,X)\mid A,Z,X].
\]
The nonlinear family consequently retains an explicit bridge through invertible measurement transformations.

\paragraph{Paired configurations.}
For each source and dimension, feature assignments are fixed across all configurations and realizations. Within a realization, all configurations use the same sampled rows; configurations in the same mechanism family also share fitted functions and random disturbances. Different realizations independently sample rows from the same evaluation pool and redraw the mechanisms and disturbances. Their sampled rows may overlap.

\FloatBarrier
\subsection{Results on semi-synthetic datasets}
\label{app:results_semisynthetic}

\begin{figure}
    \centering
    \includegraphics[width=\textwidth]{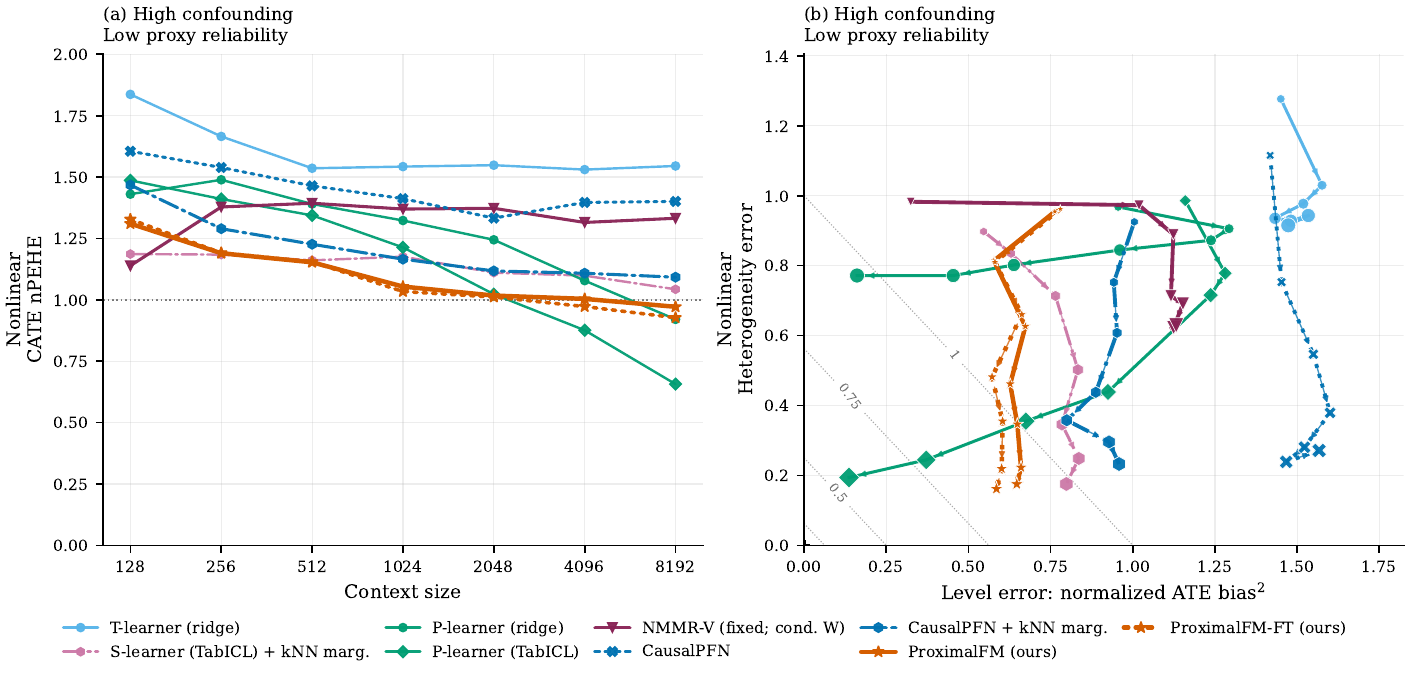}
    \caption{Learning curves and error decomposition for the nonlinear, high-confounding, low-proxy-reliability setting (36 episodes: 12 sources with three realizations each). \textbf{(a)} Median CATE nPEHE versus context size; the dotted line marks the constant-ATE benchmark ($\mathrm{nPEHE}=1$). \texttt{ProximalFM-FT} continues training the main checkpoint on the same prior and objective with contexts up to $5\,000$ rows, updating only part of the in-context predictor (\Cref{app:model_posttraining}). \textbf{(b)} For query-row errors $e_i=\hat\tau_i-\tau_i$, their mean $\bar e$, and true CATE variance $s_\tau^2=\operatorname{Var}_i(\tau_i)$, the episode-level decomposition is $\mathrm{nPEHE}^2=\bar e^2/s_\tau^2+\operatorname{Var}_i(e_i)/s_\tau^2$. Panel (b) plots the first term horizontally (squared bias in the query-average treatment effect) and the second vertically (heterogeneity error after removing that bias). Dotted contours show constant $\sqrt{x+y}$; plotted coordinates are separate component medians, so they need not reproduce panel (a)'s median nPEHE. Arrows and larger markers indicate increasing context size ($128$--$8\,192$). \texttt{ProximalFM-FT} moves mainly leftward: post-training improves the average effect estimate, with little improvement in CATE shape.}    \label{fig:semisynt_learning_curves_selection}
\end{figure}

\begin{figure}
    \centering
    \includegraphics[width=\textwidth]{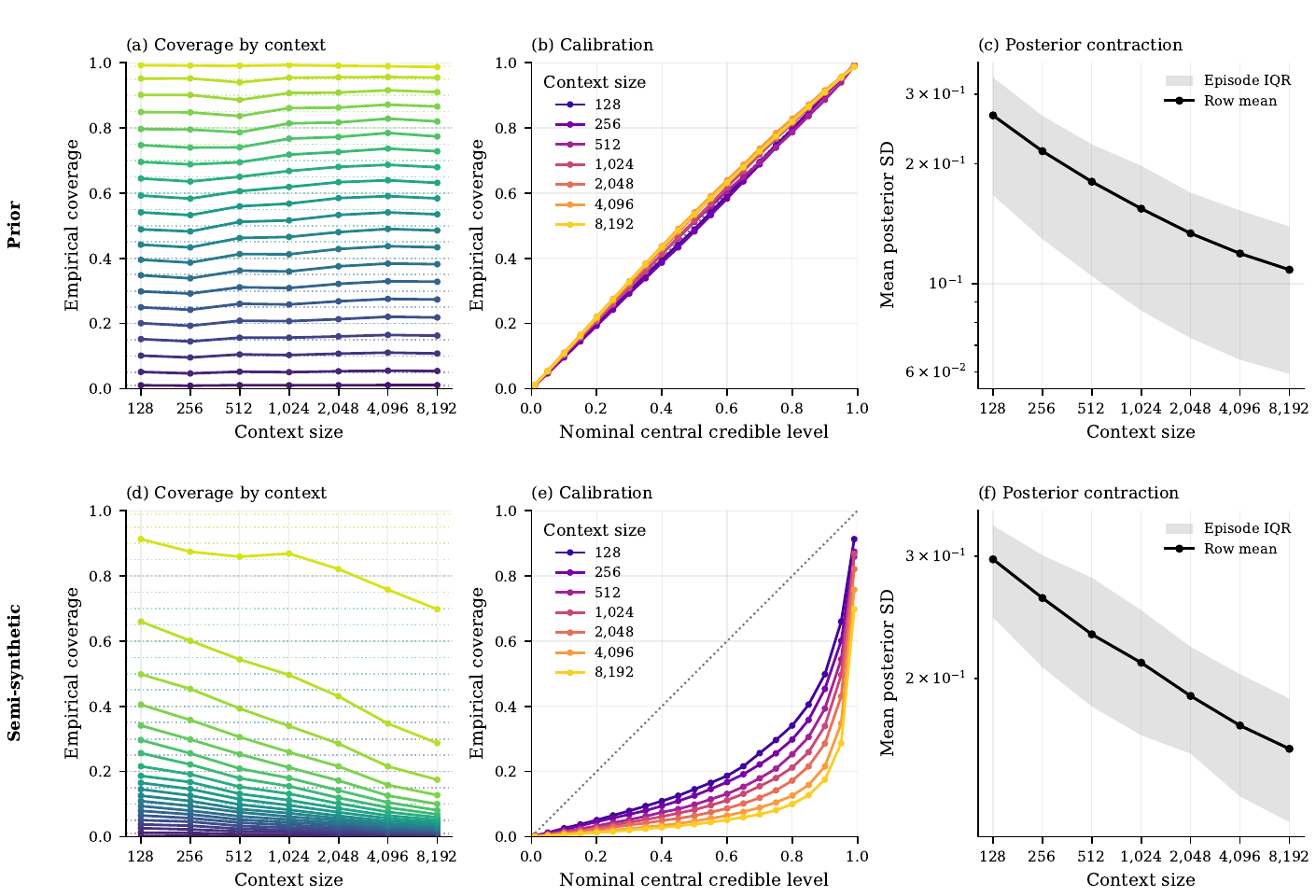}
\caption{Posterior calibration and contraction. The top row evaluates 400 datasets drawn from the pre-training prior; the bottom row evaluates 120 semi-synthetic datasets with nonlinear mechanisms, high confounding, and low proxy reliability. In every episode, the same $256$ query points are evaluated as the context grows through nested prefixes from $128$ to $8\,192$ observations. \textbf{(a,d) Coverage by context:} each solid curve gives empirical coverage of a central CATE credible interval at one nominal level (1\%, 5\%, 10\%, \ldots, 95\%, 99\%); the faint dotted line of the same colour marks that level's nominal coverage. \textbf{(b,e) Calibration:} each curve fixes a context size, coloured from purple ($128$) to yellow ($8\,192$), and plots empirical against nominal coverage. The dotted diagonal denotes perfect calibration; curves below it indicate undercoverage. \textbf{(c,f) Posterior contraction:} the black curve is the mean predicted CATE posterior standard deviation across query rows. The grey band is the interquartile range of episode-level mean standard deviations, showing variation between datasets. These panels use logarithmic vertical axes. Coverage stays close to nominal on the prior, whereas the semi-synthetic intervals increasingly undercover as context grows, even though their predicted uncertainty contracts.}    \label{fig:posterior_calibration}
\end{figure}

\begin{table*}[p]
    \centering
    \scriptsize
    \setlength{\tabcolsep}{3.2pt}
    \renewcommand{\arraystretch}{0.92}
    \caption{Exhaustive semisynthetic CATE nPEHE results at $n_{\mathrm{ctx}}=128$. Each entry first averages the three replicates of each source dataset, then averages equally over the 12 sources; lower is better. The lowest mean in each column is bold. Other underlined methods are not significantly worse than that empirical best under one-sided exact paired sign-flip tests across source-level means, with Holm correction within the column ($\alpha=0.05$).}
    \label{tab:semisynt_exhaustive_npehe_128}
    \resizebox{\textwidth}{!}{%
    \begin{tabular}{@{}lrrrrrrrr@{}}
        \toprule
        & \multicolumn{4}{c}{Linear} & \multicolumn{4}{c}{Nonlinear} \\
        \cmidrule(lr){2-5} \cmidrule(l){6-9}
        & \multicolumn{2}{c}{Low confounding} & \multicolumn{2}{c}{High confounding} & \multicolumn{2}{c}{Low confounding} & \multicolumn{2}{c}{High confounding} \\
        \cmidrule(lr){2-3} \cmidrule(lr){4-5} \cmidrule(lr){6-7} \cmidrule(l){8-9}
        Method & Low proxy & High proxy & Low proxy & High proxy & Low proxy & High proxy & Low proxy & High proxy \\
        \midrule
        S-learner (ridge) & 1.027 & 1.027 & 1.458 & 1.458 & 1.059 & 1.059 & 1.459 & 1.459 \\
        S-learner (RF) & 0.868 & 0.868 & 1.229 & 1.229 & 0.986 & \underline{0.986} & 1.542 & 1.542 \\
        S-learner (TabICL) & 0.782 & 0.782 & 1.336 & 1.336 & \textbf{0.937} & \textbf{0.937} & 1.457 & 1.457 \\
        S-learner (TabICL) + kNN marg. & 0.885 & 0.895 & \underline{1.070} & 0.921 & \underline{0.973} & \underline{0.982} & 1.259 & \underline{0.986} \\
        S-learner (TabICL) + kernel marg. & 0.888 & 0.888 & \underline{1.064} & 0.914 & \underline{0.966} & \underline{0.966} & 1.245 & \textbf{0.966} \\
        \addlinespace[2pt]
        T-learner (ridge) & \underline{0.603} & \underline{0.603} & 1.466 & 1.466 & 1.122 & 1.122 & 1.837 & 1.837 \\
        T-learner (RF) & 0.846 & 0.846 & 1.580 & 1.580 & 1.157 & 1.157 & 1.751 & 1.751 \\
        T-learner (TabICL) & 0.865 & 0.865 & 1.637 & 1.637 & 1.232 & 1.232 & 1.810 & 1.810 \\
        T-learner (ridge) + kNN marg. & \underline{0.639} & \underline{0.603} & 1.255 & \underline{0.694} & 1.063 & \underline{1.103} & 1.631 & 1.398 \\
        T-learner (ridge) + kernel marg. & \underline{0.665} & \underline{0.610} & 1.276 & \underline{0.704} & 1.064 & \underline{1.074} & 1.634 & 1.395 \\
        \addlinespace[2pt]
        X-learner (RF) & \underline{0.718} & 0.718 & 1.436 & 1.436 & \underline{1.025} & \underline{1.025} & 1.605 & 1.605 \\
        X-learner (TabICL) & 0.737 & 0.737 & 1.509 & 1.509 & 1.065 & 1.065 & 1.669 & 1.669 \\
        \addlinespace[2pt]
        CausalFM & 0.738 & 0.738 & 1.306 & 1.306 & \underline{1.048} & \underline{1.048} & 1.566 & 1.566 \\
        Do-PFN & 0.806 & 0.806 & 1.294 & 1.294 & \underline{1.606} & \underline{1.606} & 1.604 & 1.604 \\
        CausalPFN & \underline{0.641} & 0.641 & 1.408 & 1.408 & \underline{1.055} & \underline{1.055} & 1.687 & 1.687 \\
        CausalPFN + kNN marg. & \underline{0.577} & \textbf{0.565} & 1.112 & \textbf{0.654} & \underline{0.971} & \underline{0.945} & 1.460 & 1.048 \\
        CausalPFN + kernel marg. & \underline{0.597} & \underline{0.574} & 1.133 & \underline{0.668} & \underline{0.979} & \underline{0.956} & 1.466 & \underline{1.072} \\
        \addlinespace[2pt]
        P-learner (RF) & 0.731 & 0.708 & 1.267 & 0.916 & \underline{0.993} & \underline{0.980} & 1.514 & 1.222 \\
        P-learner (ridge) & \textbf{0.571} & \underline{0.581} & 1.181 & 0.771 & \underline{1.001} & \underline{0.995} & 1.511 & 1.221 \\
        P-learner (TabICL) & \underline{0.652} & 0.664 & 1.289 & 0.891 & \underline{0.984} & \underline{0.972} & 1.535 & 1.210 \\
        \addlinespace[2pt]
        KPV (cond. W) & 1.325 & 1.021 & 1.905 & 1.192 & 1.300 & \underline{1.014} & 2.054 & 1.204 \\
        KPV (independent W) & 0.955 & 0.953 & 1.395 & 0.970 & \underline{1.040} & \underline{1.003} & 1.551 & \underline{1.028} \\
        PMMR (cond. W) & 0.957 & 0.922 & \underline{1.011} & 0.953 & \underline{1.006} & \underline{1.002} & 1.105 & \underline{1.072} \\
        PMMR (independent W) & 0.993 & 0.989 & \textbf{0.982} & 0.953 & \underline{1.017} & \underline{1.021} & \textbf{1.039} & \underline{1.006} \\
        \addlinespace[2pt]
        NMMR-V (fixed; cond. W) & 0.954 & 0.910 & 1.090 & 0.917 & 1.012 & \underline{1.013} & 1.290 & 1.006 \\
        NMMR-V (fixed; ind. W) & 0.957 & 0.915 & 1.094 & 0.926 & 1.016 & \underline{1.017} & 1.292 & 1.006 \\
        NMMR-U (fixed; cond. W) & $4.98\!\times\!10^{3}$ & 430.450 & $1.13\!\times\!10^{5}$ & 155.343 & $6.16\!\times\!10^{3}$ & 55.108 & $6.11\!\times\!10^{5}$ & 937.488 \\
        NMMR-U (fixed; ind. W) & $4.90\!\times\!10^{3}$ & 408.021 & $1.14\!\times\!10^{5}$ & 148.985 & $5.94\!\times\!10^{3}$ & 55.162 & $5.97\!\times\!10^{5}$ & $1.36\!\times\!10^{3}$ \\
        NMMR-V (median; cond. W) & 1.060 & 1.051 & 1.045 & 0.983 & 1.077 & 1.072 & 1.115 & 1.028 \\
        NMMR-V (median; ind. W) & 1.061 & 1.051 & 1.044 & 0.983 & 1.077 & 1.073 & 1.115 & 1.029 \\
        NMMR-U (median; cond. W) & 6.500 & 1.479 & 5.725 & 1.768 & 51.236 & 1.439 & 6.794 & 1.341 \\
        NMMR-U (median; ind. W) & 6.940 & 1.460 & 5.725 & 1.763 & 49.715 & 1.420 & 7.024 & 1.324 \\
        \addlinespace[2pt]
        ProximalFM (ours) & \underline{0.621} & 0.608 & \underline{1.058} & \underline{0.733} & \underline{0.953} & \underline{0.947} & 1.329 & \underline{0.968} \\
        \bottomrule
    \end{tabular}%
    }
\end{table*}

\begin{table*}[p]
    \centering
    \scriptsize
    \setlength{\tabcolsep}{3.2pt}
    \renewcommand{\arraystretch}{0.92}
    \caption{Exhaustive semisynthetic CATE nPEHE results at $n_{\mathrm{ctx}}=512$. Each entry first averages the three replicates of each source dataset, then averages equally over the 12 sources; lower is better. The lowest mean in each column is bold. Other underlined methods are not significantly worse than that empirical best under one-sided exact paired sign-flip tests across source-level means, with Holm correction within the column ($\alpha=0.05$).}
    \label{tab:semisynt_exhaustive_npehe_512}
    \resizebox{\textwidth}{!}{%
    \begin{tabular}{@{}lrrrrrrrr@{}}
        \toprule
        & \multicolumn{4}{c}{Linear} & \multicolumn{4}{c}{Nonlinear} \\
        \cmidrule(lr){2-5} \cmidrule(l){6-9}
        & \multicolumn{2}{c}{Low confounding} & \multicolumn{2}{c}{High confounding} & \multicolumn{2}{c}{Low confounding} & \multicolumn{2}{c}{High confounding} \\
        \cmidrule(lr){2-3} \cmidrule(lr){4-5} \cmidrule(lr){6-7} \cmidrule(l){8-9}
        Method & Low proxy & High proxy & Low proxy & High proxy & Low proxy & High proxy & Low proxy & High proxy \\
        \midrule
        S-learner (ridge) & 1.013 & 1.013 & 1.564 & 1.564 & 1.023 & 1.023 & 1.580 & 1.580 \\
        S-learner (RF) & 0.669 & 0.669 & 1.423 & 1.423 & 0.849 & 0.849 & 1.546 & 1.546 \\
        S-learner (TabICL) & 0.499 & 0.499 & 1.294 & 1.294 & \underline{0.762} & \underline{0.762} & 1.432 & 1.432 \\
        S-learner (TabICL) + kNN marg. & 0.527 & 0.581 & 1.053 & 0.683 & 0.793 & 0.822 & \underline{1.206} & \underline{0.865} \\
        S-learner (TabICL) + kernel marg. & 0.526 & 0.574 & 1.053 & 0.678 & \underline{0.787} & 0.811 & \underline{1.196} & \underline{0.846} \\
        \addlinespace[2pt]
        T-learner (ridge) & 0.315 & 0.315 & 1.338 & 1.338 & 1.015 & 1.015 & 1.665 & 1.665 \\
        T-learner (RF) & 0.690 & 0.690 & 1.506 & 1.506 & 0.921 & 0.921 & 1.605 & 1.605 \\
        T-learner (TabICL) & 0.498 & 0.498 & 1.451 & 1.451 & \underline{0.785} & \underline{0.785} & 1.531 & 1.531 \\
        T-learner (ridge) + kNN marg. & \textbf{0.302} & \underline{0.282} & 1.041 & \textbf{0.408} & 0.978 & 0.926 & 1.441 & 1.008 \\
        T-learner (ridge) + kernel marg. & 0.307 & \textbf{0.279} & 1.051 & \underline{0.413} & 0.975 & 0.897 & 1.437 & 0.981 \\
        \addlinespace[2pt]
        X-learner (RF) & 0.530 & 0.530 & 1.381 & 1.381 & 0.814 & 0.814 & 1.508 & 1.508 \\
        X-learner (TabICL) & 0.424 & 0.424 & 1.379 & 1.379 & \textbf{0.726} & \underline{0.726} & 1.462 & 1.462 \\
        \addlinespace[2pt]
        CausalFM & 0.638 & 0.638 & 1.288 & 1.288 & 0.945 & 0.945 & 1.467 & 1.467 \\
        Do-PFN & 0.609 & 0.609 & 1.362 & 1.362 & 1.043 & 1.043 & 1.516 & 1.516 \\
        CausalPFN & 0.372 & 0.372 & 1.330 & 1.330 & \underline{0.730} & \underline{0.730} & 1.488 & 1.488 \\
        CausalPFN + kNN marg. & 0.356 & 0.310 & 1.035 & \underline{0.424} & \underline{0.760} & \underline{0.740} & \underline{1.264} & 0.826 \\
        CausalPFN + kernel marg. & 0.360 & 0.314 & 1.043 & \underline{0.436} & \underline{0.752} & \textbf{0.713} & \underline{1.251} & \textbf{0.783} \\
        \addlinespace[2pt]
        P-learner (RF) & 0.606 & 0.569 & 1.273 & 0.741 & 0.830 & 0.808 & 1.406 & 0.896 \\
        P-learner (ridge) & 0.410 & 0.379 & 1.164 & 0.567 & 0.922 & 0.910 & 1.443 & 0.991 \\
        P-learner (TabICL) & 0.498 & 0.465 & 1.244 & 0.661 & \underline{0.773} & \underline{0.747} & 1.368 & \underline{0.838} \\
        \addlinespace[2pt]
        KPV (cond. W) & 1.936 & 0.909 & 3.028 & 1.202 & 2.039 & 0.996 & 2.573 & 1.162 \\
        KPV (independent W) & 1.222 & 0.844 & 2.172 & 0.898 & 1.228 & 0.903 & 2.013 & 0.956 \\
        PMMR (cond. W) & 0.878 & 0.856 & \underline{1.034} & 0.910 & 0.914 & 0.898 & \underline{1.155} & 1.009 \\
        PMMR (independent W) & 0.911 & 0.924 & \underline{1.009} & 0.912 & 0.962 & 0.966 & \underline{1.121} & 0.993 \\
        \addlinespace[2pt]
        NMMR-V (fixed; cond. W) & 0.727 & 0.695 & 1.187 & 0.781 & 0.883 & 0.892 & 1.368 & 0.973 \\
        NMMR-V (fixed; ind. W) & 0.731 & 0.706 & 1.190 & 0.793 & 0.887 & 0.901 & 1.371 & 0.982 \\
        NMMR-U (fixed; cond. W) & 14.122 & 0.929 & 3.879 & 0.955 & 2.371 & 0.998 & 33.761 & 1.009 \\
        NMMR-U (fixed; ind. W) & 15.140 & 0.935 & 3.891 & 0.959 & 2.385 & 1.002 & 33.990 & 1.011 \\
        NMMR-V (median; cond. W) & 1.044 & 1.036 & \underline{1.015} & 0.977 & 1.040 & 1.045 & \textbf{1.109} & 0.997 \\
        NMMR-V (median; ind. W) & 1.044 & 1.037 & \underline{1.019} & 0.980 & 1.042 & 1.047 & \underline{1.111} & 1.001 \\
        NMMR-U (median; cond. W) & 1.079 & 1.081 & 1.100 & 1.023 & 1.072 & 1.067 & \underline{1.121} & 1.014 \\
        NMMR-U (median; ind. W) & 1.079 & 1.081 & 1.101 & 1.023 & 1.073 & 1.068 & \underline{1.121} & 1.014 \\
        \addlinespace[2pt]
        ProximalFM (ours) & 0.392 & 0.370 & \textbf{0.953} & 0.468 & \underline{0.780} & \underline{0.758} & \underline{1.151} & \underline{0.789} \\
        \bottomrule
    \end{tabular}%
    }
\end{table*}

\begin{table*}[p]
    \centering
    \scriptsize
    \setlength{\tabcolsep}{3.2pt}
    \renewcommand{\arraystretch}{0.92}
    \caption{Exhaustive semisynthetic CATE nPEHE results at $n_{\mathrm{ctx}}=1,024$. Each entry first averages the three replicates of each source dataset, then averages equally over the 12 sources; lower is better. The lowest mean in each column is bold. Other underlined methods are not significantly worse than that empirical best under one-sided exact paired sign-flip tests across source-level means, with Holm correction within the column ($\alpha=0.05$).}
    \label{tab:semisynt_exhaustive_npehe_1024}
    \resizebox{\textwidth}{!}{%
    \begin{tabular}{@{}lrrrrrrrr@{}}
        \toprule
        & \multicolumn{4}{c}{Linear} & \multicolumn{4}{c}{Nonlinear} \\
        \cmidrule(lr){2-5} \cmidrule(l){6-9}
        & \multicolumn{2}{c}{Low confounding} & \multicolumn{2}{c}{High confounding} & \multicolumn{2}{c}{Low confounding} & \multicolumn{2}{c}{High confounding} \\
        \cmidrule(lr){2-3} \cmidrule(lr){4-5} \cmidrule(lr){6-7} \cmidrule(l){8-9}
        Method & Low proxy & High proxy & Low proxy & High proxy & Low proxy & High proxy & Low proxy & High proxy \\
        \midrule
        S-learner (ridge) & 1.011 & 1.011 & 1.567 & 1.567 & 1.024 & 1.024 & 1.593 & 1.593 \\
        S-learner (RF) & 0.568 & 0.568 & 1.384 & 1.384 & 0.793 & 0.793 & 1.497 & 1.497 \\
        S-learner (TabICL) & 0.379 & 0.379 & 1.256 & 1.256 & \underline{0.653} & \underline{0.653} & 1.393 & 1.393 \\
        S-learner (TabICL) + kNN marg. & 0.393 & 0.426 & 0.994 & 0.552 & 0.694 & 0.732 & 1.169 & 0.783 \\
        S-learner (TabICL) + kernel marg. & 0.396 & 0.425 & 0.995 & 0.550 & 0.686 & 0.715 & 1.159 & 0.765 \\
        \addlinespace[2pt]
        T-learner (ridge) & 0.242 & 0.242 & 1.284 & 1.284 & 0.998 & 0.998 & 1.627 & 1.627 \\
        T-learner (RF) & 0.561 & 0.561 & 1.413 & 1.413 & 0.786 & 0.786 & 1.502 & 1.502 \\
        T-learner (TabICL) & 0.382 & 0.382 & 1.357 & 1.357 & \underline{0.670} & \underline{0.670} & 1.425 & 1.425 \\
        T-learner (ridge) + kNN marg. & \textbf{0.220} & \textbf{0.194} & 0.990 & \textbf{0.312} & 0.967 & 0.874 & 1.390 & 0.934 \\
        T-learner (ridge) + kernel marg. & 0.225 & \underline{0.197} & 0.992 & \underline{0.318} & 0.962 & 0.855 & 1.386 & 0.913 \\
        \addlinespace[2pt]
        X-learner (RF) & 0.453 & 0.453 & 1.326 & 1.326 & 0.734 & 0.734 & 1.459 & 1.459 \\
        X-learner (TabICL) & 0.348 & 0.348 & 1.303 & 1.303 & \textbf{0.609} & \underline{0.609} & 1.396 & 1.396 \\
        \addlinespace[2pt]
        CausalFM & 0.616 & 0.616 & 1.272 & 1.272 & 0.927 & 0.927 & 1.453 & 1.453 \\
        Do-PFN & 0.544 & 0.544 & 1.314 & 1.314 & 0.993 & 0.993 & 1.537 & 1.537 \\
        CausalPFN & 0.297 & 0.297 & 1.277 & 1.277 & \underline{0.629} & \underline{0.629} & 1.412 & 1.412 \\
        CausalPFN + kNN marg. & 0.273 & 0.224 & 0.977 & 0.333 & \underline{0.646} & \underline{0.621} & 1.175 & 0.675 \\
        CausalPFN + kernel marg. & 0.275 & 0.225 & 0.979 & 0.338 & \underline{0.637} & \textbf{0.599} & 1.168 & \textbf{0.643} \\
        \addlinespace[2pt]
        P-learner (RF) & 0.515 & 0.482 & 1.147 & 0.572 & 0.769 & 0.727 & 1.280 & 0.761 \\
        P-learner (ridge) & 0.322 & 0.290 & 1.033 & \underline{0.372} & 0.891 & 0.879 & 1.348 & 0.893 \\
        P-learner (TabICL) & 0.416 & 0.387 & 1.094 & 0.481 & \underline{0.664} & \underline{0.636} & 1.218 & \underline{0.669} \\
        \addlinespace[2pt]
        KPV (cond. W) & 1.499 & 0.926 & 2.598 & 1.208 & 1.433 & 0.825 & 2.208 & 0.960 \\
        KPV (independent W) & 0.993 & 0.775 & 1.944 & 0.840 & 1.006 & 0.858 & 1.829 & 0.896 \\
        PMMR (cond. W) & 0.878 & 0.846 & 1.062 & 0.903 & 0.918 & 0.897 & 1.188 & 0.995 \\
        PMMR (independent W) & 0.903 & 0.910 & 1.033 & 0.905 & 0.956 & 0.961 & 1.156 & 0.991 \\
        \addlinespace[2pt]
        NMMR-V (fixed; cond. W) & 0.604 & 0.580 & 1.132 & 0.720 & 0.834 & 0.841 & 1.370 & 0.925 \\
        NMMR-V (fixed; ind. W) & 0.609 & 0.602 & 1.135 & 0.745 & 0.840 & 0.859 & 1.375 & 0.937 \\
        NMMR-U (fixed; cond. W) & 0.919 & 0.894 & 1.071 & 0.905 & 0.992 & 0.984 & 1.178 & 0.990 \\
        NMMR-U (fixed; ind. W) & 0.920 & 0.903 & 1.076 & 0.920 & 0.994 & 0.989 & 1.180 & 0.995 \\
        NMMR-V (median; cond. W) & 1.006 & 1.005 & 1.055 & 0.951 & 1.038 & 1.022 & 1.175 & 1.003 \\
        NMMR-V (median; ind. W) & 1.006 & 1.009 & 1.058 & 0.956 & 1.041 & 1.027 & 1.178 & 1.006 \\
        NMMR-U (median; cond. W) & 1.063 & 1.049 & 1.124 & 1.002 & 1.060 & 1.050 & \underline{1.129} & 1.011 \\
        NMMR-U (median; ind. W) & 1.063 & 1.050 & 1.122 & 1.007 & 1.060 & 1.051 & \underline{1.132} & 1.010 \\
        \addlinespace[2pt]
        ProximalFM (ours) & 0.300 & 0.279 & \textbf{0.884} & 0.365 & \underline{0.656} & \underline{0.636} & \textbf{1.061} & \underline{0.688} \\
        \bottomrule
    \end{tabular}%
    }
\end{table*}

\begin{table*}[p]
    \centering
    \scriptsize
    \setlength{\tabcolsep}{3.2pt}
    \renewcommand{\arraystretch}{0.92}
    \caption{Exhaustive semisynthetic CATE nPEHE results at $n_{\mathrm{ctx}}=2,048$. Each entry first averages the three replicates of each source dataset, then averages equally over the 12 sources; lower is better. The lowest mean in each column is bold. Other underlined methods are not significantly worse than that empirical best under one-sided exact paired sign-flip tests across source-level means, with Holm correction within the column ($\alpha=0.05$).}
    \label{tab:semisynt_exhaustive_npehe_2048}
    \resizebox{\textwidth}{!}{%
    \begin{tabular}{@{}lrrrrrrrr@{}}
        \toprule
        & \multicolumn{4}{c}{Linear} & \multicolumn{4}{c}{Nonlinear} \\
        \cmidrule(lr){2-5} \cmidrule(l){6-9}
        & \multicolumn{2}{c}{Low confounding} & \multicolumn{2}{c}{High confounding} & \multicolumn{2}{c}{Low confounding} & \multicolumn{2}{c}{High confounding} \\
        \cmidrule(lr){2-3} \cmidrule(lr){4-5} \cmidrule(lr){6-7} \cmidrule(l){8-9}
        Method & Low proxy & High proxy & Low proxy & High proxy & Low proxy & High proxy & Low proxy & High proxy \\
        \midrule
        S-learner (ridge) & 1.011 & 1.011 & 1.560 & 1.560 & 1.022 & 1.022 & 1.606 & 1.606 \\
        S-learner (RF) & 0.513 & 0.513 & 1.356 & 1.356 & 0.740 & 0.740 & 1.476 & 1.476 \\
        S-learner (TabICL) & 0.310 & 0.310 & 1.239 & 1.239 & 0.564 & 0.564 & 1.371 & 1.371 \\
        S-learner (TabICL) + kNN marg. & 0.291 & 0.298 & 0.953 & 0.441 & 0.580 & 0.624 & 1.111 & 0.695 \\
        S-learner (TabICL) + kernel marg. & 0.293 & 0.297 & 0.953 & 0.437 & 0.585 & 0.608 & 1.111 & 0.679 \\
        \addlinespace[2pt]
        T-learner (ridge) & 0.205 & 0.205 & 1.264 & 1.264 & 1.112 & 1.112 & 1.754 & 1.754 \\
        T-learner (RF) & 0.482 & 0.482 & 1.363 & 1.363 & 0.751 & 0.751 & 1.482 & 1.482 \\
        T-learner (TabICL) & 0.306 & 0.306 & 1.303 & 1.303 & \underline{0.604} & \underline{0.604} & 1.420 & 1.420 \\
        T-learner (ridge) + kNN marg. & \textbf{0.179} & \textbf{0.132} & 0.960 & \textbf{0.271} & 1.052 & 0.936 & 1.491 & 0.984 \\
        T-learner (ridge) + kernel marg. & 0.181 & 0.134 & 0.962 & \underline{0.272} & 1.046 & 0.913 & 1.486 & 0.961 \\
        \addlinespace[2pt]
        X-learner (RF) & 0.382 & 0.382 & 1.286 & 1.286 & 0.678 & 0.678 & 1.429 & 1.429 \\
        X-learner (TabICL) & 0.299 & 0.299 & 1.262 & 1.262 & \textbf{0.504} & \textbf{0.504} & 1.352 & 1.352 \\
        \addlinespace[2pt]
        CausalFM & 0.600 & 0.600 & 1.262 & 1.262 & 0.917 & 0.917 & 1.460 & 1.460 \\
        Do-PFN & 0.454 & 0.454 & 1.279 & 1.279 & 0.961 & 0.961 & 1.547 & 1.547 \\
        CausalPFN & 0.246 & 0.246 & 1.255 & 1.255 & 0.545 & 0.545 & 1.372 & 1.372 \\
        CausalPFN + kNN marg. & 0.217 & 0.167 & 0.946 & \underline{0.285} & 0.571 & \underline{0.540} & 1.127 & \underline{0.595} \\
        CausalPFN + kernel marg. & 0.218 & 0.168 & 0.948 & \underline{0.288} & 0.567 & \underline{0.521} & 1.121 & \textbf{0.571} \\
        \addlinespace[2pt]
        P-learner (RF) & 0.420 & 0.401 & 0.985 & 0.493 & 0.675 & 0.642 & 1.112 & 0.675 \\
        P-learner (ridge) & 0.237 & 0.203 & \underline{0.882} & \underline{0.311} & 0.870 & 0.859 & 1.219 & 0.873 \\
        P-learner (TabICL) & 0.334 & 0.308 & 0.939 & 0.420 & 0.575 & \underline{0.553} & \underline{1.037} & \underline{0.589} \\
        \addlinespace[2pt]
        KPV (cond. W) & 1.745 & 0.715 & 2.778 & 0.934 & 1.752 & 0.781 & 2.591 & 0.948 \\
        KPV (independent W) & 1.163 & 0.707 & 2.081 & 0.764 & 1.110 & 0.797 & 1.982 & 0.855 \\
        PMMR (cond. W) & 0.881 & 0.843 & 1.084 & 0.905 & 0.924 & 0.892 & 1.217 & 1.005 \\
        PMMR (independent W) & 0.897 & 0.905 & 1.050 & 0.902 & 0.954 & 0.957 & 1.182 & 0.994 \\
        \addlinespace[2pt]
        NMMR-V (fixed; cond. W) & 0.527 & 0.535 & 1.104 & 0.663 & 0.803 & 0.805 & 1.369 & 0.901 \\
        NMMR-V (fixed; ind. W) & 0.531 & 0.557 & 1.109 & 0.694 & 0.807 & 0.818 & 1.372 & 0.910 \\
        NMMR-U (fixed; cond. W) & 0.893 & 0.843 & 1.061 & 0.881 & 0.991 & 0.970 & 1.173 & 0.978 \\
        NMMR-U (fixed; ind. W) & 0.897 & 0.858 & 1.063 & 0.902 & 0.992 & 0.979 & 1.175 & 0.984 \\
        NMMR-V (median; cond. W) & 0.964 & 0.966 & 1.123 & 0.939 & 1.008 & 1.007 & 1.268 & 0.990 \\
        NMMR-V (median; ind. W) & 0.969 & 0.972 & 1.127 & 0.950 & 1.008 & 1.009 & 1.271 & 0.997 \\
        NMMR-U (median; cond. W) & 1.035 & 1.032 & 1.127 & 1.002 & 1.054 & 1.047 & 1.232 & 1.005 \\
        NMMR-U (median; ind. W) & 1.035 & 1.032 & 1.134 & 1.003 & 1.054 & 1.048 & 1.228 & 1.006 \\
        \addlinespace[2pt]
        ProximalFM (ours) & 0.241 & 0.217 & \textbf{0.834} & \underline{0.302} & 0.572 & \underline{0.546} & \textbf{1.029} & \underline{0.620} \\
        \bottomrule
    \end{tabular}%
    }
\end{table*}

\begin{table*}[p]
    \centering
    \scriptsize
    \setlength{\tabcolsep}{4.0pt}
    \renewcommand{\arraystretch}{0.92}
    \caption{Extended-context semisynthetic CATE nPEHE results for nonlinear mechanisms with high confounding and low proxy reliability. Each entry first averages the 3 replicates of each source dataset, then averages equally over the 12 sources; lower is better. The lowest mean in each column is bold. Other underlined methods are not significantly worse than that empirical best under one-sided exact paired sign-flip tests across source-level means, with Holm correction within the column ($\alpha=0.05$).}
    \label{tab:semisynt_extended_context_npehe}
    \resizebox{\textwidth}{!}{%
    \begin{tabular}{@{}lrrrrrrr@{}}
        \toprule
        Method & \multicolumn{7}{c}{Context size $n_{\mathrm{ctx}}$} \\
        \cmidrule(l){2-8}
        & $128$ & $256$ & $512$ & $1,024$ & $2,048$ & $4,096$ & $8,192$ \\
        \midrule
        S-learner (ridge) & 1.459 & 1.499 & 1.580 & 1.593 & 1.606 & 1.597 & 1.608 \\
        S-learner (RF) & 1.542 & 1.536 & 1.546 & 1.497 & 1.476 & 1.602 & 1.473 \\
        S-learner (TabICL) & 1.457 & 1.449 & 1.432 & 1.393 & 1.371 & 1.340 & 1.317 \\
        S-learner (TabICL) + kNN marg. & 1.259 & \underline{1.235} & \underline{1.206} & 1.169 & 1.111 & 1.065 & 1.020 \\
        S-learner (TabICL) + kernel marg. & 1.245 & \underline{1.231} & \underline{1.196} & 1.159 & 1.111 & 1.064 & 1.016 \\
        \addlinespace[2pt]
        T-learner (ridge) & 1.837 & 1.699 & 1.665 & 1.627 & 1.754 & 1.711 & 1.647 \\
        \addlinespace[2pt]
        X-learner (RF) & 1.605 & 1.559 & 1.508 & 1.459 & 1.429 & 1.504 & 1.363 \\
        X-learner (TabICL) & 1.669 & 1.493 & 1.462 & 1.396 & 1.352 & 1.330 & 1.307 \\
        \addlinespace[2pt]
        CausalPFN & 1.687 & 1.541 & 1.488 & 1.412 & 1.372 & 1.371 & 1.376 \\
        CausalPFN + kNN marg. & 1.460 & 1.338 & \underline{1.264} & 1.175 & 1.127 & 1.121 & 1.117 \\
        CausalPFN + kernel marg. & 1.466 & 1.349 & \underline{1.251} & 1.168 & 1.121 & 1.140 & 1.127 \\
        \addlinespace[2pt]
        P-learner (RF) & 1.514 & 1.485 & 1.406 & 1.280 & 1.112 & 0.964 & 0.732 \\
        P-learner (ridge) & 1.511 & 1.489 & 1.443 & 1.348 & 1.219 & 1.110 & 0.952 \\
        P-learner (TabICL) & 1.535 & 1.467 & 1.368 & 1.218 & \underline{1.037} & \textbf{0.847} & \textbf{0.660} \\
        \addlinespace[2pt]
        NMMR-V (fixed; cond. W) & 1.290 & 1.363 & 1.368 & 1.370 & 1.369 & 1.336 & 1.337 \\
        NMMR-V (fixed; ind. W) & 1.292 & 1.366 & 1.371 & 1.375 & 1.372 & 1.338 & 1.339 \\
        NMMR-U (fixed; cond. W) & $6.11\!\times\!10^{5}$ & $4.25\!\times\!10^{3}$ & 33.761 & 1.178 & 1.173 & 1.144 & 1.141 \\
        NMMR-U (fixed; ind. W) & $5.97\!\times\!10^{5}$ & $4.29\!\times\!10^{3}$ & 33.990 & 1.180 & 1.175 & 1.149 & 1.144 \\
        NMMR-V (median; cond. W) & \underline{1.115} & \underline{1.191} & \textbf{1.109} & 1.175 & 1.268 & 1.334 & 1.391 \\
        NMMR-V (median; ind. W) & \textbf{1.115} & \textbf{1.188} & \underline{1.111} & 1.178 & 1.271 & 1.335 & 1.396 \\
        NMMR-U (median; cond. W) & 6.794 & \underline{1.318} & \underline{1.121} & 1.129 & 1.232 & 1.247 & 1.231 \\
        NMMR-U (median; ind. W) & 7.024 & \underline{1.311} & \underline{1.121} & 1.132 & 1.228 & 1.242 & 1.230 \\
        \addlinespace[2pt]
        ProximalFM (ours) & 1.329 & \underline{1.210} & \underline{1.151} & 1.061 & 1.029 & 0.988 & 0.964 \\
        ProximalFM-FT (ours) & 1.333 & \underline{1.216} & \underline{1.145} & \textbf{1.049} & \textbf{1.007} & 0.960 & 0.930 \\
        \bottomrule
    \end{tabular}%
    }
\end{table*}

\FloatBarrier
\subsubsection{Speed}
\label{app:results_speed}

We measured runtime in a separate run of the 36 semi-synthetic episodes with nonlinear mechanisms, strong confounding, and weak proxy information. Episodes were executed sequentially in one process to avoid contention between simultaneous benchmark jobs; methods retained the parallelism specified for their own hyperparameter searches. At each context size, the recorded wall-clock time covers estimator construction, fitting, and prediction on the 256 query rows. The filled points in \Cref{fig:results_speed_rank} show the median of these per-episode times at $n_{\mathrm{ctx}}=2\,048$; hollow points subtract the measured hyperparameter-search time. Accuracy ranks come independently from the complete semi-synthetic evaluation matrix, using paired episode-level CATE nPEHE ranks.

\FloatBarrier
\subsection{Synthetic study of proxy informativeness}
\label{app:proxy_informativeness}

We vary proxy information within a fixed causal mechanism to examine when proximal estimation offers an advantage over regression adjustment using proxies. For each $d\in\{1,3,5\}$, we generate $20$ information levels equally spaced from $0.01$ to $0.99$.

For each of the resulting $60$ datasets, we evaluate methods using $128$, $512$, or $2\,048$ context rows and the same $256$ held-out query rows. Each block $X,U,W,Z$ has dimension $d$. \Cref{fig:benchmark_construction}(b) shows the causal structure and proxy-information control.

\paragraph{Causal mechanism and heterogeneity.} For every observation, $X,U,\epsilon_W,\epsilon_Z$ are mutually independent $\mathcal N(0,I_d)$ vectors. The independent outcome disturbance is $\epsilon_Y\sim\mathcal N(0,1)$. Independent uniform draws determine the Bernoulli treatments. Define
\[
S_X=d^{-1/2}\sum_{j=1}^d X_j,\qquad
S_U=d^{-1/2}\sum_{j=1}^d U_j.
\]
We generate
\begin{align}
p(X,U)&=\operatorname{clip}\{\operatorname{sigmoid}(0.5S_X+1.25S_U),0.02,0.98\},\\
A&\sim\operatorname{Bernoulli}\{p(X,U)\},\\
Y(0)&=0.5S_X+1.5S_U+\epsilon_Y,\\
Y(1)&=Y(0)+0.5+S_X+0.5S_U,\\
Y&=(1-A)Y(0)+AY(1).
\end{align}
The individual and conditional average effects satisfy
\[
\operatorname{ITE}=0.5+S_X+0.5S_U,\qquad
Q(X)=\E[\operatorname{ITE}\mid X]=0.5+S_X.
\]
Thus the CATE varies with $X$, while latent variation also modifies individual effects. In the population, $\operatorname{Var}(Q(X))=1$ and $\operatorname{Var}(\operatorname{ITE}-Q(X))=0.25$. Normalising the scores by $\sqrt d$ holds the population scales of confounding and heterogeneity fixed across dimensions.

\paragraph{Joint proxy information.} For each dimension, we draw and fix independent orthogonal matrices $R_W,R_Z$ using Gaussian matrices and QR decomposition with a nonnegative triangular diagonal. At target information $s$, define $q=s/(2-s)$ and generate
\[
W=\sqrt q\,R_WU+\sqrt{1-q}\,\epsilon_W,\qquad
Z=\sqrt q\,R_ZU+\sqrt{1-q}\,\epsilon_Z.
\]
Each proxy coordinate has unit variance. Each complete proxy block alone explains a fraction $q$ of each latent coordinate's variance. The joint Gaussian distribution gives
\begin{align}
\E[U\mid W,Z]&=\frac{\sqrt q}{1+q}(R_W^\top W+R_Z^\top Z),\\
\operatorname{Cov}(U\mid W,Z)&=\frac{1-q}{1+q}I_d=(1-s)I_d.
\end{align}
Consequently, the population $R^2$ for predicting any coordinate of $U$ from $(W,Z)$ jointly is exactly $s$. Independence of $X$ preserves this value when conditioning additionally on $X$. The rotations mix latent coordinates without changing information.

\paragraph{Bridge and pairing.} Define the unbiased proxy score $\widetilde S_U(W)=\mathbf{1}^\top R_W^\top W/\sqrt{qd}$. An explicit bridge is
\[
h_0(W,a,X)=0.5S_X+(1.5+0.5a)\widetilde S_U(W)+a(0.5+S_X).
\]
Independent proxy errors give $\E[Y-h_0(W,A,X)\mid A,Z,X]=0$, and averaging the bridge contrast over $W\mid X$ gives $Q(X)=0.5+S_X$. The proxy-score error has conditional variance $(1-q)/q$. This quantifies the noise amplification as proxy information decreases, even though the bridge equation remains valid.

\paragraph{Estimator comparison.}

We fit an interacted proximal two-stage least-squares (2SLS) estimator \citep{tchetgenIntroductionProximalCausal2020} with bridge and instrument designs
\begin{align*}
D_i&=(1,X_i^\top,A_i,A_iX_i^\top,W_i^\top,A_iW_i^\top),\\
T_i&=(1,X_i^\top,A_i,A_iX_i^\top,Z_i^\top,A_iZ_i^\top).
\end{align*}
Writing $P_T$ for projection onto the instrument column space, 2SLS minimises $\|P_T(Y-D\beta)\|_2^2$. The treatment interactions allow the bridge basis to contain $h_0$ and to represent effect modification by both $X$ and the latent confounder. We compute the projection using economic QR followed by SVD, retaining singular values above $\epsilon_{\mathrm{mach}}\max(n_{\mathrm{ctx}},p_T)\sigma_{\max}$, where $p_T=2+4d$. The subsequent least-squares fit uses NumPy's \texttt{rcond=None} rank threshold.

We compare 2SLS with unpenalised OLS using the design $(1,X^\top,A,AX^\top,C^\top,AC^\top)$, with either no additional covariates ($C$ empty) or proxy adjustment ($C=(W,Z)$). For these linear estimators, we average the fitted treatment contrast over the context distribution of the additional covariates:
\[
\widehat Q(x)=\widehat\beta_A+x^\top\widehat\beta_{AX}
+\overline C_{\mathrm{ctx}}^\top\widehat\beta_{AC}.
\]
For 2SLS, $C=W$; for X-only OLS, the final term is omitted. This marginalisation targets $Q(x)$ because $X$ is independent of the additional covariates in the data-generating mechanism.

We also evaluate ProximalFM and proxy-adjusted CausalPFN. ProximalFM conditions on the observational context $(X,W,Z,A,Y)$ and predicts CATE from query $X$. CausalPFN treats $(W,Z)$ as adjustment covariates; we marginalise its treatment contrasts using 16 seeded, aligned proxy pairs sampled from the context. At each context size, all methods use the same training rows and are scored against the known $Q(X)$ on the same 256 held-out query rows. We report
$\operatorname{PEHE} =\{256^{-1}\sum_{i\in\mathrm{query}} [\widehat Q(X_i)-Q(X_i)]^2\}^{1/2}$.

\begin{figure}[t]
    \centering
    \includegraphics[width=\linewidth]{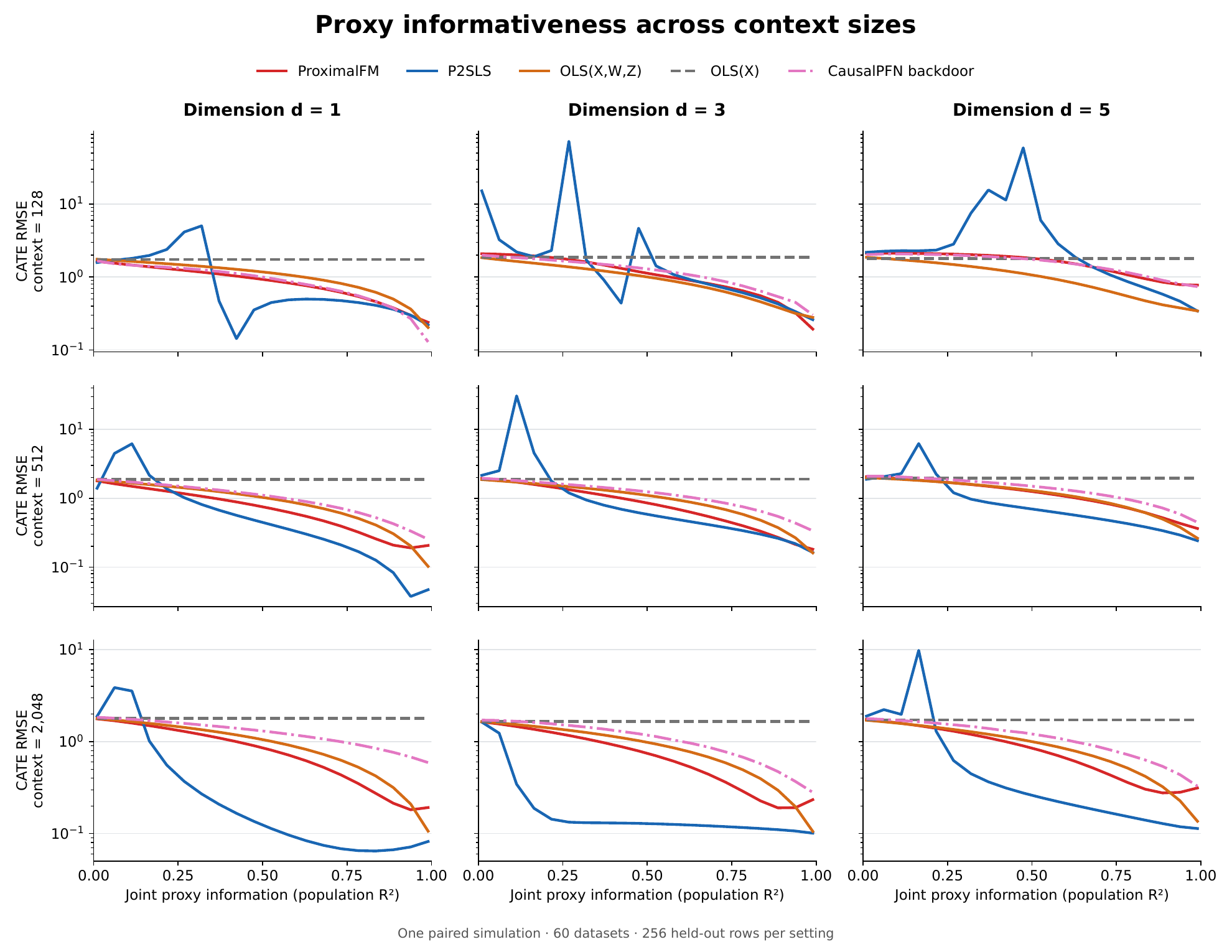}
    \caption{CATE RMSE against joint proxy information, measured by the population $R^2$ for predicting $U$ from $(W,Z)$. Columns vary proxy dimension $d$; rows vary context size. All methods are evaluated on the same 256 held-out query rows at each setting. The vertical axis is logarithmic. Within each dimension, the underlying draws are shared across information levels; results come from one realization per setting.}\label{fig:proxy_informativeness_results}
\end{figure}

\paragraph{Discussion.}

\Cref{fig:proxy_informativeness_results} shows the advantage of a well-specified proximal estimator: the 2SLS bridge basis contains the true outcome bridge. Its large fluctuations at low proxy information reflect the difficulty of estimating that bridge with weak instruments, rather than model misspecification. ProximalFM improves more smoothly as information increases and, at $n_{\mathrm{ctx}}=2\,048$, outperforms proxy-marginalized CausalPFN throughout the sweep and proxy-adjusted OLS in most settings. OLS is a favourable comparator here: the treatment effect is linear, and independence of $X$ and $(W,Z)$ makes proxy marginalisation a simple average over context rows, without conditional-neighbour estimation. As the proxies nearly reveal $U$, OLS catches up with and surpasses \modelname. These curves come from one paired simulation per setting.
\clearpage
\section{Towards a real physical system: Causal Chambers}
\label{app:causal_chambers}

We use the Light Tunnel Mk2 on the Causal Chamber platform\footnote{See \url{https://causalchamber.ai} for more details.} \citep{causalchambers,causalchambers_zeitler} to study average treatment effects in a controllable but real (non-simulated) optical system. The experiment combines physical responses with software control to create dependence between a hidden confounder, proxy variables, treatment, and outcome. This provides an exploratory step towards proximal estimation in a real-world systems, complementing the synthetic CATE benchmarks.

\begin{figure}[h!]
\centering
\includegraphics[width=154mm]{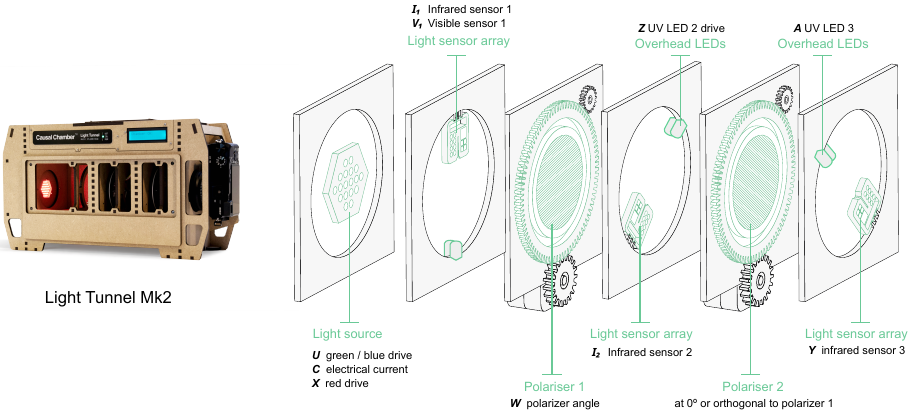}
\caption{The Light Tunnel Mk2 and a simplified diagram with the variables relevant to our setup.}
\label{fig:lt}
\end{figure}

\paragraph{Physical setup and causal structure.} Figure~\ref{fig:lt} shows the physical setup with the variables we consider, and Figure~\ref{fig:cc_setup} shows the intended pathways between them. The confounder $U$ is an LED drive drawn uniformly from the integers $0,\ldots,255$, controlling its brightness. The observed covariate $X$ controls a red LED, $Z$ controls the ultraviolet LED 2, and $W$ controls the position (in degrees) of the first polariser. Treatment $A\in\{0,1\}$ sets ultraviolet LED 3 to drive $a_1 A$, and the outcome $Y$ is the infrared sensor 3 reading. Auxiliary measurements comprise the electrical current $C$ drawn by the confounder LED, infrared readings $I_1,I_2$, and visible-light reading $V_1$.

The pathways $C\to X$, $I_1\to Z$, $V_1\to W$, and $I_2\to A$---resulting from the software control functions---link successive physical measurements to subsequent interventions. In the working proximal structure, $Z$ influences treatment assignment through $I_2$, while $W$ influences the light reaching the outcome sensor. The green LED and crossed polarisers define the nominal weak-confounding setting. The blue LED and a second polariser fixed at zero degrees define the nominal strong-confounding setting.

\begin{figure}[t]
\centering
\includegraphics[width=0.92\linewidth]{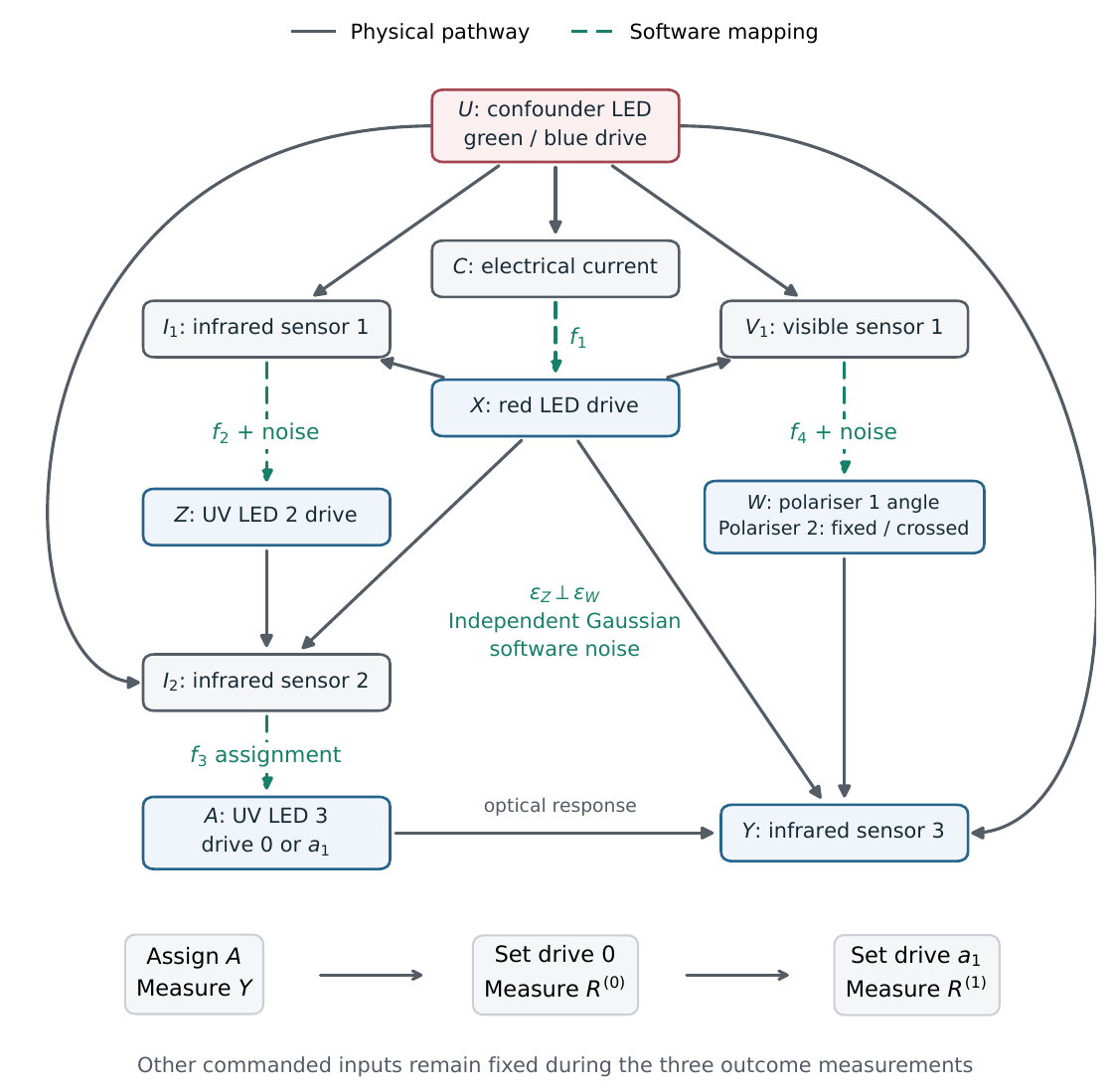}
\caption{Physical and software pathways in the optical experiment. Solid arrows represent the predetermined physical pathways and dashed green arrows represent effects resulting from the software control functions, i.e., \emph{software mappings}. Independent Gaussian noise enters the two proxy mappings. The lower sequence shows acquisition of the factual outcome and the two intervention measurements. The graph describes the intended within-observation structure. Its causal interpretation requires stable device state and the proxy restrictions discussed in the text.}
\label{fig:cc_setup}
\end{figure}

\paragraph{Control functions and acquisition.} Define
\[
C_{[l,h]}(v)=\max\{l,\min(h,v)\},\qquad
S_{L,k,m,b}(v)=\frac{L}{1+\exp\{-k(v-m)\}}+b.
\]
For each observation, we set the selected confounder LED to $U$ and measure $C$. We then set
\[
X=C_{[0,255]}\{\operatorname{round}(f_1(C))\}
\]
and measure $I_1,V_1$. Independent standard-normal draws $\epsilon_Z,\epsilon_W$ determine
\begin{align}
Z&=C_{[0,4095]}\{\operatorname{round}(f_2(I_1)+k_Z\sigma_Z\epsilon_Z)\},\\
W&=C_{[-180,180]}\{f_4(V_1)+k_W\sigma_W\epsilon_W\}.
\end{align}
Polariser angles use degrees. In the weak setting, the second polariser is crossed with the first at $W+90$, blocking most of the light reaching the third relative angle is zero. After setting these inputs, we measure $I_2$, assign treatment, and measure $Y$.

The linear family uses
\[
f_1(c)=C_{[0,255]}(-0.085c-1445),\qquad
f_2(i)=C_{[0,4095]}(0.0624i).
\]
The nonlinear family uses
\[
f_1(c)=S_{255,0.002,-19000,0}(c),\qquad
f_2(i)=S_{L,0.001,35000,0}(i),
\]
with $L=4095$ in the weak setting and $L=1000$ in the strong setting. Every family uses $f_4(v)=S_{180,0.002,4000,-90}(v)$. Thus the family label specifies the form of $f_1$ and $f_2$. Figure~\ref{fig:cc_controls} shows their saturation, the common sigmoid polariser response, and the treatment-assignment rules. These mappings change which parts of the optical response the experiment explores. The reduced nonlinear amplitude in the strong setting limits the contribution of $Z$ to the treatment-assignment sensor, preserving variation associated with the confounder.

\begin{figure}[t]
\centering
\includegraphics[width=\linewidth]{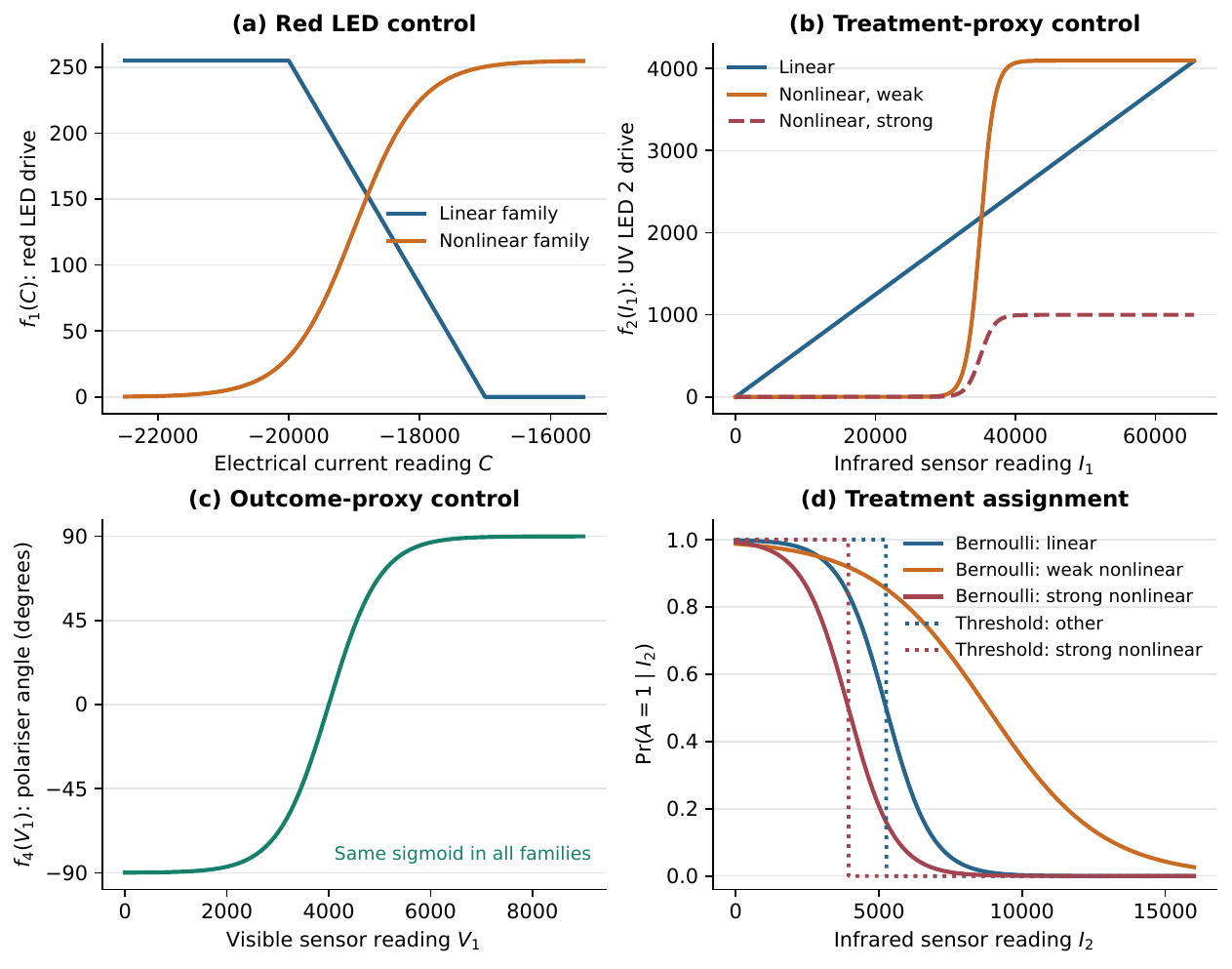}
\caption{Software control functions for the physical experiment. Panels (a)--(c) show the deterministic mappings before addition of proxy noise and final actuator clipping. Panel (d) shows the treatment probabilities for Bernoulli assignment and the corresponding hard-threshold protocols. The plots display the specified control functions over illustrative sensor ranges.}
\label{fig:cc_controls}
\end{figure}

\paragraph{Treatment and proxy-information settings.} We collect $1\,000$ under the following experimental design. We have a Bernoulli assignment as follows
\[
p(I_2)=\operatorname{sigmoid}\{(c-I_2)/s_A\},\qquad
A=\mathbf{1}\{V_A<p(I_2)\},\qquad
V_A\sim\operatorname{Unif}(0,1),
\]
where $(c,s_A)=(5250,800)$ for both linear settings, $(8789,2000)$ for weak nonlinear, and $(3930,800)$ for strong nonlinear. The Bernoulli rule supplies assignment variability at a given $I_2$. Its probabilities lie between $0.05$ and $0.95$ for $82.4\%$--$99.7\%$ of observations across the corresponding experimental conditions.

The noise multipliers $(k_Z,k_W)$ equal $(0,0)$ at high nominal informativeness, $(0.75,0.75)$ at intermediate informativeness, and $(2,2)$ at low informativeness. The strong-linear low-information condition uses $(0.6,2.5)$ to retain a measurable confounding signal while reducing proxy information. The fixed noise scales are
\[
\begin{array}{lrr}
\text{Setting}&\sigma_Z&\sigma_W\text{ (degrees)}\\
\hline
\text{Weak linear}&325&62.2\\
\text{Weak nonlinear}&1789&64.5\\
\text{Strong linear}&383&66.2\\
\text{Strong nonlinear}&412&62.6
\end{array}
\] We quantify realised information by the $R^2$ of a linear regression of $U$ on $(Z,W)$ with an intercept. Confounding is summarised by $\E_n[R^{(0)}\mid A=1]-\E_n[R^{(0)}\mid A=0]$. Because the proxies are physical actuators, their noise also changes the treatment and outcome distributions. Each noise level therefore defines a distinct operating distribution. NumPy random streams use seeds $981264976135$, $981264976142$, and $981264976146$ for $U$, proxy noises, and treatment uniforms, respectively. Restarting these streams for each condition pairs the commanded inputs and software randomness, while sensor values come from fresh physical measurements.

\paragraph{ATE reference and interpretation.} After measuring $Y$, we apply drive zero and then drive $a_1$, retaining the other commanded inputs, to obtain $R^{(0)}$ and $R^{(1)}$. The empirical average intervention reference is
\[
\widehat\theta_{\mathrm{ref}}
=\frac{1}{n}\sum_{i=1}^{n}\{R_i^{(1)}-R_i^{(0)}\}.
\]
The sequence in Figure~\ref{fig:cc_setup} makes this average directly measurable over the sampled operating states. Its causal interpretation assumes stability during each sequence and negligible systematic effects of measurement order or carryover.

The mean contrast ranges from approximately $886$ to $893$ sensor counts at drive $500$, and from $621$ to $624$ at drive $350$. Across conditions, the standard deviation of individual contrasts is $3.55$--$65.64$ counts. The root mean squared difference between $Y$ and the separately remeasured factual arm is $3.10$--$65.00$ counts. The overlap in these scales supports evaluating the average contrast, with individual contrasts interpreted in the presence of measurement variability. The physical estimand we obtain here is the ATE over the acquired states. The synthetic benchmarks in the other sections provide complementary evaluations against analytic CATE functions.

\paragraph{Adaptation of methods for ATE estimation.}

For each context prefix $\mathcal D_n=\{(X_i,W_i,Z_i,A_i,Y_i)\}_{i=1}^n$, we estimate the empirical ATE over the acquired operating states,
\[
\theta_n=\frac{1}{n}\sum_{i=1}^n \{Y_i(1)-Y_i(0)\}.
\]
All methods are fitted using the factual observations in $\mathcal D_n$ only. The paired intervention measurements $Y_i(0)$ and $Y_i(1)$ are reserved exclusively for the reference ATE. When a method produces a conditional effect estimate, we average its predictions over the same context prefix.

We consider a selection of methods from \Cref{app:results_baselineMethods}. \texttt{X-learner (TabICL)} fits arm-specific outcome and effect models using $(X,W,Z)$ and returns an estimated contrast at each observed $(X_i,W_i,Z_i)$; its ATE estimate is the mean of these contrasts. \texttt{CausalPFN} directly predicts a conditional treatment effect from the factual context, which we likewise average over the observed covariate and proxy values.

The proximal methods use the proxy roles explicitly. \texttt{P-learner} estimates outcome and treatment bridge functions and forms cross-fitted doubly robust proximal scores; its ATE estimate is the mean of these scores. \texttt{KPV} and \texttt{PMMR} estimate an outcome bridge $h(a,w,x)$ and obtain a marginal ATE by integrating the contrast $h(1,W,X)-h(0,W,X)$ over the empirical joint distribution of $(X,W)$ in the fitted context. \texttt{NMMR} uses the same empirical target, averaging its learned bridge contrast at the observed $(X_i,W_i)$ pairs. \texttt{P2SLS} uses a linear two-stage specification: it first predicts $W$ from $(A,X,Z)$ and then regresses $Y$ on $(A,X,\widehat W)$; the fitted coefficient of $A$ is its scalar ATE estimate.

Finally, \modelname{} treats the factual rows as its in-context dataset and predicts $\widehat{\tau}(X_i)$ for every context covariate value, with query-time proxies hidden. Its ATE estimate is the empirical mean $n^{-1}\sum_i\widehat{\tau}(X_i)$. This use of \modelname is not its primary design target: it uses $(W,Z)$ to interpret the factual context but predicts effects from $X$ alone, deliberately discarding query-time proxy information. It can therefore provide a useful ATE estimate by averaging its CATE predictions, but it is not designed to estimate the ATE directly or to integrate an outcome bridge over the empirical proxy distribution.

\FloatBarrier
\paragraph{Results}

\Cref{fig:causal_chambers_results} shows no single method that performs best in every physical setting. The \texttt{X-learner (TabICL)} is the most consistently accurate method, especially when confounding is weak. \texttt{CausalPFN} is also competitive, particularly for nonlinear mechanisms and weak confounding. In contrast, the proximal methods, including \modelname{}, perform well only in some settings, and larger contexts do not lead to systematic improvements. Overall, these results illustrate the difficulty and variability of causal estimation in real physical systems.

\begin{figure}
    \centering
    \includegraphics[width=\linewidth]{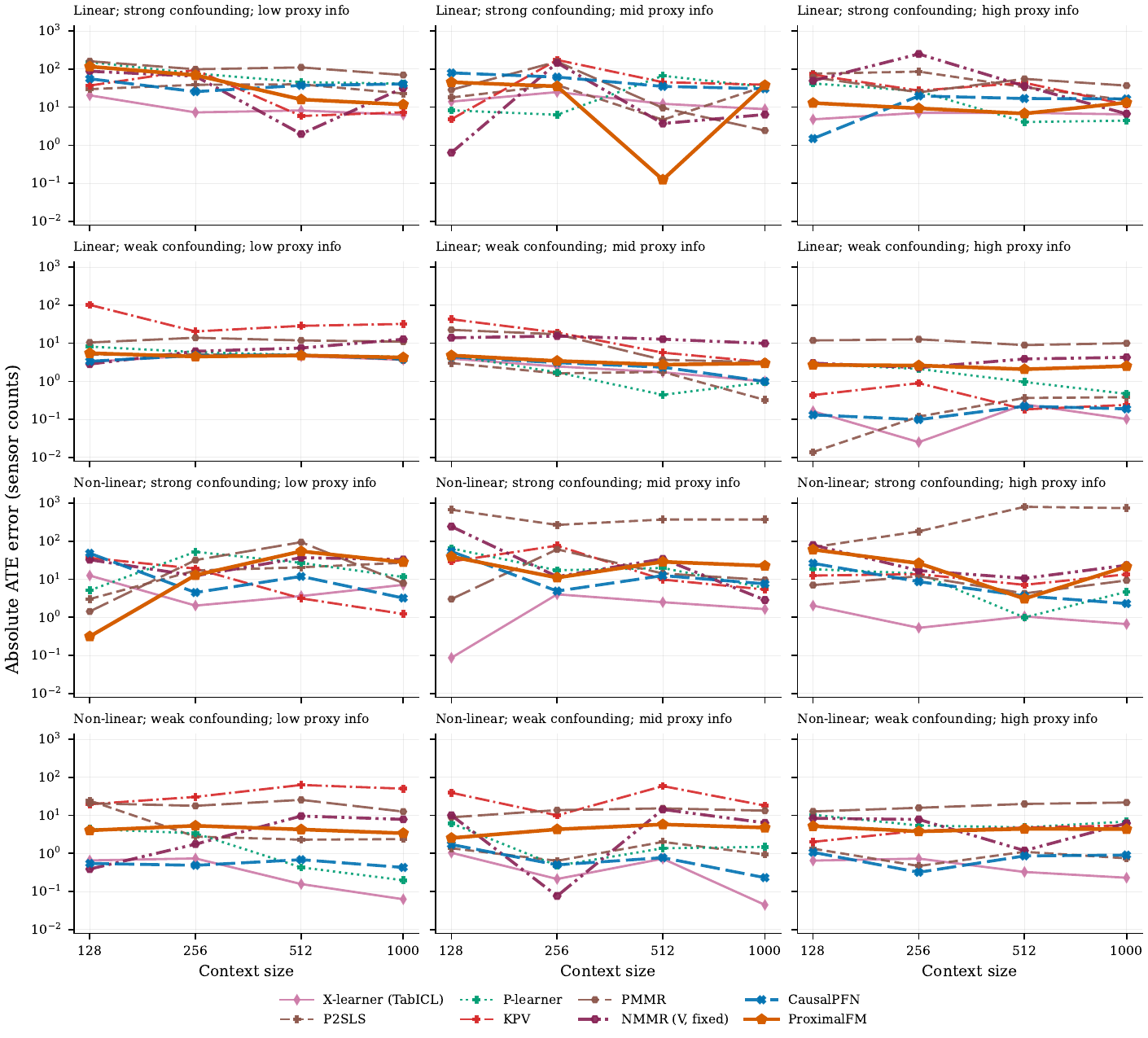}
    \caption{Absolute error of estimated average treatment effects in the Causal Chambers Light Tunnel experiment as the factual context grows. Each panel is one physical operating condition, crossing control-function family (linear or non-linear), confounding strength, and nominal proxy informativeness. Curves show $\lvert\widehat{\theta}_n-\widehat{\theta}_{\mathrm{ref},n}\rvert$ in raw infrared-sensor counts, where both the estimate and the paired-intervention reference are evaluated on the same nested context prefix of size $n$. Both axes use logarithmic scales; lower is better.}
    \label{fig:causal_chambers_results}
\end{figure}

\end{appendices}
\end{document}